\documentclass[mnsc,nonblindrev]{informs3}

\OneAndAHalfSpacedXI % current default line spacing
\usepackage{natbib}
 \bibpunct[, ]{(}{)}{,}{a}{}{,}%
 \def\bibfont{\small}%

\definecolor{myblue}{RGB}{0, 20, 114}
\usepackage[hidelinks,colorlinks=true,urlcolor=myblue,linkcolor=myblue,citecolor=myblue]{hyperref}
\def\EMAIL#1{\href{mailto:#1}{#1}}

\def\theARTICLETOP{}
\def\theLRHFirstLine{}
\def\theLRHSecondLine{}
\def\theRRHFirstLine{}
\def\theRRHSecondLine{}
\usepackage[ruled,vlined]{algorithm2e}
\usepackage{tikz}
\usepackage{enumitem}
\usepackage{subcaption}
\TheoremsNumberedThrough     % Preferred (Theorem 1, Lemma 1, Theorem 2)
\ECRepeatTheorems

\EquationsNumberedThrough    % Default: (1), (2), ...
\MANUSCRIPTNO{}

\begin{document}
%%%%%%%%%%%%%%%%

%\setcounter{page}{1}
%\DOI{10.1287/xxxx.0000.0000}%

\RUNAUTHOR{Hu, Hu and Zhou}

% \RUNTITLE{Predictive Maintenance in Manufacturing}
\RUNTITLE{Optimal Data Acquisition for Reinforcement Learning}

% \TITLE{Optimal Resource Allocation in Humanitarian Logistics: A Stochastic Programming Approach}
\TITLE{Optimal Data Acquisition for Reinforcement Learning: A Large Deviations Perspective}

\ARTICLEAUTHORS{%
\AUTHOR{Mingjie Hu}
\AFF{School of Management, Fudan University \\
H. Milton Stewart School of Industrial and Systems Engineering, Georgia Institute of Technology\\
\EMAIL{23110690009@m.fudan.edu.cn}}
\AUTHOR{Jian-Qiang Hu}
\AFF{School of Management, Fudan University, \EMAIL{hujq@fudan.edu.cn }}\AUTHOR{Enlu Zhou}
\AFF{H. Milton Stewart School of Industrial and Systems Engineering, Georgia Institute of Technology, \EMAIL{enlu.zhou@isye.gatech.edu}}
} % end of the block

\ABSTRACT{%
% Enter your abstract
Data acquisition efficiency is a central challenge in deploying reinforcement learning in business and healthcare operations, where interactions are costly, slow, and often involve humans in the loop. This paper develops a unified large deviations framework for data acquisition in infinite-horizon reinforcement learning. We introduce the exponential decay rate of the policy-selection error probability as a principled efficiency metric and derive a variational characterization of this rate via large deviations theory for Markov chains, yielding a nested optimization problem. Based on this characterization, we formalize two complementary notions of optimality in terms of the optimal solution of the nested problem. Because the resulting program is implicit and generally intractable, we propose a tractable convex relaxation with explicit constraints. We then develop a lazy one-step projected subgradient method to solve the relaxed problem and use its iterates to construct an adaptive data acquisition policy. We prove that the resulting reinforcement learning algorithm is near-robustly optimal under our optimality criterion, up to a constant factor. Finally, we extend the framework to linear function approximation to improve scalability, and numerical experiments support the effectiveness of the proposed approach.
}%

% \FUNDING{This research was supported by [grant number, funding agency].}

%Supplemental Material:
%Data Ethics & Reproducibility Note:

%\KEYWORDS{Stochastic programming, Decision support,Uncertainty, Disaster response, Optimization}

%\HISTORY{Received: Month DD, YYYY; Accepted: Month DD, YYYY; Published Online: Month DD, YYYY}

\maketitle
%%%%%%%%%%%%%%%%%%%%%%%%%%%%%%%%%%%%%%%%%%%%%%%%%%%%%%%%%%%%%%%%%%%%%%

% Text of your paper here

\section{Introduction}\label{sec:Intro}

Reinforcement learning studies how an agent can learn an optimal control policy through sequential interaction with an unknown, stochastic environment, typically modeled as a Markov decision process (MDP). As a core paradigm in artificial intelligence, RL has become a central framework for decision-making under uncertainty. Over the past decade, it has delivered striking successes across diverse domains, including Go \citep{silver2016mastering}, recommendation systems \citep{zhao2023kuaisim}, robotics, and large language model training \citep{guo2025deepseek}. 

Although reinforcement learning has achieved remarkable performance in domains such as Go and robotics, where agents can generate essentially unlimited experience by interacting with a simulated environment, its deployment in business and healthcare operations is far more challenging. In these settings, data collection often requires field experiments or human-in-the-loop interactions, making each additional sample costly and time-consuming. As a result, data acquisition efficiency becomes a primary consideration. Standard reinforcement learning algorithms, which typically overlook the cost of operational experimentation, can therefore be inefficient, impractical, and, in some cases, unacceptable.

\begin{example}[Operational Experiment Design]
    Consider a company preparing to launch a new product. Before the full roll-out, the company runs a short staged experiment to learn an effective selling policy. In each period, it observes the current market context (e.g., traffic composition, early retention signals, and inventory pressure) and selects operational actions such as discount levels and recommendation exposure. Collecting data is costly and slow because it requires real users and operational resources and may degrade user experience. The goal of the experiment is pure exploration: identify a high-performing policy with minimal interactions, and then deploy it at scale after the roll-out.
\end{example}

\begin{example}[Preclinical Treatment Experiment] 
Consider the preclinical evaluation of a multi-stage treatment protocol. Researchers conduct controlled animal studies to learn an effective treatment policy. At each stage, physiological measurements (e.g., biomarkers and vital signs) summarize the subject’s health status and guide subsequent interventions, such as selecting a drug, adjusting dosage, or switching therapies. These experiments are costly and time-consuming because each sample corresponds to a full treatment trajectory requiring monitoring and lab assays. The goal is to identify a high-performing policy using as few experimental trajectories as possible, and then advance the selected protocol to subsequent clinical testing.
\end{example}

There is also a growing literature on improving the data efficiency of reinforcement learning, often studied under the $\delta$-probably approximately correct (PAC) paradigm. In this line of work, the learner is typically allowed to interact with the environment until a prescribed confidence requirement is met, for example, returning an optimal (or near-optimal) policy whose probability of correct selection (PCS) exceeds a target threshold. While this fixed-confidence formulation is theoretically appealing, it is often ill-suited to operational settings in which experimentation cost is the binding constraint and the total number of interactions is predetermined, as in fixed-budget deployments. Moreover, meeting a stringent confidence target typically requires conservative exploration and stopping rules to control the error probability, which can introduce substantial statistical conservatism and lead to higher-than-necessary sample usage in practice. Therefore, the resulting algorithms and technical tools may not be well-suited to the fixed-budget setting.

Developing data-efficient acquisition algorithms for reinforcement learning under a fixed-budget formulation raises several fundamental challenges. First, unlike the fixed-confidence $\delta$-PAC paradigm, the fixed-budget formulation requires an efficiency criterion that directly characterizes the probability of outputting an optimal policy after a given number of interactions. Such a criterion should admit a sharp, closed-form characterization and explicitly identify the fundamental factors that determine how quickly the success probability improves as the budget increases. Second, the data acquisition process is fully adaptive: as new data arrive, the model estimate changes, and the acquisition policy must be updated accordingly. This feedback loop couples estimation and control over time, making both theoretical analysis and algorithm design substantially more challenging. Third, the fixed-budget regime shifts the goal from certifying correctness to minimizing the error probability under an adaptive, nonstationary sampling process. This departs from the classical fixed-confidence change-of-measure framework and calls for new analytical tools that optimize information accumulation along a time-varying, data-dependent trajectory.

\subsection{Main Contributions}
\subsubsection{Efficiency Metric.} We propose a new efficiency measure for fixed-budget data acquisition in reinforcement learning: the exponential decay rate of the probability of false selection (PFS). This criterion is analytically convenient because it admits a large deviations characterization, and it is operationally meaningful because it directly quantifies how quickly the policy-identification error decreases as the budget increases. It also identifies the fundamental instance-dependent factors that govern data acquisition efficiency. Using large deviations theory for Markov chains, we derive a variational representation of this rate through a nested optimization problem. To the best of our knowledge, this is the first large-deviations-based efficiency characterization for data acquisition in reinforcement learning, and it provides a concrete optimization target for designing provably efficient data acquisition policies.

\subsubsection{Notions of Optimality.} Based on this variational characterization, we introduce two notions of optimality for data acquisition: exact optimality and robust optimality. Exact optimality is instance-wise: it requires the algorithm to be consistent and its empirical data acquisition policy to converge to an optimal solution of the nested program. This requirement is generally unrealistic because the nested problem is rarely tractable; its objective is implicit, and its constraints are complex. This limitation motivates robust optimality, which requires only consistency and asymptotically optimal sampling on hard MDP instances, where identifying the optimal policy is statistically most challenging. Robust optimality, therefore, provides a practical and theoretically grounded benchmark for comparing reinforcement learning algorithms.

\subsubsection{Algorithm.} We begin the algorithm design by addressing the nested optimization problem in the rate function. Leveraging its structural properties, we derive a tractable convex surrogate problem, which directly motivates a new optimization-guided data acquisition strategy for reinforcement learning. Because the surrogate objective is non-smooth, we solve it using subgradient descent within a joint estimation-optimization loop: the MDP is updated online using the collected data, and the sampling policy is adjusted accordingly. To keep the computation lightweight, we adopt a lazy update scheme that performs only one subgradient step at selected times. We prove that the resulting algorithm is near-robustly optimal, up to a multiplicative factor of $1-\gamma$, where $\gamma$ denotes the discount factor. Our algorithm and convergence analysis use tools from convex optimization in a new way and may be of independent interest for other data acquisition problems in experimental design. To address large-scale problems, we further extend our results to the function-approximation setting. Specifically, we adopt the linear MDP framework and derive an analogous convex optimization program. The algorithm and analysis then carry over, yielding a principled approach to data acquisition in this setting.

\subsubsection{Empirical Validity.} Finally, we evaluate our method on a standard Gridworld benchmark and an operational experiment-design case study. In both settings, our method consistently achieves higher policy value and substantially higher correct-selection accuracy than state-of-the-art model-free and model-based baselines under the same budget. These results demonstrate clear gains in data acquisition efficiency. 

\subsection{Literature Review}
\subsubsection{Reinforcement Learning.} A large body of work studies reinforcement learning to achieve optimal regret \citep{jin2018q, azar2017minimax}. Regret is most natural in online learning, where the learner must balance exploration and exploitation while accruing reward. In contrast, we focus on a pure-exploration setting that prioritizes data acquisition efficiency to identify the optimal policy as quickly as possible. Closely related is the literature on $\delta$-PAC reinforcement learning \citep{fiechter1994efficient}, which adopts a fixed-confidence formulation: the objective is to identify an optimal (or near-optimal) policy with probability at least $1-\delta$ while minimizing the required number of samples, so performance is typically quantified by sample complexity. Depending on the MDP model, existing results can be grouped into finite-horizon episodic settings \citep{domingues2021episodic, tirinzoni2022near}, infinite-horizon discounted settings \citep{zanette2019almost,al2021adaptive,al2021navigating,russo2023model}, and infinite-horizon average-reward settings \citep{jin2021towards, wang2022near,tuynman2024finding}. Extensions to improve scalability under linear function approximation have also been studied \citep{wagenmaker2022instance, taupin2023best}. 

However, our work is fundamentally different because we study the fixed-budget setting. Fixed-confidence algorithms terminate once a stopping rule certifying the target error level is met; such rules are typically conservative and can lead to higher-than-necessary sample usage. In addition, our theoretical approach departs from the standard $\delta$-PAC analysis: sample complexity results are usually derived via change-of-measure arguments from a hypothesis-testing viewpoint, whereas we adopt a large deviations framework that directly characterizes error exponents. This perspective leads to different notions of optimality, informs algorithm design through variational characterizations, and yields a distinct route to performance guarantees.

\subsubsection{Ranking and Selection.} Our work is also related to fixed-budget ranking and selection in the simulation literature, where large deviations theory is used to characterize data acquisition efficiency. Ranking and selection can be viewed as a special case of reinforcement learning with a single state: the objective is to identify the best action when each action’s performance follows an unknown distribution. \cite{glynn2004large} established the large-deviation characterization for this classical setting using the G\"{a}rtner-Ellis theorem~\citep{dembo2009large}. Subsequent work has extended this framework to settings with stochastic constraints \citep{hunter2013optimal,hu2024multi}, alternative performance criteria such as expected opportunity cost \citep{gao2017new}, tom-$m$ selection \citep{zhang2023asymptotically}, contextual information \citep{du2024contextual}, similarity information \citep{zhou2024sequential,hu2024optimal}, multi-objective selection \citep{xiao2024optimal}, and input uncertainty \citep{wang2025optimal, kim2025selection}.

However, reinforcement learning is substantially more complex than ranking and selection, so extending the large deviations framework is nontrivial and introduces new technical and algorithmic challenges. Unlike the independent and identically distributed (i.i.d.) setting, data are generated by a Markov chain induced by the behavior policy, and uncertainty arises jointly from transitions and rewards, which enter the value function through a highly nonlinear dependence. Moreover, the rate function involves an intractable nested optimization, making exact optimality largely unattainable, and the Bellman flow constraints prevent arbitrary sampling of state-action pairs, which can render the optimal sampling ratios non-unique and complicate convergence analysis. Finally, scalability is critical, as computational cost grows rapidly with the size of the state-action space, requiring highly efficient algorithms.

Three papers are closely related to ours. \citet{li2021optimal} extends classical ranking and selection methods to identify the best action at the root in Monte Carlo tree search, which can be viewed as a finite-horizon MDP problem and differs from our formulation. \citet{zhu2024uncertainty} establishes central limit theorem behavior for estimated $Q$-values in infinite-horizon MDPs and proposes an exploration policy based on relative discrepancy. \citet{shi2025sample} further extends these results to asynchronous $Q$-iterations. Our contribution differs from these works in three key respects. First, while these papers primarily emphasize uncertainty quantification, we focus on data acquisition efficiency and develop a unified large deviations framework that supports optimality definitions, algorithm design, and performance guarantees. Second, their exploration policies are driven by heuristic discrepancy criteria without an explicit optimality characterization, whereas our approach is derived from a variational nested optimization characterization of the error decay rate and admits near-optimality guarantees. Third, we introduce a lazy one-step projected subgradient algorithm that is computationally lightweight, performing updates only at selected time steps rather than repeatedly solving an optimization problem to high accuracy. Empirically, our method also achieves higher sample efficiency and better scalability on larger problem instances.

\subsubsection{Adaptive Experiment Design.} Our method can also be viewed through the lens of adaptive experimental design, which is widely used in industry and has attracted substantial recent attention in academia \citep{johari2022experimental,bastani2022meta,liu2024large,chen2025efficient}.
In the literature, experiments are designed sequentially to collect data toward operational goals such as estimating average treatment effects, identifying the best decision, learning causal effects, or minimizing cumulative regret \citep{simchi2023multi,simchi2023pricing}. In contrast, we study a multi-stage experimental design problem naturally modeled as an MDP. We introduce a principled fixed-budget efficiency measure and corresponding optimality notions, and develop an optimization-guided algorithm with performance guarantees. These tools may be of independent interest and could be adapted to improve data acquisition efficiency in other adaptive experimental design settings.

The paper is organized as follows. Section \ref{sec: MDP} formulates the infinite-horizon discounted MDP setting, introduces our efficiency measure, and defines two notions of optimality. Section \ref{sec: algorithm} presents the algorithm and establishes its optimality guarantees. Section \ref{sec: large scale} extends the framework to large-scale reinforcement learning. Section \ref{sec: case study} reports numerical results that validate the algorithm’s performance. Finally, Section \ref{sec: conslusion} concludes the paper. All technical proofs are deferred to the Electronic Companion.

\section{Data Acquisition for Reinforcement Learning}
\label{sec: MDP}
In this section, we first introduce the infinite-horizon discounted MDP formulation. We then propose an efficiency measure for data acquisition in reinforcement learning from a large deviations perspective. Finally, we conclude by introducing two notions of optimality: exact optimality and robust optimality, which serve as benchmarks for evaluating reinforcement learning algorithms. 

\subsection{Infinite Horizon Discounted MDPs}
Consider an infinite-horizon, time-homogeneous, discounted tabular MDP defined by the tuple $\mathcal{M}:= (\mathcal{S},\mathcal{A}, P_{\mathcal{M}},r_{\mathcal{M}},\gamma)$, where 
\begin{itemize}
    \item[$\mathcal{S}$:] the finite state space of size $S$;
    \item[$\mathcal{A}$:] the finite action space of size $A$;
    \item[$P_{\mathcal{M}}$:] the transition kernel with $P_{\mathcal{M}}(s^\prime|s, a)$ being the probability of transitioning from state $s$ to state $s^\prime$ after taking action $a$;
    \item[$r_{\mathcal{M}}$:] the expected reward, with $r_{\mathcal{M}}(s,a)=\mathbb{E}[R_{\mathcal{M}}(s, a)]$, where $R_{\mathcal{M}}(s,a)$ is the random reward obtained when action $a$ is taken in state $s$, supported on $[0,1]$;
    \item[$\gamma$:] the discount factor $\in [0,1)$
\end{itemize}
%$\mathcal{S}$ denotes the finite state space of size $S$, $\mathcal{A}$ the finite action space of size $A$, $P_{\mathcal{M}}$ the transition kernel with $P_{\mathcal{M}}(s^\prime|s, a)$ being the probability of transitioning from state $s$ to state $s^\prime$ after taking action $a$. Let $R_{\mathcal{M}}(s,a)$ be the random reward obtained when action $a$ is taken in state $s$, supported on $[0,1]$; the expected reward is then $r_{\mathcal{M}}(s,a) = \mathbb{E}[R_{\mathcal{M}}(s, a)]$. The bounded-reward assumption is standard in the infinite-horizon MDP literature~\citep{puterman2014markov}. The scalar $\gamma \in [0,1)$ denotes the discount factor. A (stationary) policy is a mapping from the state space to distributions over actions, i.e., $\pi:\mathcal{S}\rightarrow\Delta(\mathcal{A})$. 
For a given policy $\pi$, the value function and state-action value function (Q-function) are defined as
\begin{equation*}
    V_{\mathcal{M}}^{\pi}(s) := \mathbb{E}_{\mathcal{M}}\left[\sum_{t=0}^{\infty}\gamma^tr_{\mathcal{M}}(s^{\pi}_t,a_t^{\pi})\bigg|s^{\pi}_0 = s\right]
\end{equation*}
and
\begin{equation*}
    Q_{\mathcal{M}}^{\pi}(s,a) := r_{\mathcal{M}}(s,a)+\gamma\sum_{s^\prime\in\mathcal{S}}P_{\mathcal{M}}(s^\prime|s,a)V_{\mathcal{M}}^{\pi}(s^\prime),
\end{equation*}
respectively, where $s_t^\pi$ and $a_t^\pi$ denote the state and action at time $t$ under policy $\pi$, respectively, and $\mathbb{E}_{\mathcal{M}}[\cdot]$ denotes expectation with respect to the randomness induced by $\mathcal{M}$ and $\pi$. Our goal is to identify the optimal control policy $\pi_{\mathcal{M}}^*$ that maximizes the value function. Let $V^*_{\mathcal{M}}$ denote the optimal value function of $\mathcal{M}$ and $Q^*_{\mathcal{M}}$ the corresponding optimal Q-function. Since a deterministic \textcolor{black}{stationary} Markovian policy can attain the optimal value in our setting \citep{puterman2014markov}, we focus on identifying an optimal policy within this class. Throughout, we assume that the optimal control policy is unique, and we discuss how to extend our results to identifying an $\epsilon$-optimal policy in Section \ref{sec: epsilon-optimal}. 

Throughout this paper, both the transition kernel $P_{\mathcal{M}}$ and the reward distribution $R_{\mathcal{M}}$ are unknown, placing the problem in a reinforcement learning setting that can be viewed as an adaptive data acquisition process. At each time step, the agent selects an action, observes the resulting transition and reward, and uses the accumulated samples to estimate an optimal control policy. In applications such as operational decision-making and preclinical treatment experiments, this task is particularly challenging because interactions are costly and data acquisition efficiency is critical. Moreover, the fixed-budget setting lacks a suitable efficiency measure to guide algorithm design and analysis.

To address these challenges, we adopt a fixed-budget framework in which the agent is allowed a total of $T$ interactions (samples). A learning algorithm is specified by a sampling rule and a decision rule, and the interaction proceeds sequentially. At the beginning of time $t$, the agent observes the current state $s_t$ and the history up to time $t$, $\mathcal{H}_t = \{s_0,a_0,R_0,s_1,\ldots,R_{t-1},s_t\}$, where $R_{l}$ denotes the realized reward at time $l$. Based on the observed history, the sampling rule selects an action $a_t\in\mathcal{A}$, and uses one unit of the budget. The environment then generates a reward according to $R_{\mathcal{M}}(s_t,a_t)$ and transitions to the next state $s_{t+1}\sim P_{\mathcal{M}}(\cdot|s_t,a_t)$ according to the unknown transition kernel. Once the total budget is exhausted, the algorithm terminates, and the decision rule outputs an estimated optimal policy based on all collected data.

\subsection{Efficiency Measure for Data Acquisition}

In this subsection, we quantify the efficiency of data acquisition in reinforcement learning from a large deviations perspective.
Under a fixed-budget setting, a natural performance metric is the PCS, defined as $\mathbb{P}_{\mathcal{M}}(\hat{\pi}_{T}= \pi^*_{\mathcal{M}})$. An optimal data acquisition method would maximize this quantity by suitably designing its sampling and decision rules. However, for a finite budget $T$, the PCS is generally difficult to characterize because the sampling process is adaptive and history-dependent. Therefore, this probability is not well-suited for theoretical analysis. To overcome this difficulty, we study the large-deviation behavior of the PFS and define data acquisition efficiency by its asymptotic exponential decay rate:
\begin{equation*}
    \lim_{T\rightarrow \infty} -\frac{1}{T}\log \mathbb{P}_{\mathcal{M}}\left(\hat{\pi}_{T}\neq \pi^*_{\mathcal{M}}\right) = I,
\end{equation*}
where $I$ is the associated error decay rate. Compared with the PCS, this criterion has three advantages: (1) $I$ often admits a tractable characterization via large deviations theory; (2) it identifies the fundamental instance-dependent factors that govern data acquisition efficiency; and (3) it yields an asymptotic notion of optimality, in which an efficient method drives the PFS to zero at the fastest exponential rate. This measure is widely used in the fixed-budget ranking and selection literature~\citep{glynn2004large,hunter2013optimal,gao2017new} to quantify the data acquisition efficiency of ranking and selection algorithms. 

Establishing a closed-form expression for the rate function $I$ is relatively straightforward in the ranking and selection setting. There, the false selection event can be written as the union of pairwise comparison events, and the overall exponential decay rate is determined by the bottleneck event with the slowest decay rate.  In reinforcement learning, however, deriving the corresponding rate function is substantially more challenging. First, such a pairwise decomposition is no longer tractable because the policy space grows exponentially with the number of state-action pairs. Second, state-action samples in an MDP are not i.i.d., i.e., they are generated by the dynamics of an underlying Markov chain, so standard large deviations principles for i.i.d.\ observations do not directly apply. Third, unlike classical ranking and selection, where uncertainty typically arises from reward noise, reinforcement learning must account for uncertainty in both the reward function and the transition kernel. Therefore, the rate function must jointly capture the large-deviation effects of both reward and transition uncertainty.

To derive the rate function $I$, we first specify the decision rule. Let the empirical MDP after $T$ interactions be $\bar{\mathcal{M}}(T) = (\mathcal{S},\mathcal{A}, P_{\bar{\mathcal{M}}(T)},R_{\bar{\mathcal{M}}(T)},\gamma)$, where $P_{\bar{\mathcal{M}}(T)}$ and 
$R_{\bar{\mathcal{M}}(T)}$ are the empirical transition kernel and empirical mean reward function estimated from the sampled trajectory $\{(s_t,a_t,r_t,s_{t+1})\}_{t=1}^{T}$.  For each $(s,a)\in\mathcal{S}\times\mathcal{A}$, define the visitation count up to time $T$ as
\begin{equation*}
    N(s,a;T) := \sum_{t=1}^{T} \mathbb{I}\{(s_t,a_t) = (s,a)\},
\end{equation*}
where $\mathbb{I}(\cdot)$ denotes the indicator function. Let $X_{\mathcal{M}}(s,a) = (X_1, \ldots,X_{S})$ denote the random basis vector corresponding to the next state after taking action $a$ in state $s$: $X_{s^\prime} = 1$ if the realized next state is $s^\prime \in \mathcal{S}$, and $X_{u}=0$ for all $u\neq s^\prime$. Then, $\mathbb{E}[X_{\mathcal{M}}(s,a)] = P_\mathcal{M}(s,a)$,
where $P_\mathcal{M}(s,a) = (P_{\mathcal{M}}(s^\prime|s,a))_{s^\prime \in \mathcal{S}}$ denotes the transition probability vector. The empirical transition probabilities are defined component-wise as
\begin{equation}
\label{eq: empirical_transitions}  
P_{\bar{\mathcal{M}}(T)}(s^\prime|s,a) :=  \frac{\sum_{t=1}^{T} \mathbb{I}\{(s_t,a_t,s_{t+1}) = (s,a,s^\prime)\}}{N(s,a;T)}
\end{equation}
whenever $N(s,a;T)>0$, and are set to zero otherwise. Similarly, the empirical expected rewards are defined as
\begin{equation}
\label{eq: empirical_rewards}  
 R_{\bar{\mathcal{M}}(T)}(s,a) :=  \frac{\sum_{t=1}^{T}R_{\mathcal{M}}(s_t,a_t)\mathbb{I}\{(s_t,a_t) = (s,a)\}}{N(s,a;T)} 
\end{equation}
whenever $N(s,a;T)>0$, and are set to zero otherwise. Finally, the decision rule $\hat{\pi}_{T}$ is defined as an optimal policy for the empirical MDP $\bar{\mathcal{M}}(T)$, computed for example via value iteration or policy iteration. 

To characterize the asymptotic PFS, we first circumvent the combinatorial complexity of the policy space by invoking the policy improvement theorem. This reformulates the global error event as a union of local improvement events:
\begin{equation*}
    \left\{\bigcup_{s\in \mathcal{S},a\in \mathcal{A}\setminus\{\pi^*_{\mathcal{M}}(s)\}}
    Q^{\pi^*_{\mathcal{M}}}_{\bar{\mathcal{M}}(T)}(s,a)>V^{\pi^*_{\mathcal{M}}}_{\bar{\mathcal{M}}(T)}(s)\right\}.
\end{equation*}
This reduction shifts the analysis from the high-dimensional policy space to local state-action deviations under the empirical MDP $\bar{\mathcal{M}}(T)$.
Building on this decomposition, Lemma~\ref{lemma: rate function G} relates the global large deviations rate to the corresponding local rates. 

Throughout this paper, we assume that $\bar{\mathcal{M}}(T)$ satisfies the relevant regularity conditions of $\mathcal{M}$, including ergodicity and uniqueness of the optimal policy. This assumption rules out finite-sample pathologies in which the empirical MDP may be ill-defined or degenerate. Since $\bar{\mathcal{M}}(T)$ converges to $\mathcal{M}$ under sufficient exploration, one can enforce this assumption through a uniformly exploratory data-collection policy without changing the asymptotic analysis. Similar treatments of finite-sample irregularities are common in ranking and selection \citep{du2024contextual}.
\begin{lemma}
\label{lemma: rate function G}
Assume that for each $s\in \mathcal{S}$ and each suboptimal action $a\in \mathcal{A}\setminus \{\pi^*_{\mathcal{M}}(s)\}$, the local large deviations limit exists:
\begin{equation*}
    \lim_{T\rightarrow \infty} -\frac{1}{T}\log \mathbb{P}\left(Q^{\pi^*_{\mathcal{M}}}_{\bar{\mathcal{M}}(T)}(s,a)>V^{\pi^*_{\mathcal{M}}}_{\bar{\mathcal{M}}(T)}(s)\right) = \mathcal{G}_{s,a}.
\end{equation*}
Then, the rate function for the PFS is governed by the minimum of these local rates:
\begin{equation*}
    \lim_{T\rightarrow \infty}-\frac{1}{T}\log\mathbb{P}\left(\hat{\pi}_{T}\neq \pi^*_\mathcal{M}\right) = \min_{s\in \mathcal{S}, a\neq \pi_{\mathcal{M}}^*(s)} \mathcal{G}_{s,a}.
\end{equation*}
\end{lemma}

Lemma~\ref{lemma: explicit G} provides an explicit variational representation of the local rate $\mathcal{G}_{s,a}$ under a fixed behavior policy $\pi$ used for data collection. This result is technically challenging because the observations are generated along a Markov trajectory and are therefore temporally dependent. The main \textcolor{black}{technical} novelty is to bring large deviations principles for Markov chains into the analysis of reinforcement learning data acquisition. Specifically, we first characterize the limiting log-moment generating function as a log-spectral radius using the Perron-Frobenius theorem~\citep{dembo2009large}, and then establish the large deviations principle through the G\"{a}rtner-Ellis theorem. We finally apply the Donsker-Varadhan variational formula to express the spectral radius in terms of local rate functions for transition and reward uncertainty. To state the result, we introduce the following notation. For any transition kernel $x$ and mean reward function $y$, let $\tilde{\mathcal{M}}:=(\mathcal{S},\mathcal{A},x,y,\gamma)$ denote the corresponding MDP. Define
\begin{equation}
\label{eq:error_event}
\mathcal{E}_{s,a}
:=
\left\{
(x,y):
Q^{\pi^*_{\mathcal M}}_{\tilde{\mathcal M}}(s,a)
>
V^{\pi^*_{\mathcal M}}_{\tilde{\mathcal M}}(s)
\right\},
\end{equation}
the set of alternative models under which  \(\pi^*_{\mathcal M}\) is not optimal.
For a given transition kernel \(x\), let \(\mathcal F_\pi(x)\) denote the set of stationary state-action distributions induced by \(x\) under the behavior policy \(\pi\):
$$
\mathcal F_\pi(x)
:=
\left\{
\eta_1\in\Omega:
\eta_1(s',a')
=
\sum_{s\in\mathcal S,a\in\mathcal A}
\eta_1(s,a)\,x(s'|s,a)\,\pi(a'|s'),
\ \forall (s',a')
\right\},
$$
where 
$$
\Omega
:=
\left\{
\eta_1\in\mathbb R^{S\times A}:
\eta_1(s,a)\ge 0,\ \forall (s,a),
\sum_{s\in\mathcal S,a\in\mathcal A}\eta_1(s,a)=1
\right\},
$$
\textcolor{black}{where \(\eta_1\) denotes a stationary state-action distribution induced by the behavior policy \(\pi\) under the transition kernel \(x\).}

We introduce the formal definition of ergodicity in Definition~\ref{def:ergodic-mdp}. This condition is widely used in MDP analysis and ensures well-behaved long-run dynamics \citep{puterman2014markov}.

\begin{definition}
\label{def:ergodic-mdp}
An MDP is called \emph{ergodic} if the transition matrix under every deterministic stationary policy consists of a single recurrent class.
\end{definition}

\begin{lemma}
\label{lemma: explicit G}
Suppose that the MDP $\mathcal M$ is ergodic and that the fixed behavior policy $\pi$ has full support. Then, for each \(s\in\mathcal S\) and each suboptimal action \(a\in\mathcal A\setminus\{\pi^*_{\mathcal M}(s)\}\),
$$
\mathcal G_{s,a}
=
\inf_{(x,y)\in\mathcal E_{s,a}}
\inf_{\eta_1\in\mathcal F_\pi(x)}
\sum_{s'\in\mathcal S,a'\in\mathcal A}
\eta_1(s',a')
\Big(
I_1\big(x(s',a')\big)
+
I_2\big(y(s',a')\big)
\Big),
$$
where \(I_1\big(x(s',a')\big)\) and \(I_2\big(y(s',a')\big)\) are the Fenchel-Legendre transforms of the logarithmic moment generating functions of \(X_{\mathcal M}(s',a')\) and \(R_{\mathcal M}(s',a')\), respectively.
\end{lemma}

\begin{theorem}
\label{thm: rate function}
Suppose that the MDP $\mathcal M$ is ergodic and that the fixed behavior policy $\pi$ has full support. Then, the PFS satisfies the large deviations principle with the following rate:
\begin{equation*}
\lim_{T\rightarrow\infty} -\frac{1}{T}\log \mathbb{P}_{\mathcal{M}}(\hat{\pi}_{T}\neq \pi_{\mathcal{M}}^*) = \min_{s\in \mathcal{S},a\in \mathcal{A}\setminus\{\pi^*_{\mathcal{M}}(s)\}}\mathcal G_{s,a}.
\end{equation*}
\end{theorem}

\textcolor{black}{Theorem~\ref{thm: rate function} characterizes the exact exponential decay rate of the PFS under the behavior policy $\pi$ as the sample size $T$ grows. The global rate is governed by the bottleneck state-action pair, i.e., the suboptimal pair $(s,a)$ with the smallest local rate $\mathcal G_{s,a}$. As shown in Lemma~\ref{lemma: explicit G}, $\mathcal G_{s,a}$ depends on transition and reward deviations, quantified by $I_1$ and $I_2$, respectively. This structure directly links data acquisition to learning efficiency: the behavior policy $\pi$ should be optimized to maximize the bottleneck rate. To the best of our knowledge, this is the first result that establishes a large deviations principle for empirical MDPs and uses it to quantify the efficiency of data acquisition in reinforcement learning.}

\subsection{Efficiency Metrics: Exact and Robust Optimality}

The efficiency of a data acquisition strategy can be assessed through the large deviations rate established in Theorem~\ref{thm: rate function}. Specifically, optimizing the behavior policy $\pi$ to maximize this rate yields the fastest asymptotic decay of the error probability as the budget $T$ tends to infinity.

Formally, for a given MDP $\mathcal{M}$ and a behavior policy $\pi$, define the exponential decay rate function
$$
     \mathcal{R}(\mathcal{M},\pi) := \min_{s\in \mathcal{S},a\in \mathcal{A}\setminus\{\pi^*_{\mathcal{M}}(s)\}}\inf_{(x,y)\in\mathcal E_{s,a}}
\inf_{\eta_1\in\mathcal F_\pi(x)}
\sum_{s'\in\mathcal S,a'\in\mathcal A}
\eta_1(s',a')
\Big(
I_1\big(x(s',a')\big)
+
I_2\big(y(s',a')\big)
\Big).
$$
An optimal data acquisition strategy maximizes this rate over behavior policies. Thus, the optimal exponential decay rate is
\begin{equation}
\label{eq: rate opt}
    \mathcal{R}^*(\mathcal{M}) : = \max_{\pi\in \Pi} \mathcal{R}(\mathcal{M},\pi),
\end{equation}
where $\Pi$ denotes the set of randomized stationary policies. 
Based on this quantity, Definition~\ref{def:exact-optimality} formalizes exact optimality for reinforcement learning algorithms, requiring both consistency and maximal efficiency in terms of the error decay rate.

\begin{definition}[Exact Optimality]
\label{def:exact-optimality}
A reinforcement learning algorithm is said to be \emph{exactly optimal} if it satisfies the following two conditions:

\begin{enumerate}
\item \textbf{Strong Consistency:} The algorithm identifies the optimal policy almost surely; that is,
\begin{equation*}
\mathbb{P}\left(\lim_{T \to \infty} \hat{\pi}_{T} = \pi^*_{\mathcal{M}}\right) = 1.
\end{equation*}

\item \textbf{Behavior Policy Optimality:} Let $\pi_T$ denote the behavior policy used by the algorithm at time $T$. Then
\begin{equation*}
\lim_{T\to\infty}
\mathcal{R}(\mathcal{M},\pi_T)
=
\mathcal{R}^*(\mathcal{M})
\quad \text{a.s.}
\end{equation*}
\end{enumerate}
\end{definition}
\begin{remark}
    Definition~\ref{def:exact-optimality} formalizes optimality from a large deviations perspective. It requires the algorithm to be strongly consistent and to asymptotically attain the optimal \textcolor{black}{error decay rate} in \eqref{eq: rate opt}. The term \emph{exact} emphasizes that this benchmark is defined by the exact solution of the optimization problem in \eqref{eq: rate opt}, which depends on the true transition dynamics and reward structure of the MDP.
\end{remark}

Exact optimality is a stringent notion because it requires the empirical behavior policy induced by the learning algorithm to recover an optimizer of the variational problem in \eqref{eq: rate opt}. In general, this problem is intractable. First, the rate functions $I_1$ and $I_2$ depend on the full underlying distributions and typically have no closed-form expressions. Second, the constraint set $\mathcal{E}_{s,a}$ is highly implicit, since both the $Q$-function and the value function depend nonlinearly on the transition kernel, often through matrix-inverse-type representations. Consequently, the inner feasible region is generally nonconvex and difficult to characterize. Finally, the outer maximization over $\pi\in\Pi$ is also challenging because its objective is defined through a nested optimization problem, making gradient information difficult to obtain. Thus, \eqref{eq: rate opt} should be viewed as a variational benchmark for the best achievable exponential rate, rather than as a quantity that can generally be computed exactly.

To address the computational intractability of instance-specific exact optimality, we introduce \emph{robust optimality} in Definition~\ref{def: robust-optimality}. Rather than taking an instance-wise perspective, we evaluate worst-case performance over a class of MDPs. For this purpose, we impose a generative-model assumption, under which samples can be collected from any state-action pair. This assumption is standard in worst-case analyses of reinforcement learning~\citep{azar2017minimax,jin2021towards} and is also aligned with offline reinforcement learning settings where datasets often contain sampled state-action pairs~\citep{levine2020offline}. It is needed here because the worst-case MDP may be non-communicating, in which case identifying an optimal policy from a single trajectory may be impossible. Importantly, the generative model is used only to characterize the worst-case rate function; it is not required by our algorithm. Under this assumption, the exponential decay rate depends on the sampling ratio, as characterized in Lemma~\ref{lemma: generate rate}.

\begin{lemma} 
\label{lemma: generate rate}
Under the generative model assumption, suppose that the sampling rule satisfies
\[
    \frac{N(s,a;T)}{T}\rightarrow \omega_{sa},
    \qquad \forall (s,a)\in\mathcal S\times\mathcal A,
\]
where $\omega\in\Omega$ denotes the sampling ratio on state-action pairs. Then the exponential decay rate function becomes
$$
        \mathcal{R}(\mathcal{M}, \omega) := \min_{s\in \mathcal{S},a\in \mathcal{A}\setminus\{\pi^*_{\mathcal{M}}(s)\}}\inf_{(x,y)\in\mathcal E_{s,a}}
\sum_{s'\in\mathcal S,a'\in\mathcal A}
\omega_{s^\prime a^\prime}
\Big(
I_1\big(x(s',a')\big)
+
I_2\big(y(s',a')\big)
\Big).
$$
\end{lemma}
For a prescribed gap scale $\Delta_0>0$, define the hard-instance class
$
\mathfrak M(\Delta_0)
:=
\left\{
\mathcal M:
\Delta_{\min}(\mathcal M)\ge \Delta_0
\right\},
$
\textcolor{black}{where
$
\Delta_{\min}(\mathcal M)
:=
\min_{s\in\mathcal S,\,
a\in\mathcal A\setminus\{\pi^*_{\mathcal M}(s)\}}
(
V^{\pi^*_{\mathcal M}}_{\mathcal M}(s)
-
Q^{\pi^*_{\mathcal M}}_{\mathcal M}(s,a))
$
is the minimum value gap between the optimal action and any suboptimal action in $\mathcal M$.}
Accordingly, we define the worst-case optimal exponential decay rate as
\begin{equation*}
\mathcal R^*
:=
\inf_{\mathcal M\in\mathfrak M(\Delta_0)}
\max_{\omega\in\Omega}
\mathcal R(\mathcal M,\omega).
\end{equation*}
Computing $\mathcal{R}^*$ exactly is generally intractable because it requires optimizing over the class $\mathfrak M(\Delta_0)$. Instead, we characterize its fundamental scaling behavior. We construct a hard MDP instance with decoupled state-action dynamics, in which no state-action pair is intrinsically more informative than another. Thus, identifying the optimal policy requires sampling nearly uniformly across all state-action pairs. We then construct an alternative model $\tilde{\mathcal{M}}$ under which the original optimal policy $\pi^*_{\mathcal{M}}$ becomes suboptimal. Combining this least favorable construction with analytical properties of the rate functions $I_1$ and $I_2$, we obtain
\begin{equation*}
    \mathcal{R}^* = O\left(\frac{(1-\gamma)^3\Delta^2_{0}}{C(S,A)}\right).
\end{equation*}
Here, $C(S,A)$ denotes the total number of state–action pairs. Appendix~\ref{sec: robust opt} provides the hard-instance construction and the derivation of this scaling order.

\begin{definition}[Robust Optimality]
\label{def: robust-optimality}
A reinforcement learning algorithm is said to be \emph{robustly optimal} if it satisfies the following two conditions:
\begin{enumerate} \item \textbf{Strong Consistency:} The algorithm identifies the optimal policy almost surely; that is,\begin{equation*} \mathbb{P}\left(\lim_{T \to \infty} \hat{\pi}_{T} = \pi^*_{\mathcal{M}}\right) = 1.
\end{equation*}
\item \textbf{Worst-Case Ratio Optimality:} \textcolor{black}{Let $\alpha(T) = (N(s,a;T)/T)_{s\in\mathcal{S},a\in\mathcal{A}}$ denote the empirical sampling ratio induced by the algorithm's data acquisition process. Then,} 
\begin{equation*}
\inf_{\mathcal M \in \mathfrak M(\Delta_0) }\liminf_{T\to\infty}
\mathcal{R}(\mathcal M,\alpha(T))
=
\Omega\left(
\mathcal{R}^*
\right).
\end{equation*}
\end{enumerate}
\end{definition}

\section{Algorithm with Optimal Data Acquisition}
\label{sec: algorithm}
In this section, we develop an efficient approximation method for solving problem \eqref{eq: rate opt}. Building on this approximation, we propose a reinforcement learning algorithm equipped with an adaptive data acquisition mechanism. We then establish theoretical guarantees showing that the proposed algorithm attains the robust optimal performance guarantee.

\subsection{Optimal Sampling Ratio}
\label{sec: surrogate ratio}
In this subsection, we assume that the underlying MDP $\mathcal{M}$ is fully known and study how to solve the nested optimization problem in \eqref{eq: rate opt}. The exact rate function depends on the state-action distribution $\eta_1$, which must be invariant under the candidate transition kernel $x$ and the behavior policy $\pi$. Thus, the alternative model $(x,y)$, the occupancy distribution $\eta_1$, and the behavior policy $\pi$ are tightly coupled, making the original optimization problem difficult to solve directly.

\textcolor{black}{To obtain a tractable surrogate, we optimize directly over steady-state sampling allocations. 
For a fixed behavior policy $\pi$ and transition kernel $x$, the measure 
$\eta_1\in\mathcal F_\pi(x)$ denotes the stationary state-action distribution induced by $\pi$ under $x$. 
When $x=P_{\mathcal M}$, this distribution becomes the long-run sampling ratio under the true MDP. 
Since the empirical transition kernel converges to $P_{\mathcal M}$, the corresponding stationary distribution also approaches its nominal counterpart under the true MDP. 
We therefore replace the policy-dependent stationary distribution by a deterministic allocation vector $\omega$, whose feasible set is characterized by the flow-balance constraints under $P_{\mathcal M}$:
$$
    \mathcal{W} := \left\{\omega\in\Omega:
    \forall s\in \mathcal{S},\ 
    \sum_{a\in\mathcal{A}}\omega_{sa}
    =
    \sum_{s^\prime \in\mathcal{S},a^\prime \in \mathcal{A}}
    P_{\mathcal{M}}(s|s^\prime,a^\prime)\omega_{s^\prime a^\prime}
    \right\}.
$$
This steady-state surrogate focuses on the large-deviation costs of transition and reward estimation errors while fixing the occupation measure at its nominal long-run allocation. 
Such fixed-allocation methods are widely used in simulation optimization \citep{glynn2004large}.}

We then consider the tractable surrogate problem 
\begin{equation}
\label{eq: app rate opt}
    \max_{\omega\in \mathcal{W}} \mathcal{R}(\mathcal{M}, \omega).
\end{equation}

However, problem~\eqref{eq: app rate opt} remains difficult to solve for two reasons. First, the objective has no closed-form expression: without parametric assumptions on the transition and reward models, the rate functions are defined only implicitly. Second, the value function depends nonlinearly on the transition dynamics, which leads to nonlinear constraints and can make the set $\mathcal{E}_{s,a}$ nonconvex. Thus, computing an exact solution is generally intractable.

To overcome these difficulties, we develop a tractable relaxation for problem \eqref{eq: app rate opt}. The key idea is to derive analytical lower bounds for the rate functions $I_1$ and $I_2$, and to replace each set $\mathcal{E}_{s,a}$ with an explicit outer approximation that is easier to optimize over. The resulting relaxation yields a nested optimization problem whose optimal value provides a lower bound on that of the original problem in \eqref{eq: app rate opt} under suitable conditions.   

We begin by making the implicit constraint defining each error set $\mathcal{E}_{s,a}$ explicit. Lemma~\ref{lemma: approximate constraint} provides a tractable necessary condition for this constraint, which induces an explicit superset of $\mathcal{E}_{s,a}$ and serves as a convenient outer approximation in the subsequent analysis. To simplify the exposition, we introduce notation for the discrepancy between an alternative model $\tilde{\mathcal{M}}$ and the nominal model $\mathcal{M}$. For each state-action pair $(s,a)$, define the reward difference $\Delta_{r}(s,a) := r_{\tilde{\mathcal{M}}}(s,a)-r_{\mathcal{M}}(s,a)$, and the transition difference $\Delta_{p}(s^\prime|s,a):= P_{\tilde{\mathcal{M}}}(s^\prime|s,a) - P_{{\mathcal{M}}}(s^\prime|s,a)$. We collect the transition differences into the vector $\Delta_{p}(s,a) := (\Delta_{p}(s^\prime|s,a))_{s^\prime\in\mathcal{S}}$. Let $V^{\pi^*_{\mathcal{M}}}_{{\mathcal{M}}}:= (V^{\pi^*_{\mathcal{M}}}_{{\mathcal{M}}}(s))_{s\in\mathcal{S}}$ denote the value function vector of the optimal policy under the model $\mathcal{M}$. Finally, define the optimality gap of the state-action pair $\Delta_{sa} := V^{\pi^*_{\mathcal{M}}}_{\mathcal{M}}(s)-Q^{\pi^*_{\mathcal{M}}}_{\mathcal{M}}(s,a)$.
\begin{lemma}
\label{lemma: approximate constraint}
If $Q^{\pi^*_{\mathcal{M}}}_{\tilde{\mathcal{M}}}(s,a)>V^{\pi^*_{\mathcal{M}}}_{\tilde{\mathcal{M}}}(s)$ \textcolor{black}{for a fixed $a$}, then the following inequality must hold:
$$
 \frac{1+\gamma}{1-\gamma} \max_{s^\prime\in\mathcal{S}} \left|\Delta_r(s^\prime,\pi^*_{\mathcal{M}}(s^\prime))\right| + \frac{\gamma(1+\gamma)}{1-\gamma} \max_{s^\prime\in\mathcal{S}}\left|\Delta_{p}(s^\prime,\pi^*_{\mathcal{M}}(s^\prime))^\top V^{\pi^*_{\mathcal{M}}}_{{\mathcal{M}}}\right|+\Delta_r(s,a) + \gamma \Delta_{p}(s,a)^\top V^{\pi^*_{\mathcal{M}}}_{{\mathcal{M}}}> \Delta_{sa}.
$$
\end{lemma}

Lemma~\ref{lemma: approximate constraint} gives a tractable condition linking the optimality gap of $(s,a)$ to reward and transition deviations between $\tilde{\mathcal M}$ and $\mathcal M$. Minimizing over the resulting relaxed set yields a rigorous lower bound for the original problem in~\eqref{eq: app rate opt}.

Next, we derive tractable bounds for the rate functions $I_1(x(s,a))$ and $I_2(y(s,a))$.
Since their exact forms are generally unavailable, Lemmas~\ref{lemma: I1 lb} and~\ref{lemma: I2 lb} provide explicit quadratic lower bounds. Our derivation exploits the variational representation of the rate functions combined with Bernstein-type concentration inequalities for the logarithmic moment generating functions. These non-asymptotic bounds depend only on the first two moments, allowing us to replace the implicit objective with a computable surrogate.
\begin{lemma} The rate function $I_1(x(s,a))$ satisfies the inequality:
\label{lemma: I1 lb}
$$
        (\Delta_{p}(s,a)^\top V^{\pi^*_{\mathcal{M}}}_{{\mathcal{M}}})^2 \leq 2\mathbb{V}_{ P_{{\mathcal{M}}}(s,a)}[V^{\pi^*_{\mathcal{M}}}_{\mathcal{M}}]I_1(x(s,a)) +\frac{4\sqrt{2}(\mathbb{V}_{ P_{{\mathcal{M}}}(s,a)}[V^{\pi^*_{\mathcal{M}}}_{\mathcal{M}}])^{\frac{1}{2}}I_1(x(s,a))^{\frac{3}{2}}}{3(1-\gamma)}  + \frac{4I_1(x(s,a))^2}{9(1-\gamma)^2},
$$
where $\mathbb{V}_{ P_{{\mathcal{M}}}(s,a)}[V^{\pi^*_{\mathcal{M}}}_{\mathcal{M}}]$ is the variance of random variable $V^{\pi^*_{\mathcal{M}}}_{\mathcal{M}}(s^\prime)$ with $s^\prime\sim P_{{\mathcal{M}}}(s,a)$.
\end{lemma}

\begin{lemma} The rate function $I_2(y(s,a))$ satisfies the inequality:
\label{lemma: I2 lb}
$$
        \Delta^2_{r}(s,a) \leq 2 \mathbb{V}[R_{\mathcal{M}}(s,a)] I_2(y(s,a)) + \frac{4\sqrt{2}(\mathbb{V}[R_{\mathcal{M}}(s,a)])^{\frac{1}{2}} I_2(y(s,a))^{\frac{3}{2}}}{3}   + \frac{4I_2(y(s,a))^2}{9},
$$
where $\mathbb{V}[R_{\mathcal{M}}(s,a)]$ is the variance of random variable $R_{\mathcal{M}}(s,a)$.
\end{lemma}

Lemma~\ref{lemma: I1 lb} and Lemma~\ref{lemma: I2 lb} explicitly connect the objective function with the quantities appearing in the explicit necessary condition of Lemma~\ref{lemma: approximate constraint}. Building on these results, Theorem~\ref{thm: rate lower bound opt} derives an explicit optimization problem whose optimal value provides a lower bound for the original problem. Intuitively, the optimal sampling ratio in this lower-bound formulation balances the squared optimality gap $\Delta_{sa}^2$ against several uncertainty terms. These uncertainty terms include the variance of the reward $\mathbb{V}[R_{\mathcal{M}}(s,a)]$ and the variance of the value function $\mathbb{V}_{P_{\mathcal{M}}(s,a)}[V^{\pi^*_{\mathcal{M}}}_{\mathcal{M}}]$ for each state-action pair $(s,a)$, as well as the maximum reward variance $\bar{\mathbb{V}}[R_{\mathcal{M}}]:=\max_{s\in \mathcal{S}}\mathbb{V}[R_{\mathcal{M}}(s,\pi_{\mathcal{M}}^*(s))]$ and the maximum value-function variance $\bar{\mathbb{V}}[V^{\pi^*_{\mathcal{M}}}_{\mathcal{M}}]:=\max_{s\in \mathcal{S}}\mathbb{V}_{P_{{\mathcal{M}}}(s,\pi_{\mathcal{M}}^*(s))}[V^{\pi^*_{\mathcal{M}}}_{\mathcal{M}}]$ along the optimal policy.

\begin{theorem}
\label{thm: rate lower bound opt} 
Consider the asymptotic regime of hard instances where the alternative model $\tilde{\mathcal{M}}$ is close to the nominal model $\mathcal{M}$. In this regime, the optimal exponential decay rate in~\eqref{eq: app rate opt} is, at leading order, lower bounded by the optimal value of the following optimization problem:
\begin{equation}
\label{eq: rate lower bound opt}
\max_{\omega\in\mathcal{W}}\min_{s\in\mathcal{S},a\in\mathcal{A}\setminus\{\pi^*_{\mathcal{M}}(s)\}} L_{sa}(\omega,\mathcal M)^{-1},
\end{equation}
where
$$
L_{sa}(\omega,\mathcal M) := \frac{2}{\Delta_{sa}^2} \left[
    \frac{(1+\gamma)^2 \left(\bar{\mathbb{V}}[R_{\mathcal{M}}] + \gamma^2 \bar{\mathbb{V}}[V^{\pi^*_{\mathcal{M}}}_{\mathcal{M}}]\right)}{\omega_o(1-\gamma)^2}
    +
    \frac{\mathbb{V}[R_\mathcal{M}(s,a)] + \gamma^2{\mathbb{V}_{P_{{\mathcal{M}}}(s,a)}[V^{\pi^*_{\mathcal{M}}}_{\mathcal{M}}]}}{\omega_{sa}}
\right],
$$
and $\omega_o := \min_{s^\prime \in \mathcal{S}} \omega_{s^\prime \pi^*_{\mathcal{M}}(s^\prime)}$.
\end{theorem}

\begin{remark}
Theorem~\ref{thm: rate lower bound opt} is derived in the hard-instance regime, where the alternative model $\tilde{\mathcal M}$ is close to the nominal model $\mathcal M$. In this regime, the rate functions $I_1$ and $I_2$ vanish, so the higher-order terms in Lemmas~\ref{lemma: I1 lb} and~\ref{lemma: I2 lb} are negligible. The leading-order behavior is therefore
$$
    I_1(x(s,a))
    \ge
    \frac{\bigl(\Delta_p(s,a)^\top V^{\pi^*_{\mathcal M}}_{\mathcal M}\bigr)^2}
    {2\,\mathbb V_{P_{\mathcal M}(s,a)}[V^{\pi^*_{\mathcal M}}_{\mathcal M}]}
    (1-o(1)),
    \qquad
    I_2(y(s,a))
    \ge
    \frac{\Delta_r(s,a)^2}
    {2\,\mathbb V[R_{\mathcal M}(s,a)]}
    (1-o(1)).
$$
Thus, the original implicit optimization problem~\eqref{eq: app rate opt} admits a tractable leading-order surrogate. Moreover, despite being derived under asymptotic assumptions, solving this surrogate problem allows us to establish robust optimality guarantees.
\end{remark}

Based on Theorem~\ref{thm: rate lower bound opt}, we define a surrogate optimal sampling ratio, denoted by $\tilde{\omega}^*(\mathcal{M})$, as the solution to the following equivalent optimization problem:
\begin{equation}
\label{eq: rate lower bound opt 2}
\min_{\omega\in\mathcal{W}} \max_{s\in\mathcal{S},a\in\mathcal{A}\setminus\{\pi^*_{\mathcal{M}}(s)\}} L_{sa}(\omega,\mathcal M)
\end{equation}
 
Let $\mathcal{C}^*(\mathcal{M})$ denote the set of all optimal solutions to \eqref{eq: rate lower bound opt 2}. Lemma~\ref{lemma: convexity} shows that the surrogate problem in \eqref{eq: rate lower bound opt 2} is a convex program, and thus can be solved efficiently to obtain a global minimizer. Moreover, Lemma~\ref{lemma: convexity} shows that $\mathcal{C}^*(\mathcal{M})$ is convex, a property we will leverage in the convergence analysis of our algorithm.
\begin{lemma}
\label{lemma: convexity}
The optimization problem in~\eqref{eq: rate lower bound opt 2} is a convex program. Moreover, it attains its minimum over $\mathcal{W}$, and every optimal solution $\tilde{\omega}^{*}(\mathcal{M})$ satisfies
$\tilde{\omega}^{*}_{sa}(\mathcal{M})>0$ for all $(s,a)\in\mathcal{S}\times\mathcal{A}$. 
Furthermore, the set of optimal solutions $\mathcal{C}^*(\mathcal{M})$ is convex.
\end{lemma}

\subsection{Lazy One-Step Projected Subgradient Descent Algorithm}
\label{sec: sub-alg}
In this subsection, we develop a reinforcement learning algorithm for efficient data acquisition based on the surrogate optimal sampling ratio $\tilde{\omega}^*(\mathcal{M})$. Turning this surrogate prescription into a practical learning algorithm is technically nontrivial for two reasons. First, the MDP $\mathcal{M}$ is unknown, so $\tilde{\omega}^*(\mathcal{M})$ cannot be evaluated directly and must instead be computed from a continuously refined, data-dependent model estimate, introducing estimation error and instability in the data acquisition process. Second, computing $\tilde{\omega}^*(\mathcal{M})$ requires solving the optimization problem~\eqref{eq: rate lower bound opt 2}. In an adaptive procedure, this problem would need to be re-solved whenever the model estimate changes.

We address the first challenge through an adaptive estimation-optimization loop that couples model learning with data acquisition. Specifically, at each time step $t$, we use all observations collected so far to construct an empirical MDP estimate $\bar{\mathcal{M}}(t)$ via the transition estimator in \eqref{eq: empirical_transitions} and the reward estimator in \eqref{eq: empirical_rewards}. Treating $\bar{\mathcal{M}}(t)$ as a plug-in estimate of $\mathcal{M}$, we solve \eqref{eq: rate lower bound opt 2} to obtain the empirical surrogate-optimal sampling ratio $\tilde{\omega}^*(\bar{\mathcal{M}}(t))$. We then convert this ratio into an explicit data acquisition policy by normalizing across actions within each state:
\begin{equation}
\label{eq: exploration policy}
\pi_{\bar{\mathcal{M}}(t)}^e(a|s) = \frac{\tilde{\omega}^*_{sa}(\bar{\mathcal{M}}(t))}{\sum_{a^\prime \in\mathcal{A}}\tilde{\omega}^*_{sa^\prime}(\bar{\mathcal{M}}(t))}\quad\forall s\in\mathcal{S}, a\in\mathcal{A}.
\end{equation}

Using $\pi_{\bar{\mathcal{M}}(t)}^e$, we collect new data, update the empirical model, and repeat the procedure until the data acquisition budget is exhausted. Upon termination, we compute the final policy by solving for an optimal policy under the terminal model estimate $\bar{\mathcal{M}}(T)$ (e.g., via value iteration or policy iteration), yielding the estimated optimal policy $\hat{\pi}_T$.

Next, we address the computational challenge of solving the optimization problem~\eqref{eq: rate lower bound opt 2} efficiently. Define 
$$
F(\omega,\mathcal{M}) := \max_{s\in\mathcal{S},a\in\mathcal{A}\setminus\{\pi^*_{\mathcal{M}}(s)\}} L_{sa}(\omega,\mathcal M).
$$
By Lemma~\ref{lemma: convexity}, $F(\omega,\mathcal{M})$ is convex in $\omega$. However, because of the outer maximum, it is generally nonsmooth and may fail to be differentiable. A natural approach is therefore projected subgradient descent. Yet solving \eqref{eq: rate lower bound opt 2} to convergence every time the empirical model changes would be computationally prohibitive. To reduce this overhead, we adopt a lazy projected subgradient scheme, in which only one projected subgradient step is performed at a sparse sequence of update times.

Specifically, let $\mathcal{T} = \{t_n\}_{n\geq 1}$ be a deterministic sequence of update times, and let $\Gamma_n:=t_{n+1}-t_{n}$, with a prescribed initial update time $t_1$. We choose the gaps to grow linearly,
\begin{equation}
\label{eq: growth condition}
\Gamma_n = \lceil \tilde c n\rceil
\end{equation}
for some constant $\tilde c>0$. \textcolor{black}{This linear growth condition is used for the convergence analysis: it gives each interval enough time for the empirical sampling ratio to track the current target allocation, while still allowing infinitely many updates. The lazy structure keeps the data acquisition policy fixed on each interval $[t_n,t_{n+1})$, which is crucial for proving convergence of the empirical allocation.} At update time $t_n$, we compute a subgradient under the current empirical model and perform one projected step:
\begin{equation}
\begin{aligned}
\label{eq: subgradient step}
x_{n} &= \Pi_{\mathcal{W}^\epsilon(\bar{\mathcal{M}}(t_n))}(x_{n-1} - \eta_n \bar{g}_{n}),\\
\omega_n &= \frac{n-1}{n}\omega_{n-1} + \frac{1}{n}x_n,
\end{aligned}
\end{equation}
where $\bar{g}_{n} \in \partial_{\omega} F(x_{n-1},\bar{\mathcal{M}}(t_n))$, and $\Pi_{\mathcal{W}^\epsilon(\bar{\mathcal{M}}(t_n))}$ denotes the projection onto the $\epsilon$-restricted feasible set
$$
\mathcal{W}^{\epsilon}(\bar{\mathcal{M}}(t_n)) := \left\{\omega\in\mathcal{W}(\bar{\mathcal{M}}(t_n)):\quad
    \omega_{sa}\geq \epsilon,\quad\forall (s,a)\in\mathcal{S}\times\mathcal{A}\right\},
$$
where $\epsilon>0$ is chosen  sufficiently small so that $\mathcal{W}^{\epsilon}(\bar{\mathcal{M}}(t_n))$ is nonempty and the iterates remain bounded away from the singular boundary.
In addition, $\partial_{\omega}F$ is the sub-differential with respect to $\omega$, and $\eta_n>0$ is a step-size to be specified later. The recursion is initialized from some feasible point $x_0\in\mathcal{W}^{\epsilon}(\bar{\mathcal{M}}(t_1))$. Finally, the Polyak-style averaging that defines $\omega_n$ stabilizes the iterates and facilitates convergence of the overall algorithm.

We use $\omega_n$ as a computational surrogate for the costly optimizer $\tilde{\omega}^*(\bar{\mathcal{M}}(t_n))$ in \eqref{eq: exploration policy}. Specifically, at each update time $t_n$, we replace $\tilde{\omega}^*(\bar{\mathcal{M}}(t_n))$ with $\omega_n$ to construct the data acquisition policy. With a slight abuse of notation, we continue to denote the resulting exploration policy by $\pi^e_{\bar{\mathcal{M}}(t_n)}$. This design preserves an optimization-driven exploration strategy while requiring only one projection and one subgradient evaluation per update. The behavior policy is updated only at times in $\mathcal{T}$ and remains fixed between successive updates:
\begin{equation}
\label{eq: sampling policy}
    \pi_{t} = \begin{cases}
        \epsilon_{t} \pi^u +  (1-\epsilon_{t}) \pi^e_{\bar{\mathcal{M}}(t)}, \quad &\text{if} \quad t=t_n\in \mathcal{T};\\
        \pi_{t-1},\quad& \text{if}\quad t\notin \mathcal{T},
    \end{cases} 
\end{equation}
where $\pi^u$ is the uniform data acquisition policy, defined by $\pi^u(a|s) = 1/A$, and $\epsilon_{t} = t^{-\alpha}$ with $\alpha\in(0,1/2)$ specifying the $\epsilon_t$-greedy exploration schedule. The uniform mixture enforces persistent exploration, which is essential for consistent estimation of $\mathcal{M}$, while the lazy-update mechanism substantially reduces the frequency of expensive optimization steps. A complete description of the procedure is given in Algorithm~\ref{alg:algorithm1}.

\begin{algorithm}[H]
\linespread{0.6}
\selectfont
\SetAlgoLined
\caption{Lazy One-Step Projected Subgradient Descent Algorithm (LazyGradient)}
\label{alg:algorithm1}
\SetKwInOut{Input}{Input}
\SetKwInOut{Output}{Output}
\SetKwInOut{KwInit}{Initialize}
\Input{Total budget $T$, update time set $\mathcal{T}=\{t_n\}_{n\ge 1}$, initialization length $n_0$, and exploration rate $\{\epsilon_t\}_{t\ge 1}$.}
\KwInit {
Collect $n_0$ samples using an initial behavior policy $\pi$ (e.g., uniform or maximum-coverage). Set $t\gets n_0$. Construct the empirical MDP estimate $\bar{\mathcal M}(t)$ from the collected data, update the visitation counts $N(s,a;t)$ for all $(s,a)\in\mathcal S\times\mathcal A$, and compute the estimated optimal policy $\hat{\pi}_t$. Set the update index $n\gets 1$.}
\While{$t<T$}{
\eIf{$t=t_n \in \mathcal{T}$}{
Perform one projected subgradient update to obtain $x_{n}$ and $\omega_{n}$ via \eqref{eq: subgradient step}.\\
Construct the exploration policy $\pi^e_{\bar{\mathcal M}(t)}$ from $\omega_{n}$ by replacing $\tilde{\omega}^*_{sa}(\bar{\mathcal{M}}(t))$ in~\eqref{eq: exploration policy} with $\omega_n$.\\
Update the behavior policy $\pi_t$ using \eqref{eq: sampling policy}.\\
Set $n\gets n+1$.
}{
Retain the previous behavior policy: $\pi_t \gets \pi_{t-1}$.
}
Sample action $a_t\sim \pi_t(\cdot\mid s_t)$ and observe the reward and next state.\\
Set $t\gets t+1$.\\
Update the visitation counts $N(s,a;t)$ and the empirical MDP estimate $\bar{\mathcal M}(t)$.\\
Compute the estimated optimal policy $\hat{\pi}_t$ via value iteration (or policy iteration).
}
\Output{Empirical optimal policy $\hat{\pi}_{T}$.}
\end{algorithm}

\subsection{Optimality Analysis}
\label{sec: algorithm analysis}
We now turn to the optimality analysis of the lazy one-step projected subgradient descent algorithm. A key technical prerequisite is \emph{consistency} of the data collection process, namely,  that every state-action pair is sampled infinitely often under the behavior policy.  We first show in Lemma~\ref{lemma: consistency} that this property indeed holds under our algorithm. Intuitively, the result follows from two ingredients: the ergodicity assumption guarantees recurrent state visitation under admissible behavior policies, while the algorithm maintains a nonzero, though decaying, exploration probability, ensuring that every action is selected infinitely often whenever its state is visited.
\begin{lemma}
\label{lemma: consistency}
Assume that the MDP $\mathcal{M}$ is ergodic. Then, Algorithm~\ref{alg:algorithm1} ensures that every state-action pair is visited infinitely often almost surely:
\begin{equation*}
    \lim_{t \rightarrow \infty} N(s,a;t) = \infty, \quad \forall s \in \mathcal{S}, a \in \mathcal{A}.
\end{equation*}
\end{lemma}

For a given interaction budget $T$, let $N(T)$ denote the total number of updates performed, defined as the unique integer satisfying:
\begin{equation*}
N(T):=\max\Big\{n\ge 0:\ \sum_{k=1}^n \Gamma_k \le T\Big\},
\end{equation*}
\textcolor{black}{Theorem~\ref{thm: sub-grad converge} shows that the lazy one-step projected subgradient scheme converges to the optimal surrogate value $F^*$. The key insight is that, although the data-acquisition policy is updated only at sparse times, the intervals between updates are long enough for the empirical sampling ratio to track the current target allocation. At the same time, as more data are collected, the empirical MDP $\bar{\mathcal M}(t_n)$ becomes increasingly accurate, so the projected subgradient updates asymptotically behave like updates for the true optimization problem. The main technical challenge is therefore to control the interaction between sampling error and optimization error in this adaptive, data-dependent procedure.}

\begin{theorem}
\label{thm: sub-grad converge}
Assume that the MDP $\mathcal{M}$ is ergodic. The sequence of average iterates $\{\omega_n\}_{n\geq 1}$, generated by the projected subgradient updates in \eqref{eq: subgradient step} with step size $\eta_n = 1/\sqrt{N(T)}$, ensures that:
\begin{equation*}
F(\omega_{n},\mathcal{M})\rightarrow F^*,
\end{equation*}
almost surely as $n\rightarrow \infty$, where $F^*$ is the optimal value of the optimization problem \eqref{eq: rate lower bound opt 2}.
\end{theorem}

Lemma~\ref{lemma: ratio convergence} establishes that, as $n\rightarrow\infty$, the empirical sampling distribution over
$[t_n,t_{n+1})$ converges to the stationary distribution of the time-homogeneous Markov chain induced by the behavior policy; see Appendix \ref{sec: ratio converge} for the detailed proof.

\begin{lemma}
\label{lemma: ratio convergence}
Under the growth condition \eqref{eq: growth condition}, we have 
\begin{equation*}
\max_{(s,a) \in \mathcal{S}\times \mathcal{A}} \left| \frac{1}{\Gamma_n} N^n(s,a) - \beta^n(s,a) \right| \rightarrow 0,
\end{equation*}
almost surely as $n\to\infty$, where $N^n(s,a)$ denotes the number of visits to $(s,a)$ during the interval $[t_{n}, t_{n+1})$, and $\beta^n$ denotes the stationary distribution of the Markov chain induced by the behavior policy $\pi_{t_n}$.
\end{lemma}

Theorem~\ref{thm: robust optimality} provides the main performance guarantee for Algorithm~\ref{alg:algorithm1}: the behavior policy produced by the lazy one-step projected subgradient updates is near–robustly optimal, up to a multiplicative $1-\gamma$ factor; see Appendix \ref{sec: robust optimality} for the detailed proof. 

\begin{theorem}
\label{thm: robust optimality}
Under the same condition in Theorem \ref{thm: rate lower bound opt}, the lazy one-step projected subgradient descent algorithm (Algorithm~\ref{alg:algorithm1}) is near–robustly optimal up to a multiplicative factor of $1-\gamma$.
\end{theorem}

\section{Extensions to Large-Scale Reinforcement Learning}
\label{sec: large scale}
In many practical applications, the state space size $S$ and action space size $A$ can be extremely large, leading to substantial computational costs when updating the MDP model at each iteration. In this section, we extend our algorithm and theoretical results to such large-scale MDP settings.

A common approach to addressing large-scale reinforcement learning problems is to employ function approximation. For instance, in order to make Q-learning scalable to large MDPs, the Q-function can be parameterized using a neural network, as in deep Q-learning (DQN) \citep{mnih2015human}. Although neural networks are highly expressive and can represent complex functions, we adopt a linear function approximation framework in order to gain theoretical insights into the behavior of function approximation methods. The notion of a linear MDP is formally introduced in Definition \ref{def: linear mdp}.

\begin{definition}[Linear MDP]
\label{def: linear mdp}
An MDP $\mathcal{M}$ is a linear MDP with a feature map $\phi: \mathcal{S}\times\mathcal{A}\rightarrow \mathbb{R}^{d}$, if there exist a family of unknown (signed) measures $\mu_{\mathcal{M}}\in\mathbb{R}^{d\times S}$ over $\mathcal{S}$,
 and an unknown vector $\theta_{\mathcal{M}}\in\mathbb{R}^{d}$, such that for any $(s,a)\in\mathcal{S}\times\mathcal{A}$, we have
$$
    P_{\mathcal{M}}(s^\prime|s,a) = \phi(s,a)^\top \mu_{\mathcal{M}}(s^\prime),\quad r_{\mathcal{M}}(s,a) = \phi(s,a)^\top \theta_{\mathcal{M}}.
$$
Without loss of generality, we assume $\lVert \phi(s,a)\rVert\le 1$ for all $(s,a)\in\mathcal{S}\times \mathcal{A}$, $\lVert \sum_{s\in \mathcal{S}}|\mu_{\mathcal{M}}(s)| \rVert\le \sqrt{d}$, and $\lVert \theta_{\mathcal{M}} \rVert \le \sqrt{d}$.
\end{definition}

Linear MDPs are widely used in the reinforcement learning literature as a tractable framework for studying function approximation methods \citep{jin2020provably,hu2022nearly}. The linear MDP assumption offers several advantages. First, it includes classical tabular MDPs as a special case, where each state-action pair is represented by a distinct basis feature. Second, linear MDPs admit more tractable theoretical analysis than neural-network-based models. Finally, by connections to kernel methods and Mercer’s theorem \citep{williams2006gaussian}, linear representations in possibly infinite-dimensional feature spaces can capture a broad class of smooth functions.

To obtain a closed-form characterization of the exponential decay rate function for linear MDPs, we build on ideas from the classical MDP setting. However, the linear MDP assumption introduces significant technical challenges in characterizing the exponential decay rate. 
In particular, unlike classical MDPs, state-action pairs in linear MDPs are no longer independent: their rewards and transition dynamics are coupled through the unknown global parameters $\mu_{\mathcal{M}}$ and $\theta_{\mathcal{M}}$. As a result, analytical techniques developed for classical MDPs do not directly apply.

Lemma \ref{lemma: linear rate} establishes a local quadratic lower bound for the rate function $I_1(x(s,a))$. To derive a tractable bound, we upper-bound the log-moment generating function of $X_{\mathcal{M}}(s,a)$ via a Taylor expansion, with the higher-order moments controlled through bounds in terms of the infinity norm. The result then follows by solving the associated dual optimization problem.

\begin{lemma} 
\label{lemma: linear rate}
For the linear MDP $\mathcal{M}$, for any nonnegative bounded vector
$v\in\mathbb{R}^{|\mathcal S|}$, the rate function $I_1(x(s,a))$ satisfies
\begin{equation}
     I_1(x(s,a)) \ge \frac{(\gamma v^\top (x(s,a)-P_{\mathcal{M}}(s,a)))^2}{3\lVert \gamma v^\top X_{\mathcal{M}}(s,a)\rVert^2_{\infty}}.
\end{equation}
\end{lemma}

Using the sub-Gaussian property of $R_{\mathcal{M}}(s,a)$, we obtain a quadratic lower bound for the rate function $I_2(y(s,a))$:
\begin{equation*}
    I_2(y(s,a)) \ge \frac{(y(s,a)-r_{\mathcal{M}}(s,a))^2}{2}.
\end{equation*}
Combining this with the quadratic lower bound for $I_1(x(s,a))$, we conclude that the weighted rate function admits a quadratic lower bound in terms of the linear MDP parameters (see the proof
of Theorem~\ref{thm: rate lower bound opt}):
$$
\sum_{s^\prime \in \mathcal{S}, a^\prime \in \mathcal{A}}\omega_{s^\prime a^\prime}\left(I_1(x(s^\prime,a^\prime))+I_2(y(s^\prime,a^\prime))\right) \ge \frac{(1-\gamma)^2}{6} \lVert \theta_{\tilde{\mathcal{M}}}-\theta_{\mathcal{M}}+\gamma(\mu_{\tilde{\mathcal{M}}}-\mu_{{\mathcal{M}}}) V_{\tilde{\mathcal{M}}}^{\pi^*_{\mathcal{M}}} \rVert^2_{\Lambda(\omega)},$$
where \begin{equation*}
   \Lambda(\omega) =  \sum_{s^\prime \in \mathcal{S}, a^\prime \in \mathcal{A}} \omega_{s^\prime a^\prime} \phi(s^\prime,a^\prime) \phi(s^\prime,a^\prime)^\top.
\end{equation*}
is the design matrix associated with the linear MDP representation, and
$\|z\|_{\Lambda(\omega)}^2 := z^\top \Lambda(\omega) z$.

Lemma \ref{lemma: linear constraint} provides a tractable necessary condition
under which the error event
$Q^{\pi^*_{\mathcal{M}}}_{\tilde{\mathcal{M}}}(s,a)>
V^{\pi^*_{\mathcal{M}}}_{\tilde{\mathcal{M}}}(s)$ occurs. \textcolor{black}{Intuitively, this event can occur only if the perturbation from the true MDP $\mathcal M$ to the alternative MDP $\tilde{\mathcal M}$ is large enough to overcome the value gap $\Delta_{sa}$.}  

\begin{lemma}
\label{lemma: linear constraint}
For the linear MDP $\mathcal{M}$, if $Q^{\pi^*_{\mathcal{M}}}_{\tilde{\mathcal{M}}}(s,a)>V^{\pi^*_{\mathcal{M}}}_{\tilde{\mathcal{M}}}(s)$, then the following condition holds:
$$
    \Delta_{sa} \le \left( \lVert \phi(s,a)-\phi(s,\pi_{\mathcal{M}}^*(s)) \rVert_{\Lambda(\omega)^{-1}}+\frac{2\gamma}{1-\gamma} \max_{(s,a)\in\mathcal{S}\times \mathcal{A}}\lVert \phi(s,a)\rVert_{\Lambda(\omega)^{-1}}\right) 
\lVert \zeta\rVert_{\Lambda(\omega)},
$$
where $\zeta = \theta_{\tilde{\mathcal{M}}}-\theta_{\mathcal{M}}+\gamma(\mu_{\tilde{\mathcal{M}}}-\mu_{{\mathcal{M}}}) V_{\tilde{\mathcal{M}}}^{\pi^*_{\mathcal{M}}}$.
\end{lemma}

Theorem \ref{thm: linear rate lower bound opt} establishes a closed-form lower bound on the optimal exponential decay rate from \eqref{eq: app rate opt}. This formulation reduces the variational problem to a constrained quadratic optimization, whose optimal value yields the explicit lower bound.

\begin{theorem}
\label{thm: linear rate lower bound opt} 
For the linear MDP $\mathcal{M}$, the optimal decay rate in~\eqref{eq: app rate opt} is lower bounded by the optimal value of the following optimization problem:
$$
\max_{\omega\in\mathcal{W}}\min_{s\in\mathcal{S},a\in\mathcal{A}\setminus\{\pi^*_{\mathcal{M}}(s)\}}\frac{(1-\gamma)^2}{6} \left(\frac{\Delta_{sa}}{\big\|
\phi(s,a)-\phi\!\left(s,\pi_{\mathcal{M}}^*(s)\right)
\big\|_{\Lambda(\omega)^{-1}}
+
\frac{2\gamma}{1-\gamma}
\max_{(s^\prime,a^\prime)\in\mathcal{S}\times\mathcal{A}}
\big\|
\phi(s^\prime,a^\prime)
\big\|_{\Lambda(\omega)^{-1}}}\right)^2.
$$
\end{theorem}

We can solve the optimization problem in Theorem \ref{thm: linear rate lower bound opt} to obtain the optimal sampling ratio. However, the complicated denominators
in the objective make the problem computationally challenging. To address this issue, we consider two surrogate optimization problems that provide computationally tractable approximations. Specifically, we can solve
$$
\max_{\omega\in\mathcal{W}}\min_{s\in\mathcal{S},a\in\mathcal{A}\setminus\{\pi^*_{\mathcal{M}}(s)\}} \frac{\Delta^2_{sa}}{\big\|
\phi(s,a)-\phi\!\left(s,\pi_{\mathcal{M}}^*(s)\right)
\big\|^2_{\Lambda(\omega)^{-1}}}
$$
or alternatively,
\begin{equation*}
\max_{\omega\in\mathcal{W}}\min_{s\in\mathcal{S},a\in\mathcal{A}\setminus\{\pi^*_{\mathcal{M}}(s)\}} \frac{\Delta^2_{\min}}{\big\|
\phi(s,a)
\big\|^2_{\Lambda(\omega)^{-1}}}.
\end{equation*}
Both surrogate problems are convex and can be efficiently solved using off-the-shelf convex optimization solvers or the projected subgradient descent method developed in Section~\ref{sec: sub-alg}. The first
surrogate becomes more accurate as $\gamma\to 0$, since the propagated-error
term vanishes in this regime. Although these
surrogates are computationally tractable, establishing their optimality
guarantees remains an open problem, which we leave for future work. 

\section{Numerical Experiments}
\label{sec: case study}
In this section, we compare the data acquisition efficiency of different reinforcement learning algorithms using a standard Gridworld example and an operational experiment design case study. We conduct a thorough comparison with state-of-the-art model-based and model-free reinforcement learning algorithms:
\begin{itemize}
    \item \textbf{Q-learning} \citep{sutton1998reinforcement}: a canonical tabular, model-free, value-based baseline.
    \item \textbf{Actor--critic} \citep{konda1999actor}: a basic policy-gradient baseline equipped with a learned critic.
    \item \textbf{PPO/TRPO} \citep{schulman2017proximal,schulman2015trust}: widely used actor--critic methods that stabilize policy updates via proximal or trust-region constraints.
    \item \textbf{PSRL} \citep{osband2013more}: a principled model-based exploration baseline based on posterior sampling over MDP models.
    \item \textbf{QOCBA} \citep{zhu2024uncertainty}: an optimization-guided, model-based algorithm that allocates data adaptively using uncertainty-aware sampling rules.
\end{itemize}
These benchmarks are widely used in the reinforcement learning literature and cover complementary algorithmic paradigms, including value-based and policy-based learning as well as heuristic and Bayesian exploration. As a result, they provide representative and appropriate reference points for comparing data acquisition efficiency. Our proposed Algorithm~\ref{alg:algorithm1}, referred to as \emph{LazyGradient}, is a model-based method.

\subsection{Gridworld}
We consider a stochastic Gridworld with $S=16$ states and $A=4$ actions under a limited interaction budget. Each state has a unique preferred action: choosing it moves the agent toward the goal with high probability, while other actions tend to move away and markedly increase the chance of staying in place. We add a small uniform mixing term to keep the Markov chain communicating; the goal state is nearly absorbing but occasionally transitions back to the start.

For non-goal states, the mean reward is a small negative baseline plus a progress-based shaping term, with small bonuses for taking the preferred action and for reaching the goal; \textcolor{black}{it is clipped to $[-0.08,0.20]$ to keep rewards bounded and make the problem sufficiently challenging.} At the goal, the mean reward is $1.0$. Rewards are observed with zero-mean Gaussian noise whose variance is in $[0.006, 0.02]$.  The discount factor is set to $\gamma=0.99$.

We run $50$ independent macro replications for each algorithm under different interaction budgets. Table~\ref{tab:gridworld_pcs_pm_90ci} reports the empirical PCS for the optimal policy after the sampling budget $T$ is fully exhausted, along with the corresponding $90\%$ confidence intervals. LazyGradient is consistently the most data-efficient method: it achieves high PCS already at $T=600$ and exceeds $0.9$ by $T=800$, reaching perfect selection at $T\ge 1000$. In contrast, QOCBA and PSRL improve more gradually and only approach high reliability at much larger budgets, while the standard model-free baselines (Q-learning, Actor-Critic, PPO, TRPO) remain far below in this range. Overall, the table shows that LazyGradient attains near-100\% correct selection with much fewer interactions than competing model-based and model-free alternatives.
\begin{table}[htbp]
\centering
\small
\caption{Probability of correct selection at selected budgets (mean $\pm$ 90\% CI)}
\label{tab:gridworld_pcs_pm_90ci}
\begin{tabular}{lcccc}
\hline
Algorithm & $T=600$ & $T=800$ & $T=1000$ & $T=1200$ \\
\hline
Q-Learning   & $0.04 \pm 0.046$ & $0.02 \pm 0.033$ & $0.12 \pm 0.076$ & $0.16 \pm 0.085$ \\
Actor-Critic & $0.02 \pm 0.033$ & $0.02 \pm 0.033$ & $0.08 \pm 0.063$ & $0.26 \pm 0.102$ \\
PPO          & $0.00 \pm 0.000$ & $0.00 \pm 0.000$ & $0.10 \pm 0.070$ & $0.16 \pm 0.085$ \\
TRPO         & $0.00 \pm 0.000$ & $0.00 \pm 0.000$ & $0.00 \pm 0.000$ & $0.04 \pm 0.046$ \\
PSRL         & $0.02 \pm 0.033$ & $0.18 \pm 0.089$ & $0.38 \pm 0.113$ & $0.56 \pm 0.115$ \\
QOCBA & 0.46 $\pm$ 0.116 & 0.70 $\pm$ 0.107 & 0.92 $\pm$ 0.063 & 0.92 $\pm$ 0.063 \\
\textbf{LazyGradient} & $\mathbf{0.78 \pm 0.096}$ & $\mathbf{0.92 \pm 0.063}$ & $\mathbf{1.00 \pm 0.000}$ & $\mathbf{1.00 \pm 0.000}$ \\
\hline
\end{tabular}
\end{table}

Figure \ref{fig: policy value} compares the state-averaged policy values learned under two interaction budgets. LazyGradient achieves the highest return at both $T=600$ and $T=800$, and its confidence intervals are tight, indicating stable performance across replications. In particular, it already reaches near-optimal value at $T=600$ and improves further by $T=800$, while the closest competitor, QOCBA, remains noticeably lower at both budgets. The remaining baselines perform worse, indicating slower policy improvement under the same interaction budget. Overall, the figure shows that LazyGradient attains higher-quality policies with fewer interactions than competing model-based and model-free methods, demonstrating superior sample efficiency.
\begin{figure}
    \centering
    \includegraphics[width=0.98\linewidth]{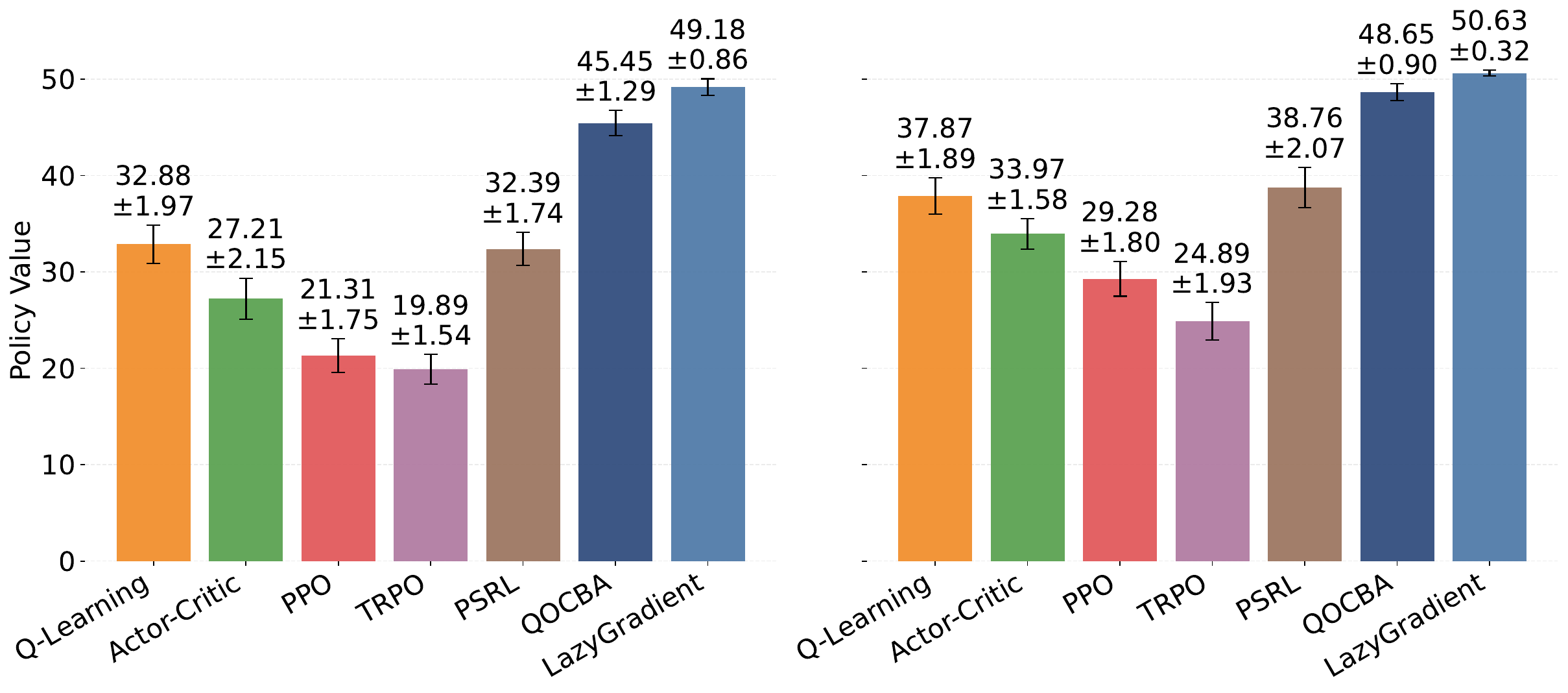}
    \caption{Policy Value Comparison, State-Averaged. $T=600$ (left) vs. $T=800$ (right). 90\% CI}
    \label{fig: policy value}
\end{figure}
In LazyGradient, we first run an initialization policy to collect data and obtain an initial MDP model; a better initial estimate leads to more efficient downstream data acquisition. \textcolor{black}{For initialization}, we propose a heuristic-maximum-coverage policy: at each state, it selects the action that maximizes a weighted coverage score that favors under-visited $(s,a)$ pairs, actions that are likely to transition to under-visited states, and actions that target globally rare states; ties are broken uniformly at random. We fix the initialization budget to $n_0=450$ and compare this policy with a uniform policy over 50 independent macro replications. Figure \ref{fig:warmup_method} shows that maximum-coverage initialization yields a clear and consistent improvement in PCS across all budgets, with the largest gap at smaller $T$. In contrast, the state-averaged policy value changes only marginally, and the two policies are often close within the confidence intervals. Notably, even with uniform initialization, the resulting performance remains higher than that of all benchmark algorithms.
\begin{figure}
    \centering
    \includegraphics[width=0.98\linewidth]{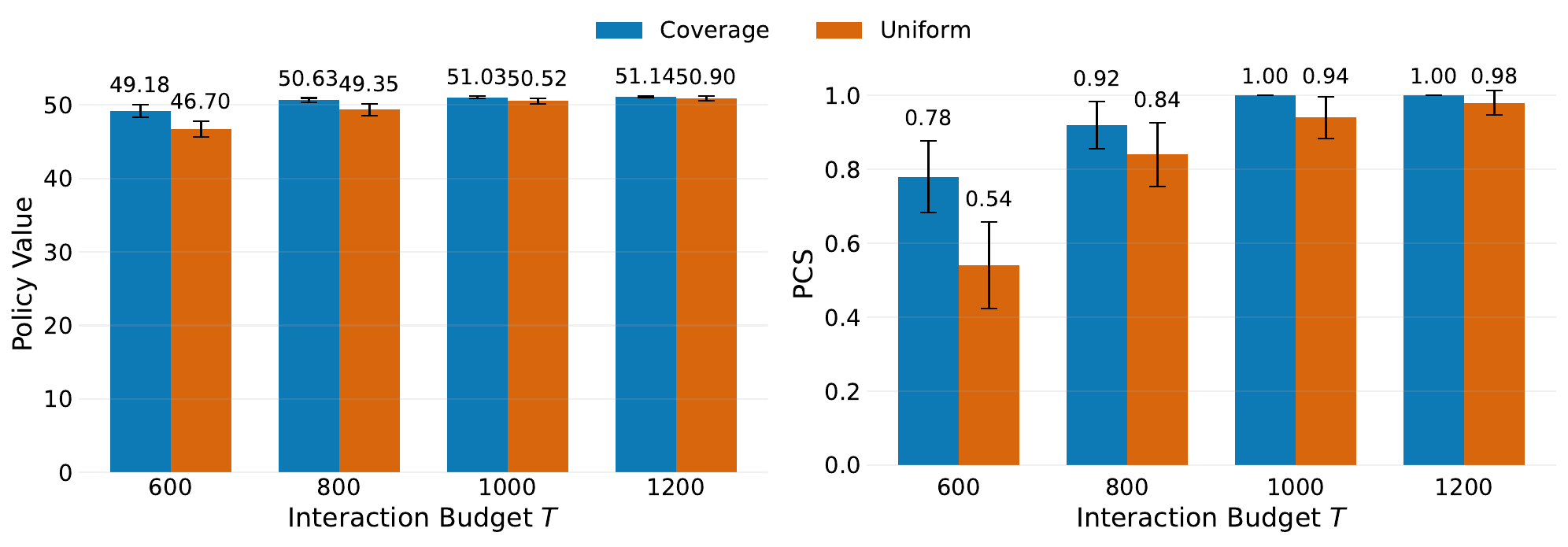}
    \caption{Comparison of Different Initialization Policies. 90\% CI}
    \label{fig:warmup_method}
\end{figure}

Figure~\ref{fig:warmup_step} compares performance across different initialization budgets $n_0$ over 50 independent macro replications. The results show that varying $n_0$ has only a modest effect: both the state-averaged policy value and PCS exhibit similar trends across all tested $n_0$, and the curves remain close within the confidence intervals. This indicates that the method is not sensitive to the exact choice of initialization budget in this range. Moreover, for every $n_0$, the resulting performance is consistently strong across budgets $T$ and remains above the benchmark algorithms reported earlier, demonstrating robust data acquisition efficiency.
\begin{figure}
    \centering
    \includegraphics[width=0.98\linewidth]{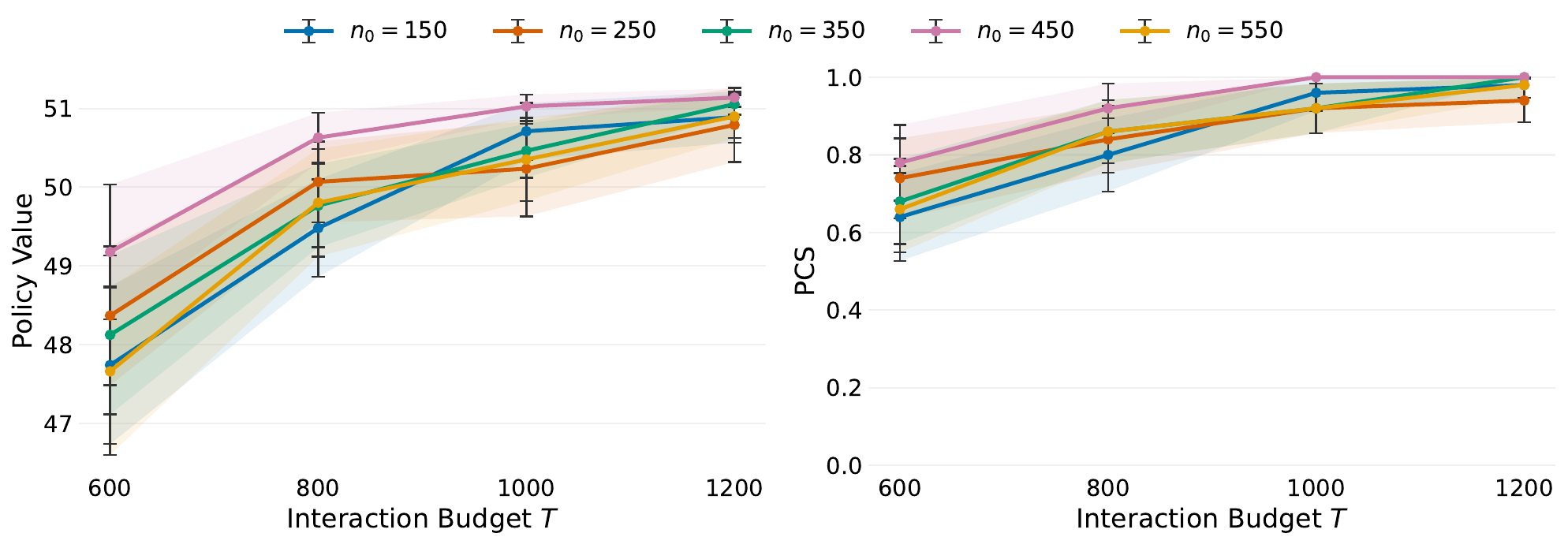}
    \caption{Effect of Initialization Budget $n_0$. 90\% CI}
    \label{fig:warmup_step}
\end{figure}

\subsection{Case Study: Operational Experiment Design}
We study a product-launch experiment in which a company interacts with users over a limited number of periods to identify a high-performing selling policy prior to full deployment. In each period, the platform observes a coarse market context and selects an operational action, such as a discount level and recommendation exposure. Because interactions consume real traffic and operational resources and may affect user experience, the experiment is a pure-exploration problem: given a fixed budget, the goal is to identify a near-optimal stationary policy for deployment. This setting is motivated by online experimentation and RL-based decision systems in e-commerce and digital platforms \citep{liu2023dynamic}, where pricing, promotions, and exposure decisions must be optimized under limited experimentation budgets. 

We model this operational problem as an infinite-horizon discounted tabular MDP. The environment has $S=50$ states, each representing a coarse market context that summarizes user-segment composition, early engagement and retention signals, and inventory pressure. States are obtained by discretizing these operational signals into finitely many bins and mapping the resulting context cells to $S$ tabular states. The action space has $A=30$ actions, each corresponding to a discount level paired with a recommendation-exposure intensity; we discretize discount and exposure levels to form a structured action grid and map the resulting combinations to $A$ actions. Rewards and transitions are generated by a simple, behaviorally motivated mechanism. For each $(s,a)$, we compute an action-dependent purchase-likelihood score from the current context and the chosen discount and exposure. The mean reward $r(s,a)$ is then defined as expected profit, increasing with the likelihood of purchase and per-sale margin, and decreasing with exposure cost, inventory-pressure cost, and a mild penalty for context–action mismatch, and is clipped to $[0.06, 0.40]$. The realized reward is noisy: $R(s,a) = r(s,a)+\xi_{s,a}$, where $\xi_{s,a}\sim\mathcal{N}(0,\sigma^2(s,a))$ captures residual demand and measurement uncertainty. The transition kernel $P(\cdot|s,a)$ is time-homogeneous and stochastic and couples the evolution of user-segment composition, engagement and retention, and inventory pressure: discount and exposure affect how segment mix and engagement drift over time, while inventory pressure updates through action-dependent demand induced by the current context. The discount factor is set to $\gamma=0.99$.

Because the product-launch instance has $1500$ state–action pairs, far more than the Gridworld instance (64), QOCBA is not practical to run: the resulting optimization problem is too large, and the solver does not terminate within a reasonable time. We therefore exclude QOCBA and report results for the remaining benchmarks. We run $5$ independent macro replications for each algorithm with interaction budget $T=2000$. Figure~\ref{fig:operation_exp_dense} reports the state-averaged policy values after $T=2000$ interactions. LazyGradient achieves the highest return among all methods and outperforms the competing benchmarks by a noticeable margin. This advantage is stable within the reported confidence intervals, indicating more reliable policy improvement under the same interaction budget. Overall, the results demonstrate the superior sample efficiency of LazyGradient on the operational experiment design instance.

\begin{figure}
    \centering
    \includegraphics[width=0.6\linewidth]{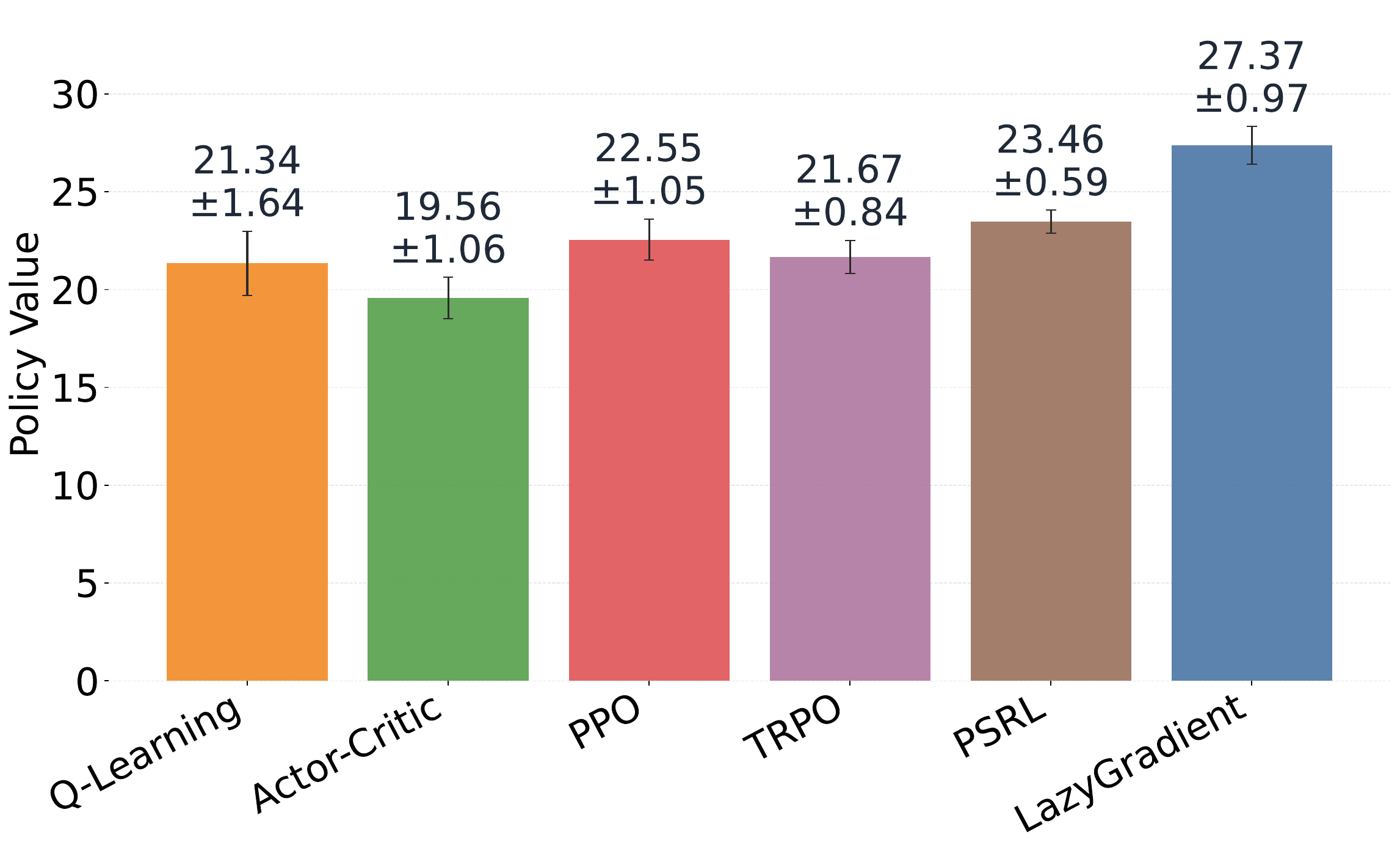}
    \caption{Policy Value Comparison, State-Averaged. $T=2000$. 90\% CI}
    \label{fig:operation_exp_dense}
\end{figure}

Figure \ref{fig:sclae_law} compares the state-averaged policy values across different state sizes. Across all tested numbers of states, LazyGradient consistently attains the highest policy value, and its advantage remains stable as the state space grows. This robustness, together with the relatively tight confidence intervals, suggests that LazyGradient is more data efficient than the baseline methods in this scaling regime.
\begin{figure}
    \centering
    \includegraphics[width=0.6\linewidth]{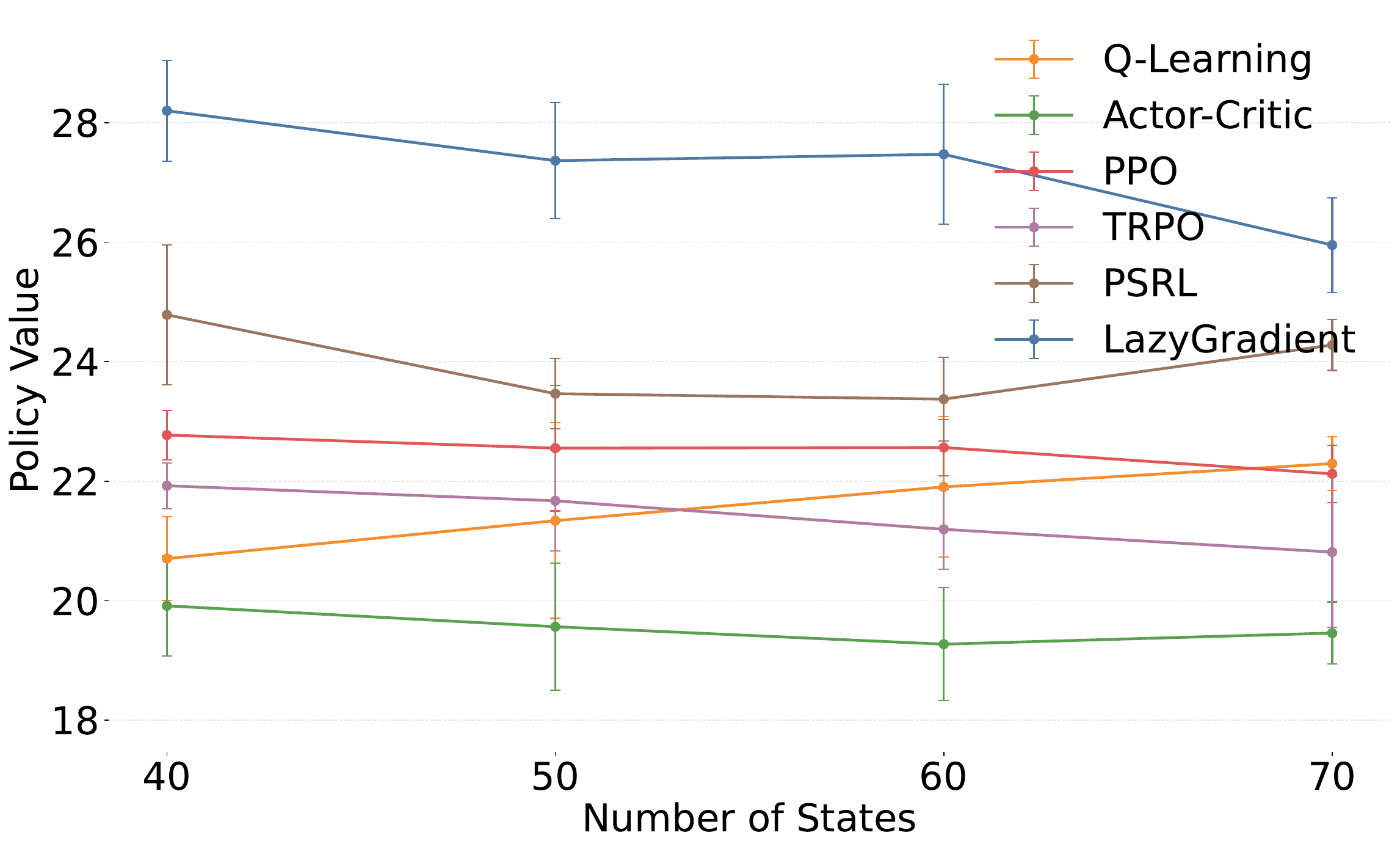}
    \caption{Policy Value under Different State Sizes, State-Averaged. $T=2000$. 90\% CI}
    \label{fig:sclae_law}
\end{figure}
\section{Conclusions}
\label{sec: conslusion}
In this paper, we study optimal data acquisition for infinite-horizon reinforcement learning from a large deviations perspective. We propose the exponential decay rate of the PFS as a principled efficiency metric and derive a variational characterization of this rate using large deviations theory. Building on this characterization, we introduce two complementary notions of optimality and develop a lazy one-step projected subgradient algorithm. We show that the resulting method is near–robustly optimal up to a constant-factor loss. Finally, we extend our analysis to the linear function-approximation setting and validate the proposed approach through a Gridworld example and an operational experiment design study. Several directions remain open. In particular, designing reinforcement learning algorithms that are robustly optimal, as well as methods that attain exact optimality, requires new technical tools and poses substantial analytical challenges. Addressing these problems is an important step toward fully optimizing data acquisition in reinforcement learning.

% NOTE: Use the Code and Data Disclosure section to provide instructions for where your code and data, along with README file, can be found. If the paper received an exemption, then state the reason the exemption was granted.

% \section{Code and Data Disclosure}\label{sec:Code and Data Disclosure}The code and data to support the numerical experiments in this paper can be found at URL.

%\THEEndNotes
% \begingroup \parindent 0pt \parskip 0.0ex \def\enotesize{\normalsize} \theendnotes \endgroup

% Appendix here
% Options are (1) APPENDIX (with or without general title) or
%             (2) APPENDICES (if it has more than one unrelated sections)
% Outcomment the appropriate case if necessary
%
% \begin{APPENDIX}{<Title of the Appendix>}
% \end{APPENDIX}
%
%   or
%
% \begin{APPENDICES}
% \section{<Title of Section A>}
% \section{<Title of Section B>}
% etc
% \end{APPENDICES}

% Acknowledgments here
\ACKNOWLEDGMENT{Jian-Qiang Hu and Mingjie Hu are supported by the
National Natural Science Foundation of China (NSFC) under grants 72033003, 72350710219, 72342006,
and 72293565.}

% References here (outcomment the appropriate case)

% CASE 1: BiBTeX used to constantly update the references
%   (while the paper is being written).
%\bibliographystyle{informs2014} % outcomment this and next line in Case 1
%\bibliography{<your bib file(s)>} % if more than one, comma separated

\bibliographystyle{informs2014} % outcomment this and next line in Case 1
\bibliography{sample} % if more than one, comma separated

% CASE 2: BiBTeX used to generate mypaper.bbl (to be further fine tuned)
%\input{mypaper.bbl} % outcomment this line in Case 2

%If you don't use BiBTex, you can manually itemize references as shown below.

%\bibliographystyle{nonumber}

\ECSwitch
\ECHead{Proofs of Statements}
This document provides proofs for the theoretical statements in the paper “Optimal Data Acquisition for Reinforcement Learning: A Large Deviations Perspective."

\section{Proof of Theorem~\ref{thm: rate function}}

\proof{Proof of Lemma~\ref{lemma: rate function G}} 
Define $\mathcal{E}_1 = \{\hat{\pi}_{T} \neq \pi^*_{\mathcal{M}}\}$ and 
\begin{equation*}
   \mathcal{E}_2 = \left\{\exists s\in \mathcal{S}, a\in \mathcal{A}\setminus\{\pi^*_{\mathcal{M}}(s)\}, Q^{\pi^*_{\mathcal{M}}}_{\bar{\mathcal{M}}(T)}(s,a)>V^{\pi^*_{\mathcal{M}}}_{\bar{\mathcal{M}}(T)}(s)\right\}.
\end{equation*}
We show that $\mathbb{P}(\mathcal{E}_1) = \mathbb{P}(\mathcal{E}_2)$.

We first establish that $\mathcal{E}_1\Rightarrow \mathcal{E}_2$. By the definition of $\mathcal{E}_1$, $\pi^*_{\mathcal{M}}$ is not optimal for the MDP $\bar{\mathcal{M}}(T)$. We claim that there must exist some state $s\in\mathcal S$ and some action $a\in\mathcal A\setminus\{\pi^*_{\mathcal M}(s)\}$ such that
$
Q^{\pi^*_{\mathcal M}}_{\bar{\mathcal M}(T)}(s,a)
>
V^{\pi^*_{\mathcal M}}_{\bar{\mathcal M}(T)}(s).
$

Indeed, suppose otherwise. Then for every $s\in\mathcal S$ and every $a\neq \pi^*_{\mathcal M}(s)$,
$
Q^{\pi^*_{\mathcal M}}_{\bar{\mathcal M}(T)}(s,a)
\le
V^{\pi^*_{\mathcal M}}_{\bar{\mathcal M}(T)}(s).
$
Moreover, by definition,
$
Q^{\pi^*_{\mathcal M}}_{\bar{\mathcal M}(T)}(s,\pi^*_{\mathcal M}(s))
=
V^{\pi^*_{\mathcal M}}_{\bar{\mathcal M}(T)}(s),
\forall s\in\mathcal S.
$
Hence,
$
\max_{a\in\mathcal A}
Q^{\pi^*_{\mathcal M}}_{\bar{\mathcal M}(T)}(s,a)
=
V^{\pi^*_{\mathcal M}}_{\bar{\mathcal M}(T)}(s), \forall s\in\mathcal S,
$
which shows that $\pi^*_{\mathcal M}$ is greedy with respect to its own action-value function in $\bar{\mathcal M}(T)$. Therefore, $\pi^*_{\mathcal M}$ satisfies the Bellman optimality equation in $\bar{\mathcal M}(T)$, and hence is an optimal policy for $\bar{\mathcal M}(T)$. This contradicts $\mathcal E_1$. Consequently, $\mathcal E_2$ must hold.

We next establish that $\mathcal{E}_2\Rightarrow \mathcal{E}_1$. By the definition of $\mathcal{E}_2$, we have there exists a state $s\in \mathcal{S}$ and action $a\in\mathcal{A}\setminus\{\pi^*_{\mathcal{M}}(s)\}$ such that
$
Q^{\pi^*_{\mathcal{M}}}_{\bar{\mathcal{M}}(T)}(s,a)>V^{\pi^*_{\mathcal{M}}}_{\bar{\mathcal{M}}(T)}(s).
$

Define the greedy policy
$$
\pi^g(x)\in \arg\max_{b\in\mathcal A}
Q^{\pi^*_{\mathcal{M}}}_{\bar{\mathcal{M}}(T)}(x,b),
\qquad \forall x\in\mathcal S.
$$

Then for every $x\in\mathcal S$,
$
Q^{\pi^*_{\mathcal{M}}}_{\bar{\mathcal{M}}(T)}(x,\pi^g(x))
\ge
V^{\pi^*_{\mathcal{M}}}_{\bar{\mathcal{M}}(T)}(x),
$ and the above strict inequality implies that this inequality is strict for at least one state. Hence, by policy improvement theorem~\citep{sutton1998reinforcement},
$$
V^{\pi^g}_{\bar{\mathcal{M}}(T)}(x)\ge
V^{\pi^*_{\mathcal{M}}}_{\bar{\mathcal{M}}(T)}(x),\quad \forall x\in\mathcal S,
$$
with strict inequality for at least one state. Therefore, $\pi^*_{\mathcal M}$ is not optimal for $\bar{\mathcal M}(T)$, i.e., $\mathcal E_1$ holds.

The probability of false selection is upper-bounded by
$$
\mathbb{P}\left(\hat{\pi}_{T}\neq \pi^*_\mathcal{M}\right) \leq |\mathcal{S}|(|\mathcal{A}|-1)\max_{s\in\mathcal{S},a\neq \pi^*_{\mathcal{M}}(s)} \mathbb{P}\left(Q^{\pi^*_{\mathcal{M}}}_{\bar{\mathcal{M}}(T)}(s,a)>V^{\pi^*_{\mathcal{M}}}_{\bar{\mathcal{M}}(T)}(s)\right),
$$
and is lower bounded by
\begin{equation*}
\mathbb{P}\left(\hat{\pi}_{T}\neq \pi^*_\mathcal{M}\right) \geq \max_{s\in\mathcal{S},a\neq \pi^*_{\mathcal{M}}(s)} \mathbb{P}\left(Q^{\pi^*_{\mathcal{M}}}_{\bar{\mathcal{M}}(T)}(s,a)>V^{\pi^*_{\mathcal{M}}}_{\bar{\mathcal{M}}(T)}(s)\right).
\end{equation*}

Therefore,
$$
-\frac1T\log M_T-\frac1T\log(|\mathcal S|(|\mathcal A|-1))
\le
-\frac1T\log\mathbb P(\hat\pi_T\neq \pi^*_{\mathcal M})
\le
-\frac1T\log M_T,
$$
where 
$$
M_T:=
\max_{s\in\mathcal S,\;a\neq \pi^*_{\mathcal M}(s)}
\mathbb P\!\left(
Q^{\pi^*_{\mathcal M}}_{\bar{\mathcal M}(T)}(s,a)>
V^{\pi^*_{\mathcal M}}_{\bar{\mathcal M}(T)}(s)
\right).
$$
Since $|\mathcal S|(|\mathcal A|-1)$ is independent of $T$, we have
$
\frac{1}{T}\log\bigl(|\mathcal S|(|\mathcal A|-1)\bigr)\to 0,\text{as }T\to\infty,
$
and hence the middle term has the same limit as $-\frac1T\log M_T$. Because the index set is finite,
$$
-\frac1T\log M_T
=
\min_{s\in\mathcal S,\;a\neq \pi^*_{\mathcal M}(s)}
\left(
-\frac1T\log
\mathbb P\!\left(
Q^{\pi^*_{\mathcal M}}_{\bar{\mathcal M}(T)}(s,a)>
V^{\pi^*_{\mathcal M}}_{\bar{\mathcal M}(T)}(s)
\right)
\right).
$$

Assume that for each $s\in \mathcal{S}, a\in \mathcal{A}\setminus \{\pi^*_{\mathcal{M}}(s)\}$,
\begin{equation}
\label{eq: rate function G}
\lim_{T\rightarrow \infty} -\frac{1}{T}\log \mathbb{P}\left(Q^{\pi^*_{\mathcal{M}}}_{\bar{\mathcal{M}}(T)}(s,a)>V^{\pi^*_{\mathcal{M}}}_{\bar{\mathcal{M}}(T)}(s)\right) = \mathcal{G}_{s,a}
\end{equation}
for some rate function $\mathcal{G}_{s,a}$. Taking $T\to\infty$ and using \eqref{eq: rate function G}, we obtain
\begin{equation*}
\lim_{T\rightarrow \infty}-\frac{1}{T}\log\mathbb{P}\left(\hat{\pi}_{T}\neq \pi^*_\mathcal{M}\right) =\min_{s\in \mathcal{S}, a\neq \pi^*_{\mathcal{M}}(s)}\mathcal{G}_{s,a}.
\end{equation*}
\qed

\begin{theorem}[Perron-Frobenius in \citealt{dembo2009large}] \label{thm:perron-frobenius}
Let $\mathbf{B} = \{B(i,j)\}_{i,j=1}^{|\Sigma|}$ be an irreducible matrix. Then $\mathbf{B}$ possesses an eigenvalue $\rho$ (called the Perron-Frobenius eigenvalue) such that:
\begin{enumerate}
    \item[(a)] $\rho > 0$ is real.
    \item[(b)] For any eigenvalue $\lambda$ of $\mathbf{B}$, $|\lambda| \le \rho$.
    \item[(c)] There exist left and right eigenvectors corresponding to the eigenvalue $\rho$ that have strictly positive coordinates.
    \item[(d)] The left and right eigenvectors $\mu, \vartheta$ corresponding to the eigenvalue $\rho$ are unique up to a constant multiple.
    \item[(e)] For every $i \in \Sigma$ and every $\phi = (\phi_1, \dots, \phi_{|\Sigma|})$ such that $\phi_j > 0$ for all $j$,
    $$
    \lim_{n\to\infty} \frac{1}{n} \log \left[ \sum_{j=1}^{|\Sigma|} B^n(i,j) \phi_j \right]
    =
    \lim_{n\to\infty} \frac{1}{n} \log \left[ \sum_{j=1}^{|\Sigma|} \phi_j B^n(j,i) \right] = \log \rho.
    $$
\end{enumerate}
\end{theorem}

\proof{Proof of Lemma~\ref{lemma: explicit G}}
Define
$
X_{\bar{\mathcal{M}}(T)}(s,a):=P_{\bar{\mathcal{M}}(T)}(\cdot|s,a)\in \mathbb{R}^{S}.
$
Consider the behavior policy $\pi$. For each $s\in\mathcal{S}$ and $a\in\mathcal{A}$, define the random vector $\bar{M}(s,a;T)\in \mathbb{R}^{S+1}$ as the concatenation of $X_{\bar{\mathcal{M}}(T)}(s,a)$ and $R_{\bar{\mathcal{M}}(T)}(s,a)$, i.e.,
$
\bar{M}(s,a;T):=(X_{\bar{\mathcal{M}}(T)}(s,a),R_{\bar{\mathcal{M}}(T)}(s,a)).
$
Then, the empirical MDP model $\bar{\mathcal{M}}(T)$ can be represented by the random array $\bar{M}(T)\in \mathbb{R}^{(S+1)\times S\times A}$ obtained by collecting $\bar{M}(s,a;T)$ for all $(s,a)\in\mathcal{S}\times\mathcal{A}$.

Let $(\rho,\lambda)\in \mathbb{R}^{(S+1)\times S\times A}$ denote the collection of dual variables $(\rho(s,a),\lambda(s,a))$ over all state-action pairs, where $\rho(s,a)\in \mathbb{R}^{S}$ is associated with the transition component and $\lambda(s,a)\in\mathbb{R}$ is associated with the reward component.

To describe the dynamic empirical behavior of the controlled Markov chain, define the empirical flow
$$
\widehat{\eta}(s,a,s',a')
:=
\frac{1}{T}\sum_{t=1}^T
\mathbb I\{(s_t,a_t,s_{t+1},a_{t+1})=(s,a,s',a')\},
$$
and let its first marginal be
$$
\widehat{\eta}_1(s,a):=\sum_{s',a'}\widehat{\eta}(s,a,s',a').
$$
Thus, $\widehat{\eta}_1(s,a)$ represents the empirical occupation ratio of the state-action pair $(s,a)$. 

For $(\rho,\lambda)\in \mathbb{R}^{(S+1)\times S\times A}$, define the pathwise additive functional
$$
S_T(\rho,\lambda)
:=
\sum_{t=1}^T
\Big(
\rho(s_t,a_t)^\top X_{\mathcal M}(s_t,a_t)
+
\lambda(s_t,a_t)R_{\mathcal M}(s_t,a_t)
\Big),
$$
where $X_{\mathcal{M}}(s_t,a_t)$ and $R_{\mathcal{M}}(s_t,a_t)$ denote the one-step transition vector and reward generated at time $t$ from the state-action pair $(s_t,a_t)$.
Since
$$
T\sum_{s\in\mathcal{S},a\in\mathcal{A}}\widehat{\eta}_1(s,a)
\Big(
\rho(s,a)^\top X_{\bar{\mathcal M}(T)}(s,a)
+
\lambda(s,a)R_{\bar{\mathcal M}(T)}(s,a)
\Big)
=
S_T(\rho,\lambda),
$$
the scaled log-moment generating function can be written as
$$
\Lambda_T\!\big(T(\rho,\lambda)\big)
:=\log\mathbb E\Big[\exp\big(S_T(\rho,\lambda)\big)\Big]\\
=\log\mathbb E\Bigg[\exp\Bigg(\sum_{t=1}^T
\Big(\rho(s_t,a_t)^\top X_{\mathcal M}(s_t,a_t)
+\lambda(s_t,a_t)R_{\mathcal M}(s_t,a_t)\Big)\Bigg)\Bigg].
$$

Thus, the limiting log-moment generating function is associated with an additive functional of the controlled Markov chain on $\mathcal S\times\mathcal A$. To apply the G\"artner--Ellis theorem~\citep{dembo2009large}, it suffices to establish the existence of the limit
$$
\Psi((\rho,\lambda))
:=
\lim_{T\to\infty}\frac{1}{T}\Lambda_T\!\big(T(\rho,\lambda)\big).
$$

Since the observations in an MDP form a Markov chain on $\mathcal S\times\mathcal A$, the exponent is an additive functional along the sample path and thus
$$
    \Psi((\rho, \lambda))
    =
    \lim_{T\rightarrow \infty} \frac{1}{T}\log \mathbb{E}\left[\exp\left( \sum_{t=1}^T \left[ \rho(s_t,a_t)^\top X_{\mathcal{M}}(s_t,a_t) + \lambda(s_t,a_t) R_{\mathcal{M}}(s_t,a_t) \right] \right) \right].
$$

We introduce a tilted transition matrix $\mathbf{Q}_{(\rho, \lambda)}$ defined on the state-action space $\mathcal{S} \times \mathcal{A}$. For a transition from $(s,a)$ to $(s',a')$, its matrix element is defined by
$$
    \mathbf{Q}_{(\rho, \lambda)}((s,a), (s',a'))
    :=
    \pi(a'|s') P_{\mathcal{M}}(s'|s,a)
    \exp(\rho(s'|s,a))
    \mathbb{E}\left[\exp(\lambda(s,a) R_{\mathcal{M}}(s,a)) \right],
$$
where $\rho(s'|s,a) \in \mathbb{R}$ denotes the component of the vector $\rho(s,a)$ indexed by the next state $s'$.

Let $\alpha$ denote the initial distribution of the state-action pair $(s_1,a_1)$. By iterating conditional expectations along the Markov chain $(s_t,a_t)$, there exists a strictly positive vector $v_{(\rho,\lambda)}$ such that
$$
\mathbb{E}\Bigg[ \exp\left( \sum_{t=1}^T \left[ \rho(s_t,a_t)^\top X_{\mathcal{M}}(s_t,a_t) + \lambda(s_t,a_t) R_{\mathcal{M}}(s_t,a_t) \right] \right)\Bigg]
=
\alpha^\top \mathbf{Q}_{(\rho,\lambda)}^{\,T-1}v_{(\rho,\lambda)}.
$$
Here the vector $v_{(\rho,\lambda)}$ absorbs the terminal one-step contribution and therefore does not affect the exponential growth rate.

By part (e) of the Perron--Frobenius theorem (Theorem \ref{thm:perron-frobenius}), since $\mathbf Q_{(\rho,\lambda)}$ is irreducible and nonnegative,
$$
\Psi((\rho,\lambda))
=
\lim_{T\to\infty}\frac{1}{T}\log\!\left(\alpha^\top \mathbf{Q}_{(\rho,\lambda)}^{\,T-1}v_{(\rho,\lambda)}\right)
=
\log \sigma(\mathbf Q_{(\rho,\lambda)}),
$$
where $\sigma(\mathbf Q_{(\rho,\lambda)})$ denotes the Perron--Frobenius eigenvalue of $\mathbf Q_{(\rho,\lambda)}$.

Moreover, because $\mathbf Q_{(\rho,\lambda)}$ is a finite-dimensional irreducible nonnegative matrix whose entries depend smoothly on $(\rho,\lambda)$, its Perron--Frobenius eigenvalue $\sigma(\mathbf Q_{(\rho,\lambda)})$ is positive and differentiable with respect to $(\rho,\lambda)$. Therefore, $
\Psi((\rho,\lambda))=\log \sigma(\mathbf Q_{(\rho,\lambda)})
$ is differentiable throughout $\mathbb R^{(S+1)\times S\times A}$. Since the effective domain is all of $\mathbb R^{(S+1)\times S\times A}$, its boundary is empty; thus the steepness requirement in the definition of essential smoothness is vacuous. Consequently, $\Psi$ is essentially smooth \citet[Definition 2.3.5]{dembo2009large}, and the regularity conditions of the G\"artner-Ellis theorem are satisfied.

We now turn to the variational representation involving an auxiliary candidate flow $\eta$, from which the rate function for the transition-reward pair $(x,y)$ is obtained after eliminating $\eta$, where
$$
x=\{x(\cdot\mid s,a):(s,a)\in\mathcal S\times\mathcal A\},
\qquad
y=\{y(s,a):(s,a)\in\mathcal S\times\mathcal A\}.
$$
Here, for each $(s,a)$, $x(\cdot\mid s,a)$ denotes the empirical transition distribution from $(s,a)$, and $y(s,a)$ denotes the corresponding empirical reward mean. 

Since the exponent above is an additive functional of the controlled Markov chain, the G\"artner--Ellis theorem applies once the limiting log-moment generating function
$
\Psi((\rho,\lambda))
=
\log \sigma(\mathbf Q_{(\rho,\lambda)})
$
has been identified and verified to be essentially smooth. This provides the convex-analytic structure underlying the rate function. To derive an explicit expression for the latter in terms of $(x,y)$, we further combine $\Psi$ with Varadhan's variational characterization of the spectral radius.

We next invoke Varadhan’s variational characterization of the spectral radius for irreducible nonnegative matrices; see \citet[Exercise 3.1.9(b)]{dembo2009large}. Applied to $\mathbf Q_{(\rho,\lambda)}$ (indexed by $(s,a)\to(s',a')$), it yields
$$
\log \sigma(\mathbf Q_{(\rho,\lambda)})
=\\
\sup_{\eta\in\mathcal E}
\left\{
\sum_{s,a}\sum_{s',a'} \eta(s,a,s',a')\,
\log \mathbf Q_{(\rho,\lambda)}\big((s,a),(s',a')\big)
-
\sum_{s,a}\sum_{s',a'} \eta(s,a,s',a')\,
\log \frac{\eta(s,a,s',a')}{\eta_1(s,a)}
\right\},
$$
where $\eta_1(s,a)\triangleq \sum_{s',a'}\eta(s,a,s',a')$ and
$$
\mathcal E \triangleq \Big\{\eta\ge 0:\ \sum_{s,a,s',a'}\eta(s,a,s',a')=1,\ 
\sum_{s',a'}\eta(s,a,s',a')=\sum_{\tilde{s},\tilde{a}}\eta(\tilde{s},\tilde{a},s,a),\ \forall(s,a)\Big\},
$$
and $\eta(s,a,s^\prime,a^\prime)=0$ whenever $\mathbf Q_{(\rho,\lambda)}\big((s,a),(s',a')\big) = 0$.

Using the identity
$$
\log \mathbf Q_{(\rho,\lambda)}((s,a),(s',a'))
=
\log(\pi(a'|s')P_{\mathcal M}(s'|s,a))+\rho(s'|s,a)
+\log \mathbb E\!\left[e^{\lambda(s,a)R_{\mathcal M}(s,a)}\right],
$$
we are led to the variational problem
\begin{equation*}
\begin{aligned}
\sup_{\rho,\lambda}\inf_{\eta\in\mathcal E}
\Bigg\{
&\sum_{s,a}\eta_1(s,a)\sum_{s'}x(s'|s,a)\rho(s'|s,a)
+\sum_{s,a}\eta_1(s,a)\lambda(s,a)y(s,a)\\
&-\sum_{s,a,s',a'}\eta(s,a,s',a')\log(\pi(a'|s')P_{\mathcal M}(s'|s,a))\\
&-\sum_{s,a,s',a'}\eta(s,a,s',a')\rho(s'|s,a)
-\sum_{s,a,s',a'}\eta(s,a,s',a')\log \mathbb E\!\left[e^{\lambda(s,a)R_{\mathcal M}(s,a)}\right]\\
&+\sum_{s,a,s',a'}\eta(s,a,s',a')\log\frac{\eta(s,a,s',a')}{\eta_1(s,a)}
\Bigg\},
\end{aligned}
\end{equation*}
where $\eta$ is an auxiliary flow variable introduced through Varadhan's variational formula. The resulting optimization problem is therefore a lifted variational representation, from which the explicit rate function for $(x,y)$ will be obtained after eliminating $\eta$.
We next simplify this variational problem and identify the resulting rate expression for the empirical transition-reward pair $(x,y)$.

The set \(\mathcal E\) is convex and compact. For fixed \(\eta\in\mathcal E\), the objective is concave and upper semicontinuous in \((\rho,\lambda)\); for fixed \((\rho,\lambda)\), it is convex and lower semicontinuous in \(\eta\). Hence, by Sion's minimax theorem, the order of \(\sup_{(\rho,\lambda)}\) and \(\inf_{\eta\in\mathcal E}\) can be interchanged. 

Rearranging the terms involving \(\rho\), we get
$$
\sum_{s,a,s'}
\left(
\eta_1(s,a)x(s'|s,a)-\sum_{a'}\eta(s,a,s',a')
\right)\rho(s'|s,a).
$$
Since each \(\rho(s'|s,a)\in\mathbb R\) is unconstrained, the supremum over \(\rho\) is finite only if
\begin{equation}
\label{eq:eta-marginal-constraint}
\sum_{a'}\eta(s,a,s',a')=\eta_1(s,a)x(s'|s,a), \qquad \forall (s,a,s').
\end{equation}
Indeed, if \eqref{eq:eta-marginal-constraint} fails for some \((s,a,s')\), then by sending the corresponding \(\rho(s'|s,a)\) to \(+\infty\) or \(-\infty\), the objective becomes \(+\infty\). Therefore, only \(\eta\in\mathcal E\) satisfying \eqref{eq:eta-marginal-constraint} can contribute to \(I(x,y)\).

Under the constraint \eqref{eq:eta-marginal-constraint}, the terms involving \(P_{\mathcal M}\), \(\pi\), and the entropy term become
\begin{equation*}
\begin{aligned}
&-\sum_{s,a,s',a'}\eta(s,a,s',a')\log(\pi(a'|s')P_{\mathcal M}(s'|s,a))
+\sum_{s,a,s',a'}\eta(s,a,s',a')\log\frac{\eta(s,a,s',a')}{\eta_1(s,a)}\\
&=
-\sum_{s,a,s',a'}\eta(s,a,s',a')\log P_{\mathcal M}(s'|s,a)
-\sum_{s,a,s',a'}\eta(s,a,s',a')\log \pi(a'|s')\\
&\qquad
+\sum_{s,a,s',a'}\eta(s,a,s',a')\log\frac{\eta(s,a,s',a')}{\eta_1(s,a)}.
\end{aligned}
\end{equation*}
For fixed \((s,a,s')\), the quantity \(\sum_{a'}\eta(s,a,s',a')\) is prescribed by \eqref{eq:eta-marginal-constraint}. Hence, for fixed \((s,a,s')\), the only terms depending on the conditional distribution of \(a'\) are
$$
-\sum_{a'}\eta(s,a,s',a')\log \pi(a'|s')
+\sum_{a'}\eta(s,a,s',a')\log \eta(s,a,s',a').
$$
Subject to the constraint
$$
\sum_{a'}\eta(s,a,s',a')=\eta_1(s,a)x(s'|s,a),
$$
this expression is minimized when
$$
\eta(s,a,s',a')=\eta_1(s,a)x(s'|s,a)\pi(a'|s')
$$
for all \((s,a,s',a')\). Equivalently, the minimizer is attained when the conditional distribution on \(a'\) coincides with the reference distribution \(\pi(\cdot|s')\). Substituting this optimizer yields
$$
\eta_1(s',a')
=
\sum_{s,a}\eta_1(s,a)x(s'|s,a)\pi(a'|s'),
\qquad \forall (s',a'),
$$
that is, $\eta_1$ must be invariant under the kernel induced by $x$ and $\pi$. Moreover,
\begin{equation*}
\begin{aligned}
&-\sum_{s,a,s',a'}\eta(s,a,s',a')\log(\pi(a'|s')P_{\mathcal M}(s'|s,a))
+\sum_{s,a,s',a'}\eta(s,a,s',a')\log\frac{\eta(s,a,s',a')}{\eta_1(s,a)}\\
&=
\sum_{s,a}\eta_1(s,a)\sum_{s'}x(s'|s,a)\log\frac{x(s'|s,a)}{P_{\mathcal M}(s'|s,a)}.
\end{aligned}
\end{equation*}
Therefore, the contribution of the transition component is
$$
I_1\big(x(s,a)\big)
=
D_{\mathrm{KL}}\!\left(x(\cdot\mid s,a)\,\middle\|\,P_{\mathcal M}(\cdot\mid s,a)\right) = \sum_{s'}x(s'|s,a)\log\frac{x(s'|s,a)}{P_{\mathcal M}(s'|s,a)},
$$
which is equivalently the Fenchel-Legendre transform of the logarithmic moment generating function of $X_{\mathcal{M}}(s,a)$.

Moreover, the supremum over $\lambda(s,a)$ is separable across $(s,a)$ and yields the Cram\'er transform
$$
I_2\big(y(s,a)\big)=\sup_{\lambda(s,a)\in\mathbb R}\left\{\lambda(s,a)\,y(s,a)-\log \mathbb E\!\left[e^{\lambda(s,a) R_{\mathcal M}(s,a)}\right]\right\}.
$$
Consequently, for a fixed empirical flow $\eta$ satisfying \eqref{eq:eta-marginal-constraint}, 
$$
J(\eta,x,y)
=
\sum_{s\in\mathcal{S},a\in\mathcal{A}}\eta_1(s,a)\Big(I_1\big(x(s,a)\big)+I_2\big(y(s,a)\big)\Big).
$$

Therefore,
$$
I(x,y)
=
\inf_{\eta\in\mathcal E:\,\eqref{eq:eta-marginal-constraint}\ \text{holds}}
J(\eta,x,y)
=
\inf_{\eta\in\mathcal E:\,\eqref{eq:eta-marginal-constraint}\ \text{holds}}
\sum_{s\in\mathcal{S},a\in\mathcal{A}}\eta_1(s,a)\Big(I_1\big(x(s,a)\big)+I_2\big(y(s,a)\big)\Big).
$$

After minimizing over the conditional distribution on $a'$, the remaining cost depends on $\eta$ only through its first marginal $\eta_1$. Therefore, the rate function can equivalently be written as
$$
I(x,y)
=
\inf
\left\{
\sum_{s,a}\eta_1(s,a)\Big(I_1\big(x(s,a)\big)+I_2\big(y(s,a)\big)\Big)
\right\},
$$
where the infimum is taken over all \(\eta_1 \in \mathcal{F}_{\pi}(x)\), where
$$
\mathcal{F}_{\pi}(x)
:=
\left\{
\eta_1\in\Omega:
\eta_1(s',a')
=
\sum_{s,a}\eta_1(s,a)x(s'|s,a)\pi(a'|s'),
\ \forall (s',a')
\right\}.
$$

Conversely, for any $\eta_1\in\mathcal F_\pi(x)$, defining
$
\eta(s,a,s',a'):=\eta_1(s,a)x(s'|s,a)\pi(a'|s')
$
produces a feasible element of $\mathcal E$ satisfying \eqref{eq:eta-marginal-constraint}.
Hence, the optimization over $\eta$ is equivalent to the optimization over $\eta_1\in\mathcal F_\pi(x)$.

Finally, define
$$
\mathcal{E}_{s,a}
=
\left\{
(x,y): Q^{\pi^*_{\mathcal{M}}}_{(x,y)}(s,a) > V^{\pi^*_{\mathcal{M}}}_{(x,y)}(s)
\right\}.
$$
It remains to verify that the error event \(\mathcal E_{s,a}\) is an \(I\)-continuity set \citep[Section 1.2]{dembo2009large}, so that the large deviations upper and lower bounds coincide on \(\mathcal E_{s,a}\). For the policy \(\pi^*_{\mathcal M}\), both \((x,y)\mapsto V^{\pi^*_{\mathcal M}}_{(x,y)}\) and \((x,y)\mapsto Q^{\pi^*_{\mathcal M}}_{(x,y)}(s,a)\) are continuous. Hence \(\mathcal E_{s,a}\) is open.

Let
$$
I(x,y)
:=
\inf_{\eta_1\in\mathcal F_\pi(x)}
\sum_{s'\in\mathcal S,a'\in\mathcal A}
\eta_1(s',a')
\Big(
I_1\big(x(s',a')\big)
+
I_2\big(y(s',a')\big)
\Big).
$$
To show that \(\mathcal E_{s,a}\) is an \(I\)-continuity set, it remains to prove that
$
\inf_{(x,y)\in\mathcal E_{s,a}} I(x,y)
=
\inf_{(x,y)\in\overline{\mathcal E}_{s,a}} I(x,y).
$
Since \(\mathcal E_{s,a}\subset \overline{\mathcal E}_{s,a}\), we automatically have
$
\inf_{(x,y)\in\overline{\mathcal E}_{s,a}} I(x,y)
\le
\inf_{(x,y)\in\mathcal E_{s,a}} I(x,y).
$
It therefore suffices to prove the reverse inequality.

Fix any \((x,y)\in \partial\mathcal E_{s,a}\) such that \(I(x,y)<\infty\). We first consider the case \(y(s,a)<1\). The boundary case \(y(s,a)=1\) can be handled similarly by perturbing the model in the opposite direction, for example by slightly decreasing the reward of the action prescribed by \(\pi^*_{\mathcal M}\) at state \(s\). This lowers \(V^{\pi^*_{\mathcal M}}_{(x,y)}(s)\) relative to \(Q^{\pi^*_{\mathcal M}}_{(x,y)}(s,a)\) and therefore moves the point into \(\mathcal E_{s,a}\); the continuity argument is unchanged.
For \(\varepsilon>0\) sufficiently small so that \(y(s,a)+\varepsilon\in[0,1]\), define
$$
y^\varepsilon(s',a')
:=
y(s',a')+\varepsilon\,\mathbf 1\{(s',a')=(s,a)\}.
$$
Then
$$
V^{\pi^*_{\mathcal M}}_{(x,y^\varepsilon)}
=
V^{\pi^*_{\mathcal M}}_{(x,y)},
\qquad
Q^{\pi^*_{\mathcal M}}_{(x,y^\varepsilon)}(s,a)
=
Q^{\pi^*_{\mathcal M}}_{(x,y)}(s,a)+\varepsilon.
$$
Since \((x,y)\in\partial\mathcal E_{s,a}\), we have
$
Q^{\pi^*_{\mathcal M}}_{(x,y)}(s,a)
=
V^{\pi^*_{\mathcal M}}_{(x,y)}(s),
$
and therefore
$$
Q^{\pi^*_{\mathcal M}}_{(x,y^\varepsilon)}(s,a)
=
V^{\pi^*_{\mathcal M}}_{(x,y^\varepsilon)}(s)+\varepsilon
>
V^{\pi^*_{\mathcal M}}_{(x,y^\varepsilon)}(s).
$$
Hence \((x,y^\varepsilon)\in\mathcal E_{s,a}\) for all sufficiently small \(\varepsilon>0\), and \((x,y^\varepsilon)\to (x,y)\) as \(\varepsilon\downarrow0\).
Next, we show that
$$
I(x,y^\varepsilon)\to I(x,y)
\qquad\text{as }\varepsilon\downarrow0.
$$
For any \(\delta>0\), choose \(\eta_1^\delta\in\mathcal F_\pi(x)\) such that
$$
\sum_{s',a'}\eta_1^\delta(s',a')
\Big(
I_1(x(s',a'))+I_2(y(s',a'))
\Big)
\le I(x,y)+\delta.
$$
Since \(x\) is unchanged, \(\mathcal F_\pi(x)\) is the same for \((x,y)\) and \((x,y^\varepsilon)\). Therefore,
$$
\begin{aligned}
I(x,y^\varepsilon)
&\le
\sum_{s',a'}\eta_1^\delta(s',a')
\Big(
I_1(x(s',a'))+I_2(y^\varepsilon(s',a'))
\Big)\\
&=
\sum_{s',a'}\eta_1^\delta(s',a')
\Big(
I_1(x(s',a'))+I_2(y(s',a'))
\Big)
+
\eta_1^\delta(s,a)\Big(I_2(y(s,a)+\varepsilon)-I_2(y(s,a))\Big)\\
&\le
I(x,y)+\delta
+
\eta_1^\delta(s,a)\Big(I_2(y(s,a)+\varepsilon)-I_2(y(s,a))\Big).
\end{aligned}
$$
Because the rewards are bounded in \([0,1]\), the logarithmic moment generating function of \(R_{\mathcal M}(s,a)\) is finite for all \(\lambda\in\mathbb R\), and hence its Fenchel--Legendre transform \(I_2\) is continuous on its effective domain. Letting \(\varepsilon\downarrow0\) in the above display yields
$
\limsup_{\varepsilon\downarrow0} I(x,y^\varepsilon)\le I(x,y)+\delta.
$
Since \(\delta>0\) is arbitrary,
$
\limsup_{\varepsilon\downarrow0} I(x,y^\varepsilon)\le I(x,y).
$
On the other hand, \(I\) is lower semicontinuous, so
$
I(x,y)\le \liminf_{\varepsilon\downarrow0} I(x,y^\varepsilon).
$
Combining the last two inequalities gives
$
I(x,y^\varepsilon)\to I(x,y)
~\text{as }\varepsilon\downarrow0.
$
Thus, every boundary point of \(\mathcal E_{s,a}\) can be approximated by points in \(\mathcal E_{s,a}\) without changing the rate in the limit. Consequently,
$
\inf_{(x,y)\in\mathcal E_{s,a}} I(x,y)
\le
\inf_{(x,y)\in\overline{\mathcal E}_{s,a}} I(x,y).
$
Together with the reverse inequality already noted above, we conclude that
$
\inf_{(x,y)\in\mathcal E_{s,a}} I(x,y)
=
\inf_{(x,y)\in\overline{\mathcal E}_{s,a}} I(x,y).
$
Therefore, \(\mathcal E_{s,a}\) is an \(I\)-continuity set. By the large deviations principle,
$$
\lim_{T\to\infty}
-\frac{1}{T}
\log
\mathbb P\!\left(
Q^{\pi^*_{\mathcal M}}_{\bar{\mathcal M}(T)}(s,a)
>
V^{\pi^*_{\mathcal M}}_{\bar{\mathcal M}(T)}(s)
\right)
=
\inf_{(x,y)\in\mathcal E_{s,a}} I(x,y).
$$
\qed

\proof{Proof of Theorem~\ref{thm: rate function}}
The theorem follows directly from combining Lemma~\ref{lemma: rate function G} with Lemma~\ref{lemma: explicit G}.

\section{Proof of Lemma \ref{lemma: generate rate}}
Under the generative-model assumption, each state-action pair can be sampled independently with a prescribed proportion $\omega_{sa}$. Hence the empirical occupation measure is no longer induced by trajectory dynamics and is simply fixed to be $\omega$. This invariance constraint therefore disappears, and the infimum over $\eta_1$ collapses to evaluation at $\omega$. As a result, 
\begin{equation*}
    I(x,y):= \sum_{s'\in\mathcal S,a'\in\mathcal A}
\omega_{s^\prime a^\prime}
\Big(
I_1\big(x(s',a')\big)
+
I_2\big(y(s',a')\big)
\Big),
\end{equation*}
which yields the exponential decay rate
$$
        \mathcal{R}(\mathcal{M}, \omega) := \min_{s\in \mathcal{S},a\in \mathcal{A}\setminus\{\pi^*_{\mathcal{M}}(s)\}}\inf_{(x,y)\in\mathcal E_{s,a}}
\sum_{s'\in\mathcal S,a'\in\mathcal A}
\omega_{s^\prime a^\prime}
\Big(
I_1\big(x(s',a')\big)
+
I_2\big(y(s',a')\big)
\Big).
$$

\section{$\epsilon$-Optimal Policy Identification}
\label{sec: epsilon-optimal}
Define the set of $\epsilon$-optimal policies for an MDP $\mathcal{M}$ by
$
    \Pi^{\epsilon}_\mathcal{M} := \left\{\pi: \max_{s\in\mathcal{S}}(V^*_{\mathcal{M}}(s)-V^{\pi}_{\mathcal{M}}(s))\le \epsilon\right\}.
$
We consider a conservative error event
$
    \mathcal{E}_1 := \{\pi^*_{\mathcal{M}} \notin \Pi^{\epsilon}_{\bar{\mathcal{M}}(T)}\}
$. Lemma~\ref{lemma: epsilon_sufficient} below provides a necessary condition for $\mathcal{E}_1$.
\begin{lemma}
\label{lemma: epsilon_sufficient}
Let $\bar{\mathcal{M}}(T)$ be a discounted MDP with discount factor $\gamma\in(0,1)$. Define
$$
\mathcal{E}_1 = \left\{ {\pi}^*_{{\mathcal{M}}} \notin \Pi^\epsilon_{\bar{\mathcal{M}}(T)}\right\},
\qquad
\mathcal{E}_2 = \left\{\exists s\in \mathcal{S},\ a\in \mathcal{A}\setminus\{\pi^*_{\mathcal{M}}(s)\}:\ 
Q^{\pi^*_{\mathcal{M}}}_{\bar{\mathcal{M}}(T)}(s,a)>V^{\pi^*_{\mathcal{M}}}_{\bar{\mathcal{M}}(T)}(s)+(1-\gamma)\epsilon\right\}.
$$
Then $\mathcal{E}_1$ implies $\mathcal{E}_2$.
\end{lemma}

\proof{Proof of Lemma \ref{lemma: epsilon_sufficient}}
Assume that $\mathcal{E}_2$ does not occur. Since
$Q^{\pi^*_{\mathcal{M}}}_{\bar{\mathcal{M}}(T)}(s,\pi^*_{\mathcal{M}}(s))
=V^{\pi^*_{\mathcal{M}}}_{\bar{\mathcal{M}}(T)}(s)$ for all $s\in\mathcal{S}$,
the event $\mathcal{E}_2^c$ implies that, for all $s\in\mathcal{S}$ and all $a\in\mathcal{A}$,
\begin{equation}
\label{eq:adv_bound_all_sa}
Q^{\pi^*_{\mathcal{M}}}_{\bar{\mathcal{M}}(T)}(s,a)
\le
V^{\pi^*_{\mathcal{M}}}_{\bar{\mathcal{M}}(T)}(s)+(1-\gamma)\epsilon.
\end{equation}
Equivalently, defining
$
\Delta
~:=~
\sup_{s\in\mathcal{S}}\ \max_{a\in\mathcal{A}}
\Big( Q^{\pi^*_{\mathcal{M}}}_{\bar{\mathcal{M}}(T)}(s,a)
      -V^{\pi^*_{\mathcal{M}}}_{\bar{\mathcal{M}}(T)}(s)\Big),
$
we have $\Delta\le (1-\gamma)\epsilon$.

Let $T_{\bar{\mathcal{M}}(T)}$ denote the Bellman optimality operator associated with
$\bar{\mathcal{M}}(T)$. By definition and \eqref{eq:adv_bound_all_sa},
for every $s\in\mathcal{S}$,
\begin{align}
\big(T_{\bar{\mathcal{M}}(T)} V^{\pi^*_{\mathcal{M}}}_{\bar{\mathcal{M}}(T)}\big)(s)
=
\max_{a\in\mathcal{A}} Q^{\pi^*_{\mathcal{M}}}_{\bar{\mathcal{M}}(T)}(s,a)
\notag
\le
V^{\pi^*_{\mathcal{M}}}_{\bar{\mathcal{M}}(T)}(s)+\Delta
\notag\le
V^{\pi^*_{\mathcal{M}}}_{\bar{\mathcal{M}}(T)}(s)+(1-\gamma)\epsilon.
\label{eq:T_on_Vpi}
\end{align}
Using the monotonicity of $T_{\bar{\mathcal{M}}(T)}$, we obtain,
for any integer $n\ge 1$,
\begin{equation}
T_{\bar{\mathcal{M}}(T)}^{\,n} V^{\pi^*_{\mathcal{M}}}_{\bar{\mathcal{M}}(T)}
\le
V^{\pi^*_{\mathcal{M}}}_{\bar{\mathcal{M}}(T)}
+ \Delta\sum_{k=0}^{n-1}\gamma^k\,\mathbf{1}
\le
V^{\pi^*_{\mathcal{M}}}_{\bar{\mathcal{M}}(T)}
+ (1-\gamma)\epsilon\sum_{k=0}^{n-1}\gamma^k\,\mathbf{1}=
V^{\pi^*_{\mathcal{M}}}_{\bar{\mathcal{M}}(T)}
+ \epsilon(1-\gamma^n)\,\mathbf{1},
\label{eq:iterate_bound}
\end{equation}
where $\mathbf{1}$ is the all-ones vector over $\mathcal{S}$.
Taking $n\to\infty$ in \eqref{eq:iterate_bound} and using the standard identity
$
V^*_{\bar{\mathcal{M}}(T)}=\lim_{n\to\infty} T_{\bar{\mathcal{M}}(T)}^{\,n}V$
for any bounded $V$, 
we conclude that
$
V^*_{\bar{\mathcal{M}}(T)}
\le
V^{\pi^*_{\mathcal{M}}}_{\bar{\mathcal{M}}(T)}+\epsilon\,\mathbf{1}.
$
Thus ${\pi}^*_{{\mathcal{M}}}$ is $\epsilon$-optimal under $\bar{\mathcal{M}}(T)$, i.e.,
${\pi}^*_{{\mathcal{M}}}\in \Pi^\epsilon_{\bar{\mathcal{M}}(T)}$, which means $\mathcal{E}_1$ does not occur.
This proves the contrapositive, and hence $\mathcal{E}_1$ implies $\mathcal{E}_2$.
The analysis for identifying an $\epsilon$-optimal policy proceeds as in the unique-optimum case, except for an additional $(1-\gamma)\epsilon$ term. Accordingly, we can follow the same approach to derive a tractable convex relaxation.

\section{Robust Optimality}
\label{sec: robust opt}
\begin{figure}[h]
\centering
\resizebox{0.6\linewidth}{!}{
\begin{tikzpicture}[
    % --- Styles ---
    neuron/.style={
        circle, 
        draw, 
        thick, 
        minimum size=0.9cm, 
        inner sep=0pt, 
        fill=white
    },
    input/.style={
        circle, 
        draw, 
        thick, 
        minimum size=0.9cm, 
        inner sep=0pt, 
        fill=white
    },
    loop top/.style={
        ->, 
        >=latex, 
        thick,
        out=60, 
        in=120, 
        looseness=5, 
        min distance=1.2cm
    },
    loop right/.style={
        ->, 
        >=latex, 
        thick,
        out=-30, 
        in=30, 
        looseness=5, 
        min distance=1.2cm
    },
    conn/.style={
        ->, 
        >=latex, 
        thick
    },
    guideline/.style={
        densely dotted, 
        gray, 
        thin
    }
]

    % ==========================================
    % GROUP 1 (x1) - Center at x=2
    % ==========================================
    \node[input] (x1) at (2, 0) {$s^{1}_{1}$};

    % Nodes (Left at 1.0, Right at 3.0)
    \node[neuron] (y1_1_L) at (1.0, 2.5) {$s^{2}_{11}$}; 
    \node[neuron] (y1_1_R) at (3.0, 2.5) {$s^2_{1L}$};
    \node[neuron] (y2_1_L) at (1.0, 5.0) {$s^3_{11}$};
    \node[neuron] (y2_1_R) at (3.0, 5.0) {$s^3_{1L}$};
    
    % Internal Ellipsis
    \node at (2.0, 2.5) {\Large $\dots$};
    \node at (2.0, 5.0) {\Large $\dots$};

    % Connections
    \draw[conn, densely dotted] (x1) -- node[left, font=\small, pos=0.7] {$a_1$} (y1_1_L);
    \draw[conn] (x1) -- node[right, font=\small, pos=0.7] {$a_L$} (y1_1_R);
    \draw[conn] (y1_1_L) -- (y2_1_L);
    \draw[conn] (y1_1_R) -- (y2_1_R);

    % Loops
    \draw (y1_1_L) edge[loop right] (y1_1_L);
    \draw (y1_1_R) edge[loop right] (y1_1_R);
    \draw (y2_1_L) edge[loop top] (y2_1_L);
    \draw (y2_1_R) edge[loop top] (y2_1_R);

    % ==========================================
    % SEPARATION DOTS (Between Group 1 and 2)
    % ==========================================
    \node at (4.5, 2.5) {\Huge $\dots$};
    \node at (4.5, 5.0) {\Huge $\dots$};
    \node at (4.5, 0.0) {\Huge $\dots$};

    % ==========================================
    % GROUP 2 (x2) - Center at x=7
    % ==========================================
    \node[input] (x2) at (7, 0) {$s^{1}_{2}$};

    % Nodes (Left at 6.0, Right at 8.0)
    \node[neuron] (y1_2_L) at (6.0, 2.5) {$s^2_{21}$};
    \node[neuron] (y1_2_R) at (8.0, 2.5) {$s^2_{2L}$};
    \node[neuron] (y2_2_L) at (6.0, 5.0) {$s^3_{21}$};
    \node[neuron] (y2_2_R) at (8.0, 5.0) {$s^3_{2L}$};

    % Internal Ellipsis
    \node at (7.0, 2.5) {\Large $\dots$};
    \node at (7.0, 5.0) {\Large $\dots$};

    % Connections
    \draw[conn, densely dotted] (x2) -- node[left, font=\small, pos=0.7] {$a_1$} (y1_2_L);
    \draw[conn] (x2) -- node[right, font=\small, pos=0.7] {$a_L$} (y1_2_R);
    \draw[conn] (y1_2_L) -- (y2_2_L);
    \draw[conn] (y1_2_R) -- (y2_2_R);

    % Loops
    \draw (y1_2_L) edge[loop right] (y1_2_L);
    \draw (y1_2_R) edge[loop right] (y1_2_R);
    \draw (y2_2_L) edge[loop top] (y2_2_L);
    \draw (y2_2_R) edge[loop top] (y2_2_R);

    % ==========================================
    % SEPARATION DOTS (Between Group 2 and 3)
    % ==========================================
    \node at (9.5, 2.5) {\Huge $\dots$};
    \node at (9.5, 5.0) {\Huge $\dots$};
    \node at (9.5, 0.0) {\Huge $\dots$};

    % ==========================================
    % GROUP 3 (xK) - Center at x=12
    % ==========================================
    \node[input] (xk) at (12, 0) {$s^{1}_{K}$};

    % Nodes (Left at 11.0, Right at 13.0)
    \node[neuron] (y1_k_L) at (11.0, 2.5) {$s^2_{K1}$};
    \node[neuron] (y1_k_R) at (13.0, 2.5) {$s^2_{KL}$};
    \node[neuron] (y2_k_L) at (11.0, 5.0) {$s^3_{K1}$};
    \node[neuron] (y2_k_R) at (13.0, 5.0) {$s^3_{KL}$};

    % Internal Ellipsis
    \node at (12.0, 2.5) {\Large $\dots$};
    \node at (12.0, 5.0) {\Large $\dots$};

    % Connections
    \draw[conn, densely dotted] (xk) -- node[left, font=\small, pos=0.7] {$a_1$} (y1_k_L);
    \draw[conn] (xk) -- node[right, font=\small, pos=0.7] {$a_L$} (y1_k_R);
    \draw[conn] (y1_k_L) -- (y2_k_L);
    \draw[conn] (y1_k_R) -- (y2_k_R);

    % Loops
    \draw (y1_k_L) edge[loop right] (y1_k_L);
    \draw (y1_k_R) edge[loop right] (y1_k_R);
    \draw (y2_k_L) edge[loop top] (y2_k_L);
    \draw (y2_k_R) edge[loop top] (y2_k_R);

\end{tikzpicture}}
\caption{The constructed MDP problem instance}
\label{fig: MDP instance}
\end{figure}
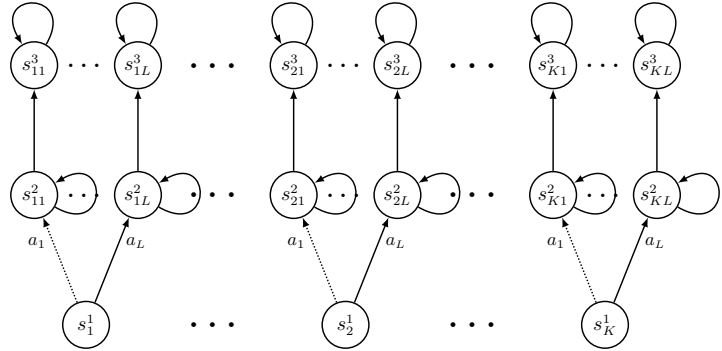

We now construct a least favorable hard instance to derive an upper bound on the scaling of $\mathcal R^*$. Consider the MDP $\mathcal M$ and the alternative model $\tilde{\mathcal M}$ defined in Figure~\ref{fig: MDP instance}. The construction is inspired by~\citep{gheshlaghi2013minimax}, but is adapted to the problem of identifying the optimal policy rather than estimating the full $Q$-function.

For these two models, there are $C(S,A)=3KL$ state-action pairs, where $K$ and $L$ are positive integers. We assume that the state space $\mathcal{S}$ can be partitioned into three disjoint subsets, $\mathcal{S} = \mathcal{S}_1\cup \mathcal{S}_2\cup\mathcal{S}_3$, where $\mathcal{S}_1 = \{s^{1}_{1},\ldots,s^{1}_{K}\}$, $\mathcal{S}_2 = \{s^{2}_{11},\ldots s^2_{KL}\}$ and $\mathcal{S}_3 = \{s^{3}_{11},\ldots s^3_{KL}\}$. The action space $\mathcal{A}$ is similarly partitioned as $\mathcal{A} = \mathcal{A}_1\cup \mathcal{A}_2$, where $\mathcal{A}_1 = \{a_1,\ldots,a_{L}\}$ and $\mathcal{A}_2 = \{a_0\}$. For states in $\mathcal{S}_1$, the set of admissible actions is $\mathcal{A}_1$. Taking action $a_j\in \mathcal{A}_1$ from any state $s^1_{i}\in\mathcal{S}_1$ leads deterministically to state $s^2_{ij}\in\mathcal{S}_2$ with probability one, and yields a deterministic reward $r_{\mathcal{M}}(s^1_i,a_j)=0$ for all $(s^1_i,a_j)\in\mathcal{S}_1\times \mathcal{A}_1$. For states in $\mathcal{S}_2 \cup \mathcal{S}_3$, the only admissible action is $a_0\in\mathcal{A}_2$. Under MDP $\mathcal{M}$, taking action $a_0$ for any state $s^{2}_{i1}\in\mathcal{S}_2$ induces the transition probabilities: 
\begin{equation*}
    P_{\mathcal{M}}(s^2_{i1}|s^2_{i1},a_0) = p+\alpha, \quad P_{\mathcal{M}}(s^3_{i1}|s^2_{i1},a_0) = 1-(p+\alpha),
\end{equation*}
where $0<p<p+\alpha<1$.
For any $j\neq 1$, the transitions are:
\begin{equation*}
    P_{\mathcal{M}}(s^2_{ij}|s^2_{ij},a_0) = p, \quad P_{\mathcal{M}}(s^3_{ij}|s^2_{ij},a_0) = 1-p.
\end{equation*}
The corresponding reward is deterministic and given by $r_{\mathcal{M}}(s^2_{ij},a_0)=1$ for all $s^2_{ij}\in\mathcal{S}_2$. Finally, taking action $a_0$ in any state $s^3_{ij}\in\mathcal{S}_3$ results in a self-transition with probability one and yields a deterministic reward $r_{\mathcal{M}}(s^3_{ij},a_0) = 0$.

For a given state $s^1_{i}\in\mathcal{S}_1$ and a suboptimal action $a_j\in\mathcal{A}_1\setminus \{a_1\}$, we construct the alternative MDP $\tilde{\mathcal{M}}$ by perturbing the transition at $s^2_{ij}$ such that:
$$
    P_{\tilde{\mathcal{M}}}(s^2_{ij}|s^2_{ij},a_0) = p+\alpha+\epsilon,\quad P_{\tilde{\mathcal{M}}}(s^3_{ij}|s^2_{ij},a_0) = 1-(p+\alpha+\epsilon),
$$
where $\epsilon>0$ satisfies $p+\alpha+\epsilon<1$. All remaining transition probabilities and reward functions coincide with those of $\mathcal{M}$.
According to the Bellman optimality equation, we have that for each state $s^1_{i}\in \mathcal{S}_1$,
$$
Q^{\pi^*_{\mathcal{M}}}_{\mathcal{M}}(s^1_i,a_1) = \frac{\gamma}{1-\gamma (p+\alpha)} > \frac{\gamma}{1-\gamma p} = Q^{\pi^*_{\mathcal{M}}}_{\mathcal{M}}(s^1_i,a_j),\quad \forall a_j\neq a_1.
$$
Since $\alpha>0$, this means $\pi^*_{\mathcal{M}}(s^1_i) = a_1$ for any state $s^1_i\in \mathcal{S}_1$. However, under the MDP $\tilde{\mathcal{M}}$, for the given state $s^1_i\in\mathcal{S}_1$ and action $a_j\in\mathcal{A}_1\setminus \{a_1\}$, we have
$$
Q^{\pi^*_{{\mathcal{M}}}}_{\tilde{\mathcal{M}}}(s^1_i,a_j) = \frac{\gamma}{1-\gamma (p+\alpha+\epsilon)} > \frac{\gamma}{1-\gamma (p+\alpha)} = Q^{\pi^*_{{\mathcal{M}}}}_{\tilde{\mathcal{M}}}(s^1_i,a_1).
$$
and hence $(P_{\tilde{\mathcal{M}}},r_{\tilde{\mathcal{M}}})$ belongs to the set $\mathcal{E}_{s^1_i,a_j}$. 
Based on the definition of $\mathcal{R}^*$ and the setting of MDPs $\mathcal{M}$ and $\tilde{\mathcal{M}}$, we have that
$$
    \mathcal{R}^* \leq \max_{\omega\in\Omega}\mathcal{R}(\mathcal{M},\omega)
    \leq \max_{\omega\in\Omega}\min_{s^1_i\in\mathcal{S}_1,a_j\in\mathcal{A}_1\setminus\{a_1\}}\omega_{s^2_{ij}a_0} I_1(x(s^2_{ij},a_0)).
$$

Note that $I_1(x(s^2_{ij},a_0))$ is the Fenchel-Legendre transform of the logarithmic moment generating function of $X_{\mathcal{M}}(s^2_{ij},a_0) = (X_{\mathcal{M}}(s^2_{ij}|s^2_{ij},a_0),X_{\mathcal{M}}(s^3_{ij}|s^2_{ij},a_0))$, and we have that
\begin{equation*}
    \mathbb{E}[X_{\mathcal{M}}(s^2_{ij},a_0)] = (p,1-p).
\end{equation*}
By the definition of $I_1(x(s^2_{ij},a_0))$, we have
$$
I_1(x(s^2_{ij},a_0)) 
= \sup\{\lambda_1 x(s^2_{ij}|s^2_{ij},a_0)+\lambda_2x(s^3_{ij}|s^2_{ij},a_0)-\log(p\exp(\lambda_1)+(1-p)\exp(\lambda_2))\}.
$$
For notation simplicity, we denote by $x_1 = x(s^2_{ij}|s^2_{ij},a_0), x_2 = x(s^3_{ij}|s^2_{ij},a_0)$ and $A = p\exp(\lambda_1)+(1-p)\exp(\lambda_2)$. By solving the optimization problem, it holds that
\begin{equation*}
    \lambda_1 = \log \left(\frac{x_1A}{p}\right),\quad \lambda_2 = \log\left(\frac{x_2A}{1-p}\right),
\end{equation*}
and 
\begin{equation*}
    I_1(x(s^2_{ij},a_0)) = x_1 \log\left(\frac{x_1}{p}\right) + x_2 \log\left(\frac{x_2}{1-p}\right).
\end{equation*}
Let $x_1 = p+\alpha+\epsilon$ and $x_2 = 1-(p+\alpha+\epsilon)$, we can obtain that
$$
    I_1(x(s^2_{ij},a_0)) = (p+\alpha+\epsilon)\log\frac{p+\alpha+\epsilon}{p} + (1-p-\alpha-\epsilon)\log\frac{1-p-\alpha-\epsilon}{1-p}.
$$
For any $z\in(-1,\infty)$, as $z\rightarrow 0$, the natural logarithm admits the Taylor expansion
\begin{equation*}
    \log(1+z) = z -\frac{z^2}{2} + O(z^3),
\end{equation*}
where $O(z^3)$ denotes the remainder term of the expansion and satisfies $\lim_{z\rightarrow 0}O(z^3)/z=0$. Then, it holds that
$$
(p+\alpha+\epsilon)\log\frac{p+\alpha+\epsilon}{p}
=(\alpha+\epsilon) + \frac{(\alpha+\epsilon)^2}{2p}+O\left((\alpha+\epsilon)^3\right),
$$
and 
$$
(1-p-\alpha-\epsilon)\log \frac{1-p-\alpha-\epsilon}{1-p}
=-(\alpha+\epsilon)+\frac{(\alpha+\epsilon)^2}{2(1-p)} + O\left((\alpha+\epsilon)^3\right)
$$

Therefore, we obtain that as $(\alpha+\epsilon)\rightarrow 0$,
\begin{equation*}
     I_1(x(s^2_{ij},a_0)) = \frac{(\alpha+\epsilon)^2}{2p} + \frac{(\alpha+\epsilon)^2}{2(1-p)} + O\left((\alpha+\epsilon)^3\right).
\end{equation*}
By letting \(\epsilon\downarrow 0\) and choosing
$
    p=\frac{4\gamma-1}{3\gamma},
$
where \(\gamma\in(1/4,1)\), we have
\begin{equation*}
    \Delta_{\min}(\mathcal M)
    =
    \frac{\gamma}{1-\gamma(p+\alpha)}-\frac{\gamma}{1-\gamma p}.
\end{equation*}
We choose $\alpha$ so that $\Delta_{\min}(\mathcal M)=\Delta_0$. Equivalently, as $\Delta_0\downarrow0$, a first-order expansion gives
\begin{equation*}
    \alpha
    =
    \frac{(1-\gamma p)^2}{\gamma^2}\Delta_{0}
    +O(\Delta_{0}^2).
\end{equation*}
Substituting this relation into the above expansion gives
$$
    I_1(x(s^2_{ij},a_0))
    =
    \frac{\alpha^2}{2p}
    +
    \frac{\alpha^2}{2(1-p)}
    +
    O(\alpha^3)
    =
    O\left((1-\gamma)^3\Delta_{0}^2\right).
$$

Therefore,
\[
\mathcal R^*
\le
\max_{\omega\in\Omega}
\min_{s^1_i\in\mathcal S_1,\,
a_j\in\mathcal A_1\setminus\{a_1\}}
\omega_{s^2_{ij}a_0}
I_1(x(s^2_{ij},a_0)).
\]
The right-hand side is maximized, up to a constant factor, by allocating sampling effort uniformly over the critical pairs
$
\{(s^2_{ij},a_0): i\in[K],\ j\in[L]\}.
$
Since the number of such pairs is of order $C(S,A)$, we obtain
\[
\mathcal R^*
=
O\left(
\frac{(1-\gamma)^3\Delta_0^2}{C(S,A)}
\right).
\]
This proves the desired upper bound on the worst-case optimal exponential decay rate.

\section{Proof of Lemma~\ref{lemma: approximate constraint}}
Recall the definitions of the model deviations between the alternative model $\tilde{\mathcal{M}}$ and the nominal model $\mathcal{M}$: $\Delta_{r}(s,a) := r_{\tilde{\mathcal{M}}}(s,a)-r_{\mathcal{M}}(s,a)$, and $\Delta_{p}(\cdot|s,a):= P_{\tilde{\mathcal{M}}}(\cdot|s,a) - P_{{\mathcal{M}}}(\cdot|s,a)$, and $\Delta_{V}:=V^{\pi^*_{\mathcal{M}}}_{\tilde{\mathcal{M}}}- V^{\pi^*_{\mathcal{M}}}_{{\mathcal{M}}}$. Additionally, recall the optimality gap $\Delta_{sa} := V^{\pi^*_{\mathcal{M}}}_{\mathcal{M}}(s)-Q^{\pi^*_{\mathcal{M}}}_{\mathcal{M}}(s,a)$.

By the definition of the $Q$-function, we expand $Q^{\pi^*_{\mathcal{M}}}_{\tilde{\mathcal{M}}}(s,a)$ as follows:
\begin{align}
Q^{\pi^*_{\mathcal{M}}}_{\tilde{\mathcal{M}}}(s,a) &= r_{\tilde{\mathcal{M}}}(s,a) + \gamma \sum_{s^\prime \in \mathcal{S}}P_{\tilde{\mathcal{M}}}(s^\prime|s,a)V^{\pi^*_{\mathcal{M}}}_{\tilde{\mathcal{M}}}(s^\prime)\nonumber\\
 &=r_{{\mathcal{M}}}(s,a) + \Delta_r(s,a) + \gamma \sum_{s^\prime \in \mathcal{S}} (P_{{\mathcal{M}}}(s^\prime|s,a)+\Delta_p(s^\prime|s,a))V^{\pi^*_{\mathcal{M}}}_{\tilde{\mathcal{M}}}(s^\prime)\nonumber\\
 &=r_{{\mathcal{M}}}(s,a) + \gamma P_{\mathcal{M}}(s,a)^\top V^{\pi^*_{\mathcal{M}}}_{{\mathcal{M}}} + \gamma P_{\mathcal{M}}(s,a)^\top \Delta_{V} + \Delta_r(s,a) + \gamma \Delta_{p}(s,a)^\top V^{\pi^*_{\mathcal{M}}}_{\tilde{\mathcal{M}}}\nonumber\\
 &= Q^{\pi^*_{\mathcal{M}}}_{{\mathcal{M}}}(s,a) + \gamma P_{\tilde{\mathcal{M}}}(s,a)^\top \Delta_{V} + \Delta_r(s,a) + \gamma \Delta_{p}(s,a)^\top V^{\pi^*_{\mathcal{M}}}_{{\mathcal{M}}} \nonumber\\
&=V^{\pi^*_{\mathcal{M}}}_{{\mathcal{M}}}(s) - \Delta_{sa} + \gamma P_{\tilde{\mathcal{M}}}(s,a)^\top \Delta_{V} + \Delta_r(s,a) + \gamma \Delta_{p}(s,a)^\top V^{\pi^*_{\mathcal{M}}}_{{\mathcal{M}}}
\label{eq:Q_expansion_final}
\end{align}

Given the false selection condition $Q^{\pi^*_{\mathcal{M}}}_{\tilde{\mathcal{M}}}(s,a) > V^{\pi^*_{\mathcal{M}}}_{\tilde{\mathcal{M}}}(s)$, we substitute \eqref{eq:Q_expansion_final} into the LHS and use $V^{\pi^*_{\mathcal{M}}}_{\tilde{\mathcal{M}}}(s) = V^{\pi^*_{\mathcal{M}}}_{{\mathcal{M}}}(s) + \Delta_V(s)$ for the RHS:

$$
V^{\pi^*_{\mathcal{M}}}_{{\mathcal{M}}}(s) - \Delta_{sa} + \gamma P_{\tilde{\mathcal{M}}}(s,a)^\top \Delta_{V} + \Delta_r(s,a) + \gamma \Delta_{p}(s,a)^\top V^{\pi^*_{\mathcal{M}}}_{{\mathcal{M}}}>V^{\pi^*_{\mathcal{M}}}_{{\mathcal{M}}}(s) + \Delta_V(s).\nonumber
$$
Subtracting $V^{\pi^*_{\mathcal{M}}}_{{\mathcal{M}}}(s)$ from both sides and rearranging terms yields:
\begin{align}
\label{eq:main_inequality}
    (\gamma P_{\tilde{\mathcal{M}}}(s,a)-e_{s})^\top \Delta_{V}+\Delta_r(s,a) + \gamma \Delta_{p}(s,a)^\top V^{\pi^*_{\mathcal{M}}}_{{\mathcal{M}}}> \Delta_{sa}.
\end{align}
where $e_s$ denotes the standard basis vector for state $s$ (i.e., $\Delta_V(s) = e_s^\top \Delta_V$).
We now derive an upper bound for $\Delta_V$. For any state $s \in \mathcal{S}$, the value difference satisfies:
\begin{align}
\Delta_V(s) &=V^{\pi^*_{\mathcal{M}}}_{\tilde{\mathcal{M}}}(s) - V^{\pi^*_{\mathcal{M}}}_{\mathcal{M}}(s)\nonumber\\ 
&=r_{\tilde{\mathcal{M}}}(s,\pi^*_{\mathcal{M}}(s)) - r_{{\mathcal{M}}}(s,\pi^*_{\mathcal{M}}(s)) + \gamma P_{\tilde{\mathcal{M}}}(s,\pi^*_{\mathcal{M}}(s))^\top V^{\pi^*_{\mathcal{M}}}_{\tilde{\mathcal{M}}} - \gamma P_{{\mathcal{M}}}(s,\pi^*_{\mathcal{M}}(s))^\top V^{\pi^*_{\mathcal{M}}}_{{\mathcal{M}}} \nonumber\\
&=\Delta_r(s,\pi^*_{\mathcal{M}}(s)) + \gamma \Delta_{p}(s,\pi^*_{\mathcal{M}}(s))^\top V^{\pi^*_{\mathcal{M}}}_{\tilde{\mathcal{M}}} + \gamma P_{{\mathcal{M}}}(s,\pi^*_{\mathcal{M}}(s))^\top (V^{\pi^*_{\mathcal{M}}}_{\tilde{\mathcal{M}}} - V^{\pi^*_{\mathcal{M}}}_{\mathcal{M}})\nonumber\\
&=\Delta_r(s,\pi^*_{\mathcal{M}}(s)) + \gamma \Delta_{p}(s,\pi^*_{\mathcal{M}}(s))^\top V^{\pi^*_{\mathcal{M}}}_{{\mathcal{M}}}+\gamma P_{\tilde{\mathcal{M}}}(s,\pi^*_{\mathcal{M}}(s))^\top \Delta_V\nonumber.
\end{align}
Taking the absolute value and maximizing over all states $s \in \mathcal{S}$:
$$
    \max_{s^\prime\in\mathcal{S}} \left|\Delta_V(s^\prime)\right| \leq \max_{s^\prime\in\mathcal{S}} \left|\Delta_r(s^\prime,\pi^*_{\mathcal{M}}(s^\prime)) + \gamma \Delta_{p}(s^\prime,\pi^*_{\mathcal{M}}(s^\prime))^\top V^{\pi^*_{\mathcal{M}}}_{{\mathcal{M}}}\right| + \gamma \max_{s^\prime\in\mathcal{S}} \left|\Delta_V(s^\prime)\right|,
$$
Rearranging for $\max_{s\in\mathcal{S}} \left|\Delta_V(s)\right|$, we obtain
\begin{equation}
\label{eq:delta_V_bound}
\begin{aligned}
      \max_{s^\prime\in\mathcal{S}} \left|\Delta_V(s^\prime)\right|  &\leq \frac{1}{1-\gamma}\max_{s^\prime\in\mathcal{S}} \left|\Delta_r(s^\prime,\pi^*_{\mathcal{M}}(s^\prime)) + \gamma \Delta_{p}(s^\prime,\pi^*_{\mathcal{M}}(s^\prime))^\top V^{\pi^*_{\mathcal{M}}}_{{\mathcal{M}}}\right|\\
      &\leq \frac{1}{1-\gamma} \max_{s^\prime\in\mathcal{S}} \left|\Delta_r(s^\prime,\pi^*_{\mathcal{M}}(s^\prime))\right| + \frac{\gamma}{1-\gamma} \max_{s^\prime\in\mathcal{S}}\left|\Delta_{p}(s^\prime,\pi^*_{\mathcal{M}}(s^\prime))^\top V^{\pi^*_{\mathcal{M}}}_{{\mathcal{M}}}\right|.
\end{aligned}
\end{equation}
We now derive a bound to the first term on the LHS of \eqref{eq:main_inequality}. Using the triangle inequality and the property that $\|P_{\tilde{\mathcal{M}}}(s,a)\|_1=1$, we have:
$$
    (\gamma P_{\tilde{\mathcal{M}}}(s,a)-e_s)^\top \Delta_{V} \leq \left|(\gamma P_{\tilde{\mathcal{M}}}(s,a)-e_s)^\top \Delta_{V}\right| \leq \|\gamma P_{\tilde{\mathcal{M}}}(s,a)-e_s\|_1 \|\Delta_V\|_\infty \leq (1+\gamma)\max_{s^\prime\in\mathcal{S}} \left|\Delta_V(s^\prime)\right| .
$$

Substituting $(1+\gamma)\max_{s^\prime\in\mathcal{S}} \left|\Delta_V(s^\prime)\right|$ and then using \eqref{eq:delta_V_bound}:
$$
 \frac{1+\gamma}{1-\gamma} \max_{s^\prime\in\mathcal{S}} \left|\Delta_r(s^\prime,\pi^*_{\mathcal{M}}(s^\prime))\right| + \frac{\gamma(1+\gamma)}{1-\gamma} \max_{s^\prime\in\mathcal{S}}\left|\Delta_{p}(s^\prime,\pi^*_{\mathcal{M}}(s^\prime))^\top V^{\pi^*_{\mathcal{M}}}_{{\mathcal{M}}}\right|+\Delta_r(s,a) + \gamma \Delta_{p}(s,a)^\top V^{\pi^*_{\mathcal{M}}}_{{\mathcal{M}}}> \Delta_{sa}.
$$

\section{Proof of Lemma~\ref{lemma: I1 lb}}
For a given state-action pair $(s,a)$, we define a random variable
$Z:\mathcal{S}\rightarrow\mathbb{R}$ associated with the next state $S^\prime$ (random variable), centered around its expectation under the nominal model $\mathcal{M}$:
\begin{equation*}
    Z(S^\prime) := V^{\pi^*_{\mathcal{M}}}_{\mathcal{M}}(S^\prime) - \sum_{k \in\mathcal{S}}P_{\mathcal{M}}(k|s,a)V^{\pi^*_{\mathcal{M}}}_{\mathcal{M}}(k).
\end{equation*}
Then, the variable $Z$ satisfies the following key properties. First, the expectation of $Z$ under distribution $P_{\mathcal{M}}(s,a)$ satisfies
$$
\mathbb{E}_{P_{\mathcal{M}}(s,a)}[Z] = \sum_{s^\prime\in\mathcal{S}}P_{\mathcal{M}}(s^\prime|s,a)V^{\pi^*_{\mathcal{M}}}_{\mathcal{M}}(s^\prime) - \sum_{s^\prime\in\mathcal{S}}P_{\mathcal{M}}(s^\prime|s,a)V^{\pi^*_{\mathcal{M}}}_{\mathcal{M}}(s^\prime) = 0.
$$
Second, the expectation of $Z$ under distribution $P_{\tilde{\mathcal{M}}}(s,a)$ satisfies
$$
\mathbb{E}_{P_{\tilde{\mathcal{M}}}(s,a)}[Z] = \sum_{s^\prime \in\mathcal{S}}P_{\tilde{\mathcal{M}}}(s^\prime|s,a)V^{\pi^*_{\mathcal{M}}}_{{\mathcal{M}}}(s^\prime) - \sum_{s^\prime \in\mathcal{S}}P_{{\mathcal{M}}}(s^\prime|s,a)V^{\pi^*_{\mathcal{M}}}_{{\mathcal{M}}}(s^\prime) = \Delta_{p}(s,a)^\top V^{\pi^*_{\mathcal{M}}}_{{\mathcal{M}}}. 
$$
Third, the second-order moment of $Z$ under distribution $P_{\mathcal{M}}(s,a)$ satisfies
$$
\mathbb{E}_{P_{{\mathcal{M}}}(s,a)}[Z^2] = \sum_{s^\prime\in\mathcal{S}}P_{\mathcal{M}}(s^\prime|s,a)\left(V^{\pi^*_{\mathcal{M}}}_{\mathcal{M}}(s^\prime) - \sum_{s^\prime\in\mathcal{S}}P_{\mathcal{M}}(s^\prime|s,a)V^{\pi^*_{\mathcal{M}}}_{\mathcal{M}}(s^\prime)\right)^2 =\text{Var}_{s^\prime \sim P_{{\mathcal{M}}}(s,a)}\left[V^{\pi^*_{\mathcal{M}}}_{\mathcal{M}}(s^\prime)\right]
$$
Last, $Z$ is a bounded random variable such that
$|Z| \leq {1}/(1-\gamma)$. 
Since $I_1(x(s,a))$ is the Fenchel-Legendre transform of the logarithmic moment generating functions of $X_{\mathcal{M}}(s,a)$ and is defined as
$$
     I_1(x(s,a)) = \sup_{\rho(s,a)}\left({\rho(s,a)^\top}x(s,a)- \log\mathbb{E}_{P_{\mathcal{M}}(s,a)}\left[\exp(\rho(s,a)^\top X_{\mathcal{M}}(s,a))\right]\right).
$$
By choosing $\rho(s,a)$ as the column vector $(\eta_1 Z(s'))_{s'\in\mathcal S}\in\mathbb R^S$ for some constant $\eta_1>0$,  it holds that
\begin{equation*}
    I_1(x(s,a)) \geq \eta_1 \mathbb{E}_{{P}_{\tilde{\mathcal{M}}}(s,a)}[Z] - \log \mathbb{E}_{P_{\mathcal{M}}(s,a)}[\exp(\eta_1 Z)],
\end{equation*}
which further implies that
\begin{equation}
\label{eq: variational_step}
    \Delta_{p}(s,a)^\top V^{\pi^*_{\mathcal{M}}}_{{\mathcal{M}}} = \mathbb{E}_{{P}_{\tilde{\mathcal{M}}}(s,a)}[Z] \leq \frac{1}{\eta_1}I_1(x(s,a)) +  \frac{1}{\eta_1} \log \mathbb{E}_{P_{\mathcal{M}}(s,a)}[\exp(\eta_1 Z)].
\end{equation}

Since $Z$ is a zero-mean random variable bounded by $1/(1-\gamma)$, for any $\eta_1\in(0,3(1-\gamma))$, the log-moment generating function satisfies the standard Bernstein-type bound:
\begin{equation*}
    \log \mathbb{E}_{P_{\mathcal{M}}(s,a)}[e^{\eta_1 Z}] \leq \frac{\eta_1^2 \mathbb{V}_{ P_{{\mathcal{M}}}(s,a)}[V^{\pi^*_{\mathcal{M}}}_{\mathcal{M}}]}{2\left(1 - \frac{\eta_1}{3(1-\gamma)}\right)},
\end{equation*}
where $\mathbb{V}_{ P_{{\mathcal{M}}}(s,a)}[V^{\pi^*_{\mathcal{M}}}_{\mathcal{M}}]$ is the variance of random variable $V^{\pi^*_{\mathcal{M}}}_{\mathcal{M}}(s^\prime)$ with $s^\prime\sim P_{{\mathcal{M}}}(s,a)$.
Substituting this upper bound into \eqref{eq: variational_step}, we have:
\begin{equation}
    \Delta_{p}(s,a)^\top V^{\pi^*_{\mathcal{M}}}_{{\mathcal{M}}} \leq \frac{1}{\eta_1}I_1(x(s,a)) +  \frac{\eta_1 \mathbb{V}_{ P_{{\mathcal{M}}}(s,a)}[V^{\pi^*_{\mathcal{M}}}_{\mathcal{M}}]}{2\left(1 - \frac{\eta_1}{3(1-\gamma)}\right)}.
\end{equation}
By symmetry (repeating the argument with test function $-\eta_1 Z$), this bound holds for the absolute value $|\Delta_{p}(s,a)^\top V^{\pi^*_{\mathcal{M}}}_{{\mathcal{M}}}|$. This implies:
$$
    I_1(x(s,a)) \geq \sup_{\eta_1 \in (0, 3(1-\gamma))} \left( \eta_1 |\Delta_{p}(s,a)^\top V^{\pi^*_{\mathcal{M}}}_{{\mathcal{M}}}| - \frac{\eta^2_1 \mathbb{V}_{ P_{{\mathcal{M}}}(s,a)}[V^{\pi^*_{\mathcal{M}}}_{\mathcal{M}}]}{2\left(1 - \frac{\eta_1}{3(1-\gamma)}\right)} \right).
$$

By choosing $$
\eta_1 = \frac{|\Delta_{p}(s,a)^\top V^{\pi^*_{\mathcal{M}}}_{{\mathcal{M}}}|}{\mathbb{V}_{ P_{{\mathcal{M}}}(s,a)}[V^{\pi^*_{\mathcal{M}}}_{\mathcal{M}}]+(|\Delta_{p}(s,a)^\top V^{\pi^*_{\mathcal{M}}}_{{\mathcal{M}}}|/3(1-\gamma))}$$ which yields the functional form:
$$
    I_1(x(s,a)) \geq \frac{(\Delta_{p}(s,a)^\top V^{\pi^*_{\mathcal{M}}}_{{\mathcal{M}}})^2}{2\mathbb{V}_{ P_{{\mathcal{M}}}(s,a)}[V^{\pi^*_{\mathcal{M}}}_{\mathcal{M}}] + 2(|\Delta_{p}(s,a)^\top V^{\pi^*_{\mathcal{M}}}_{{\mathcal{M}}}|/3(1-\gamma))}.
$$
Rearranging the inequality, we obtain a quadratic inequality:
$$
    {(\Delta_{p}(s,a)^\top V^{\pi^*_{\mathcal{M}}}_{{\mathcal{M}}})^2}-2\mathbb{V}_{ P_{{\mathcal{M}}}(s,a)}[V^{\pi^*_{\mathcal{M}}}_{\mathcal{M}}]I_1(x(s,a)) -
    \frac{2|\Delta_{p}(s,a)^\top V^{\pi^*_{\mathcal{M}}}_{{\mathcal{M}}}|}{3(1-\gamma)} I_1(x(s,a)) \leq 0.
$$
The positive root of the corresponding quadratic equation sets the upper bound for $|\Delta_{p}(s,a)^\top V^{\pi^*_{\mathcal{M}}}_{{\mathcal{M}}}|$:
$$
    |\Delta_{p}(s,a)^\top V^{\pi^*_{\mathcal{M}}}_{{\mathcal{M}}}| \leq \frac{\frac{2}{3(1-\gamma)} I_1(x(s,a))+\sqrt{\frac{4}{9(1-\gamma)^2} I_1(x(s,a))^2+8\mathbb{V}_{ P_{{\mathcal{M}}}(s,a)}[V^{\pi^*_{\mathcal{M}}}_{\mathcal{M}}]I_1(x(s,a))}}{2}.
$$
Using the inequality $\sqrt{x+y} \leq \sqrt{x} + \sqrt{y}$, we simplify the term:
$$
    \sqrt{\frac{4}{9(1-\gamma)^2} I_1(x(s,a))^2+8\mathbb{V}_{ P_{{\mathcal{M}}}(s,a)}[V^{\pi^*_{\mathcal{M}}}_{\mathcal{M}}]I_1(x(s,a))} \le \frac{2}{3(1-\gamma)}I_1(x(s,a)) + 2\sqrt{2\mathbb{V}_{ P_{{\mathcal{M}}}(s,a)}[V^{\pi^*_{\mathcal{M}}}_{\mathcal{M}}]I_1(x(s,a))}.
$$
Substituting this back, we obtain that
$$
    |\Delta_{p}(s,a)^\top V^{\pi^*_{\mathcal{M}}}_{{\mathcal{M}}}| \leq {\frac{2}{3(1-\gamma)}I_1(x(s,a)) + \sqrt{2\mathbb{V}_{ P_{{\mathcal{M}}}(s,a)}[V^{\pi^*_{\mathcal{M}}}_{\mathcal{M}}]I_1(x(s,a))}}.
$$
Finally, to derive the quadratic bound, we square both sides:
$$
        (\Delta_{p}(s,a)^\top V^{\pi^*_{\mathcal{M}}}_{{\mathcal{M}}})^2 \leq 2\mathbb{V}_{ P_{{\mathcal{M}}}(s,a)}[V^{\pi^*_{\mathcal{M}}}_{\mathcal{M}}]I_1(x(s,a)) +\frac{4\sqrt{2}(\mathbb{V}_{ P_{{\mathcal{M}}}(s,a)}[V^{\pi^*_{\mathcal{M}}}_{\mathcal{M}}])^{\frac{1}{2}}I_1(x(s,a))^{\frac{3}{2}}}{3(1-\gamma)}  + \frac{4I_1(x(s,a))^2}{9(1-\gamma)^2}.
$$

\section{Proof of Lemma~\ref{lemma: I2 lb}}
For a given state-action pair $(s,a)$, let $R$ be the random variable representing the immediate reward. Under the nominal model $\mathcal{M}$, $R$ follows the distribution of $R_{\mathcal{M}}(s,a)$, and under the alternative model $\tilde{\mathcal{M}}$, it follows $R_{\tilde{\mathcal{M}}}(s,a)$.
We define a centered random variable $Z$ as:
$
    Z = R - \mathbb{E}[R_{{\mathcal{M}}}(s,a)].
$
Under the nominal model $\mathcal{M}$, $Z$ satisfies the following properties. First, the expectation of $Z$ is zero $\mathbb{E}_{\mathcal{M}}[Z] = 0$. Second, the expectation under the alternative model $\tilde{\mathcal{M}}$ yields
\begin{equation}
\label{eq: alter_expect}
\mathbb{E}_{\tilde{\mathcal{M}}}[Z] = \mathbb{E}[R_{\tilde{\mathcal{M}}}(s,a)] - \mathbb{E}[R_{{\mathcal{M}}}(s,a)] =\Delta_{r}(s,a).
\end{equation} Third, the second moment matches the variance of the reward:
$
\mathbb{E}_{\mathcal{M}}[Z^2] = \mathbb{V}[R_{\mathcal{M}}(s,a)],$ 
where $\mathbb{V}[R_{\mathcal{M}}(s,a)]$ is the variance of random variable $R_{\mathcal{M}}(s,a)$. Last, $Z$ is a bounded random variable such that $|Z|\le 1$.
The rate function $I_2(y(s,a))$ is defined as the Fenchel-Legendre transform of the logarithmic moment generating function of the reward distribution $R_{\mathcal{M}}(s,a)$. Specifically, we set the target value $y(s,a) = \mathbb{E}_{P_{\tilde{\mathcal{M}}}}[R]$. For any scalar $\eta_2 > 0$, we have:
$$
I_2(y(s,a)) = \sup_{\lambda(s,a)}\left({\lambda(s,a)}y(s,a)-\log\mathbb{E}[\exp(\lambda(s,a)R_{\mathcal{M}}(s,a))]\right)\geq \eta_2\mathbb{E}_{\tilde{\mathcal{M}}}[Z] - \log \mathbb{E}_{\mathcal{M}}[\exp(\eta_2Z)].
$$

Substituting the equation \eqref{eq: alter_expect} and rearranging terms, we obtain:
\begin{equation}
\label{eq: reward_variational_step}
    \eta_2 \Delta_{r}(s,a) \leq I_2(y(s,a)) + \log \mathbb{E}_{\mathcal{M}}[\exp(\eta_2 Z)].
\end{equation}

Since $Z$ is a zero-mean random variable bounded by $1$, for any $\eta_2 \in (0, 3)$, the log-moment generating function satisfies the standard Bernstein-type bound:
\begin{equation*}
    \log \mathbb{E}_{\mathcal{M}}[e^{\eta_2 Z}] \leq \frac{\eta_2^2 \mathbb{V}[R_{\mathcal{M}}(s,a)]}{2(1 - \eta_2 /3)}.
\end{equation*}
Substituting this upper bound into \eqref{eq: reward_variational_step}, we have:
\begin{equation*}
    \eta_2 \Delta_{r}(s,a) \leq I_2(y(s,a)) + \frac{\eta_2^2 \mathbb{V}[R_{\mathcal{M}}(s,a)]}{2(1 - \eta_2 /3)}.
\end{equation*}
By symmetry (applying the same logic with test function $-\eta_2 Z$), this bound holds for the absolute value $|\Delta_{r}(s,a)|$. This implies:
\begin{equation*}
    I_2(y(s,a)) \geq \sup_{\eta_2 \in (0, 3)} \left( \eta_2 |\Delta_{r}(s,a)| - \frac{\eta_2^2 \mathbb{V}[R_{\mathcal{M}}(s,a)]}{2(1 - \eta_2 /3)} \right).
\end{equation*}

By choosing
\begin{equation*}
    \eta_2 = \frac{|\Delta_{r}(s,a)|}{\mathbb{V}[R_{\mathcal{M}}(s,a)] + |\Delta_{r}(s,a)| /3}
\end{equation*} yielding the functional form:
\begin{equation}
\label{eq: reward_bernstein_form}
    I_2(y(s,a)) \geq \frac{\Delta_{r}(s,a)^2}{2\mathbb{V}[R_{\mathcal{M}}(s,a)] + 2|\Delta_{r}(s,a)|/3}.
\end{equation}
Rearranging the inequality \eqref{eq: reward_bernstein_form} leads to a quadratic inequality with respect to $|\Delta_r(s,a)|$:
$$
    \Delta_{r}(s,a)^2 - \frac{2}{3}  I_2(y(s,a)) |\Delta_{r}(s,a)| - 2 \mathbb{V}[R_{\mathcal{M}}(s,a)] I_2(y(s,a)) \leq 0.
$$
The positive root of the corresponding quadratic equation sets the upper bound:
$$
    |\Delta_{r}(s,a)| \leq \frac{\frac{2}{3}  I_2(y(s,a)) + \sqrt{\frac{4}{9}  I_2(y(s,a))^2 + 8 \mathbb{V}[R_{\mathcal{M}}(s,a)] I_2(y(s,a))}}{2}.
$$
Using the sub-additivity of the square root $\sqrt{x+y} \leq \sqrt{x} + \sqrt{y}$, we simplify the bound:
\begin{equation*}
    |\Delta_{r}(s,a)| \leq \frac{2}{3}  I_2(y(s,a)) + \sqrt{2 \mathbb{V}[R_{\mathcal{M}}(s,a)] I_2(y(s,a))}. 
\end{equation*}
Finally, squaring both sides yields the exact quadratic bound:
\begin{align*}
    \Delta^2_{r}(s,a) &\leq \left( \sqrt{2 \mathbb{V}[R_{\mathcal{M}}(s,a)] I_2(y(s,a))} + \frac{2}{3}  I_2(y(s,a)) \right)^2 \\
    &= 2 \mathbb{V}[R_{\mathcal{M}}(s,a)] I_2(y(s,a)) + \frac{4\sqrt{2}(\mathbb{V}[R_{\mathcal{M}}(s,a)])^{\frac{1}{2}} I_2(y(s,a))^{\frac{3}{2}}}{3}   + \frac{4I_2(y(s,a))^2}{9}.
\end{align*}

\section{Proof of Theorem~\ref{thm: rate lower bound opt}}
To make the optimization tractable, we introduce auxiliary variables $\beta = (\beta_1, \beta_2, \beta_3, \beta_4)$ satisfying $\beta_i \geq 0$ and $\sum_{i=1}^4 \beta_i = 1$. These variables partition the total optimality gap $\Delta_{sa}$ among the four sources of error:
\begin{equation}
\label{eq: part_error}
\begin{aligned}
    \frac{1+\gamma}{1-\gamma} \max_{s'\in\mathcal{S}} |\Delta_r(s',\pi_{\mathcal{M}}^*(s'))| &\geq \beta_1 \Delta_{sa}, \\
\frac{\gamma(1+\gamma)}{1-\gamma} \max_{s'\in\mathcal{S}} \left|\Delta_{p}(s',\pi_{\mathcal{M}}^*(s'))^\top V^{\pi^*_{\mathcal{M}}}_{{\mathcal{M}}}\right| &\geq \beta_2 \Delta_{sa}, \\
    |\Delta_r(s,a)| &\geq \beta_3 \Delta_{sa}, \\
    \gamma \left|\Delta_{p}(s,a)^\top V^{\pi^*_{\mathcal{M}}}_{{\mathcal{M}}}\right| &\geq \beta_4 \Delta_{sa}. 
\end{aligned}
\end{equation}

From Lemma \ref{lemma: I1 lb}, we know that
$$
        (\Delta_{p}(s,a)^\top V^{\pi^*_{\mathcal{M}}}_{{\mathcal{M}}})^2 \leq 2\mathbb{V}_{ P_{{\mathcal{M}}}(s,a)}[V^{\pi^*_{\mathcal{M}}}_{\mathcal{M}}]I_1(x(s,a)) +\frac{4\sqrt{2}(\mathbb{V}_{ P_{{\mathcal{M}}}(s,a)}[V^{\pi^*_{\mathcal{M}}}_{\mathcal{M}}])^{\frac{1}{2}}I_1(x(s,a))^{\frac{3}{2}}}{3(1-\gamma)}  + \frac{4I_1(x(s,a))^2}{9(1-\gamma)^2}.
$$
We focus on the \emph{asymptotic regime} relevant to hard instances where the nominal model $\mathcal{M}$ and the alternative model $\tilde{\mathcal{M}}$ are close. In this regime, the rate function value $I_1(x(s,a))\rightarrow 0$, making the higher-order term $I_1(x(s,a))^{\frac{3}{2}}$ and $I_1(x(s,a))^2$ negligible compared to the leading quadratic term $I_1(x(s,a))$. Therefore, we utilize the leading-order lower bounds:
\begin{equation*}
    I_1(x(s,a))
    \ge
    \frac{\bigl(\Delta_{p}(s,a)^\top V^{\pi^*_{\mathcal M}}_{\mathcal M}\bigr)^2}
    {2\,\mathbb V_{P_{\mathcal M}(s,a)}\!\left[V^{\pi^*_{\mathcal M}}_{\mathcal M}\right]}
    \,(1-o(1)),
\end{equation*}

Similarly, from Lemma \ref{lemma: I2 lb} we also have that
\begin{equation*}
    I_2(y(s,a))
    \ge
    \frac{\Delta_r(s,a)^2}
    {2\,\mathbb V\!\left[R_{\mathcal M}(s,a)\right]}
    \,(1-o(1)).
\end{equation*}

Substituting these lower bounds to the partitioned constraints \eqref{eq: part_error}:
$$
        \max_{s^\prime} I_2(y(s^\prime,\pi^*_{\mathcal{M}}(s^\prime)))\geq \frac{\left( \beta_1 \Delta_{sa} \frac{1-\gamma}{1+\gamma} \right)^2}{2\bar{\mathbb{V}}[R_{\mathcal{M}}]} ,\quad \max_{s^\prime} I_1(x(s^\prime,\pi^*_{\mathcal{M}}(s^\prime))) \geq \frac{\left( \beta_2 \Delta_{sa} \frac{1-\gamma}{\gamma(1+\gamma)} \right)^2}{2\bar{\mathbb{V}}[V^{\pi^*_{\mathcal{M}}}_{\mathcal{M}}]} ,
   $$
    and
$$
    I_2(y(s,a)) \geq \frac{(\beta_3 \Delta_{sa})^2}{2\mathbb{V}[R_{\mathcal{M}}(s,a)]} ,\quad I_1(x(s,a)) \geq \frac{\left( \frac{\beta_4 \Delta_{sa}}{\gamma} \right)^2}{2\mathbb{V}_{P_{\mathcal{M}}(s,a)}[V^{\pi^*_{\mathcal{M}}}_{\mathcal{M}}]} ,
$$
where $\bar{\mathbb{V}}[V^{\pi^*_{\mathcal{M}}}_{\mathcal{M}}]$ is defined as the maximum variance over transitions: $\max_{s^\prime\in \mathcal{S}}\mathbb{V}_{P_{{\mathcal{M}}}(s^\prime,\pi_{\mathcal{M}}^*(s^\prime))}[V^{\pi^*_{\mathcal{M}}}_{\mathcal{M}}]$ and $\bar{\mathbb{V}}[R_{\mathcal{M}}]$ is defined as the maximum variance over rewards: $\max_{s^\prime\in \mathcal{S}}\mathbb{V}[R_{\mathcal{M}}(s^\prime,\pi_{\mathcal{M}}^*(s^\prime))]$.

It is straightforward to see that
$$
    \inf_{(x,y) \in \mathcal{E}_{s,a}}\sum_{s^\prime \in \mathcal{S}, a^\prime \in \mathcal{A}}\omega_{s^\prime a^\prime}\left(I_1(x(s^\prime,a^\prime))+I_2(y(s^\prime,a^\prime))\right) = \inf_{(x,y) \in \mathcal{E}_{s,a}}\sum_{(s^\prime,a^\prime)\in\mathcal{K}}\omega_{s^\prime a^\prime}\left(I_1(x(s^\prime,a^\prime))+I_2(y(s^\prime,a^\prime))\right), 
$$
where $\mathcal{K} :=\{(s,a)\} \cup \{(s^\prime, \pi^*_{\mathcal{M}}(s^\prime)):s^\prime \in\mathcal{S}\}.$
We seek to minimize the weighted sum of rate functions under the sampling distribution $\omega$. Let $\omega_o = \min_{s^\prime\in \mathcal{S}}\omega_{s^\prime\pi^*_{\mathcal{M}}(s^\prime)}$ and $\omega_{sa}$ be the weight for the specific pair $(s,a)$. The objective function is lower-bounded by: $ \sum_{i=1}^4 \mathcal{C}_i \beta_i^2$,
where the coefficients $\mathcal{C}_i$ derived from the bounds above are:
\begin{align*}
    \mathcal{C}_1 &= \omega_o \frac{\Delta_{sa}^2 (1-\gamma)^2}{2 \bar{\mathbb{V}}[R_{\mathcal{M}}] (1+\gamma)^2}, & 
    \mathcal{C}_2 &= \omega_o \frac{\Delta_{sa}^2 (1-\gamma)^2}{2 \gamma^2 (1+\gamma)^2 \bar{\mathbb{V}}[V^{\pi^*_{\mathcal{M}}}_{\mathcal{M}}]}, \\
    \mathcal{C}_3 &= \omega_{sa} \frac{\Delta_{sa}^2}{2 \mathbb{V}[R_{\mathcal{M}}(s,a)]}, & 
    \mathcal{C}_4 &= \omega_{sa} \frac{\Delta_{sa}^2}{2 \gamma^2 \mathbb{V}_{P_{\mathcal{M}}(s,a)}[V^{\pi^*_{\mathcal{M}}}_{\mathcal{M}}]}.
\end{align*}

The optimal partition $\beta^*$ that minimizes $\sum_{i=1}^{4} \mathcal{C}_i \beta_i^2$ subject to $\sum_{i=1}^{4} \beta_i = 1,\beta_i\geq 0$ yields the objective value $(\sum_{i=1}^4 \mathcal{C}_i^{-1})^{-1}$.
Substituting the coefficients, the inverse of the optimal value is:
$$
    \sum_{i=1}^4 \mathcal{C}_i^{-1} = \frac{2}{\Delta_{sa}^2} \left[ 
        \frac{\bar{\mathbb{V}}[R_{\mathcal{M}}] (1+\gamma)^2}{\omega_o (1-\gamma)^2} + 
        \frac{\gamma^2 (1+\gamma)^2 \bar{\mathbb{V}}[V^{\pi^*_{\mathcal{M}}}_{\mathcal{M}}]}{\omega_o (1-\gamma)^2} + 
        \frac{\mathbb{V}[R_{\mathcal{M}}(s,a)]}{\omega_{sa}} + 
        \frac{\gamma^2 \mathbb{V}_{P_{\mathcal{M}}(s,a)}[V^{\pi^*_{\mathcal{M}}}_{\mathcal{M}}]}{\omega_{sa}} 
    \right].
$$
Therefore, the optimal rate is given by:
$$
    \frac{\Delta_{sa}^2}{2} \left[ 
        \frac{(1+\gamma)^2}{\omega_o (1-\gamma)^2} \left( \bar{\mathbb{V}}[R_{\mathcal{M}}] + \gamma^2 \bar{\mathbb{V}}[V^{\pi^*_{\mathcal{M}}}_{\mathcal{M}}] \right) + 
        \frac{1}{\omega_{sa}} \left( \mathbb{V}[R_{\mathcal{M}}(s,a)] + \gamma^2 \mathbb{V}_{P_{\mathcal{M}}(s,a)}[V^{\pi^*_{\mathcal{M}}}_{\mathcal{M}}] \right) 
    \right]^{-1}.
$$

\section{Proof of Lemma~\ref{lemma: convexity}}
First, recall the definition of the feasible region $\mathcal{W}$:
$$
    \mathcal{W} := \left\{\omega\in\Omega:\forall s\in \mathcal{S},\sum_{a\in\mathcal{A}}\omega_{sa} = \sum_{s^\prime \in\mathcal{S},a^\prime \in \mathcal{A}}P_{\mathcal{M}}(s|s^\prime,a^\prime)\omega_{s^\prime a^\prime}\right\},
$$
where $\Omega$ represents the probability simplex (i.e., $\omega_{sa}\ge 0$ and $\sum_{s,a}\omega_{sa}=1$). The condition $\sum_{a\in\mathcal{A}}\omega_{sa} = \sum_{s^\prime \in\mathcal{S},a^\prime \in \mathcal{A}}P_{\mathcal{M}}(s|s^\prime,a^\prime)\omega_{s^\prime a^\prime}$ constitutes a system of linear equations. Since $\mathcal{W}$ is defined entirely by the intersection of a finite number of linear equalities and inequalities, it is by definition a polytope and thus a convex set.

Next, consider the objective function in \eqref{eq: rate lower bound opt 2}. Define $\mathcal{Z} := \{(s,a): s\in \mathcal{S},a\in \mathcal{A}\setminus\{\pi^*_{\mathcal{M}}(s)\}\}$ is the set of suboptimal state-action pairs. The objective function is $F(\omega, \mathcal{M}) = \max_{(s,a)\in\mathcal{Z}} L_{sa}(\omega, \mathcal{M})$. Throughout the proof, we use the convention $c/0=+\infty$ for any $c>0$.

Since the function $g(x) = 1/x$ is convex for $x>0$. The term $1/\omega_{sa}$ is a convex function of $\omega$ (since it depends on a single coordinate). The term $1/\omega_o$ can be rewritten as: \begin{equation*}
        \frac{1}{\omega_o} = \frac{1}{\min_{s^\prime} \omega_{s^\prime \pi^*_{\mathcal{M}}(s^\prime)}} = \max_{s^\prime \in \mathcal{S}} \frac{1}{\omega_{s^\prime \pi^*_{\mathcal{M}}(s^\prime)}}.
    \end{equation*}
    Since the point-wise maximum of a family of convex functions is convex, $1/\omega_o$ is convex with respect to $\omega$.
    Since $L_{sa}(\omega, \mathcal{M})$ is a non-negative weighted sum of convex functions ($1/\omega_o$ and $1/\omega_{sa}$), $L_{sa}(\omega, \mathcal{M})$ is convex. Finally, $F(\omega, \mathcal{M})$, being the point-wise maximum of the convex functions $\{L_{sa}(\omega, \mathcal{M})\}_{(s,a)\in\mathcal{Z}}$, is itself a convex function.

We next show that the optimal solution satisfies that $\tilde{\omega}^*_{sa}(\mathcal{M})>0$ for any $s\in\mathcal{S}$ and $a\in\mathcal{A}$. 
Define the set of strictly positive valid distributions as $\mathcal{W}_{++} = \mathcal{W} \cap \{\omega : \omega_{sa} > 0, \forall s,a\}$.
Since the MDP is assumed to be ergodic, there exists at least one stationary distribution with full support (e.g., induced by a policy that selects all actions with positive probability). Thus, $\mathcal{W}_{++}$ is non-empty, and the problem has a finite feasible value.

Observe that if $\omega_{sa} \to 0$ for any $(s,a) \in \mathcal{Z}$, the local term in $L_{sa}(\omega, \mathcal{M})$ tends to infinity. Similarly, if $\omega_{s' \pi^*(s')} \to 0$ for any $s'$, then $\omega_o \to 0$, causing the global term in all $L_{sa}(\omega, \mathcal{M})$ to tend to infinity. In either case, $F(\omega) \to \infty$.
Since a finite objective value is achievable in $\mathcal{W}_{++}$ and $\mathcal{W}$ is compact, an optimal solution exists. Moreover, no minimizer can have $\omega_{sa}=0$ for any $(s,a)$; hence every optimal solution lies in $\mathcal{W}_{++}$, i.e.,
$\tilde{\omega}^*_{sa}(\mathcal{M})>0$ for all $(s,a)\in\mathcal{S}\times\mathcal{A}$.

Finally, we show that the set of optimal solutions $\mathcal{C}^*(\mathcal{M})$ is convex. Let $\omega^1,\omega^2\in\mathcal{C}^*(\mathcal{M})$ and $\lambda\in[0,1]$.
Since $\mathcal{W}$ is convex, $\omega^\lambda := \lambda\omega^1+(1-\lambda)\omega^2 \in \mathcal{W}$.
By convexity of $F(\cdot,\mathcal{M})$,
$$
    F(\omega^\lambda,\mathcal{M})
    \le \lambda F(\omega^1,\mathcal{M}) + (1-\lambda)F(\omega^2,\mathcal{M})
    = \lambda F^* + (1-\lambda)F^* = F^*,
$$
where $F^*:=\min_{\omega\in\mathcal{W}}F(\omega,\mathcal{M})$.
Since $F^*$ is the minimum value, we must have $F(\omega^\lambda,\mathcal{M})=F^*$, implying $\omega^\lambda\in\mathcal{C}^*(\mathcal{M})$.
Therefore, $\mathcal{C}^*(\mathcal{M})$ is a convex set.

\section{Proof of Lemma~\ref{lemma: consistency}}
We prove that every state-action pair is sampled infinitely often almost surely. 
Fix any $s\in\mathcal S$. Define the event
\begin{equation*}
    \mathcal E_s(t):=\{N(s;t)\ge C_1 t\},
\qquad N(s;t)=\sum_{a\in\mathcal A}N(s,a;t).
\end{equation*}
By Proposition 2 of \citealp{burnetas1997optimal}, one obtains that there exist constants $C_1,C_2,C_3>0$ such that for all sufficiently large $t$,
\begin{equation*}
\mathbb P_{\mathcal M}\big(\mathcal E_s(t)^c\big)
=\mathbb P_{\mathcal M}\big(N(s;t)<C_1 t\big)
\le C_2 \exp(-C_3 t).
\end{equation*}

By definition of the behavior policy, at any time step $k$, the probability of sampling action $a\in \mathcal{A}$ is lower bounded by $\epsilon_k / A = 1 / (A k^\alpha)$ for some $\alpha \in (0, 1/2)$. We introduce the filtration $\mathcal G_k:=\sigma(s_1,a_1,R_1,\ldots,s_k)$ (the information available right after observing $s_k$ and before choosing $a_k$), and define the indicator variable $X_k = \mathbb{I}\{s_k = s, a_k = a\}$. Then
$$
    p_k:=\mathbb E[X_k\mid \mathcal G_k]
=\mathbb P(a_k=a\mid \mathcal G_k)\,\mathbb I(s_k=s)
\ \ge\ \frac{1}{A k^\alpha}\,\mathbb I(s_k=s).
$$
Summing over $k\le t$ yields
\begin{equation*}
\sum_{k=1}^t p_k
\ \ge\ \frac{1}{A}\sum_{k=1}^t k^{-\alpha}\mathbb I(s_k=s).
\end{equation*}
On the event $\mathcal E_s(t)$ we have $\sum_{k=1}^t \mathbb I(s_k=s)=N(s;t)\ge C_1 t$. Since $k^{-\alpha}$ is decreasing in $k$, the weighted sum is minimized by placing these
$C_1 t$ indicators on the largest time indices, hence
$$
    \sum_{k=1}^t k^{-\alpha}\mathbb I(s_k=s)
\ \ge\ \sum_{k=t-\lfloor C_1 t\rfloor+1}^t k^{-\alpha}
\ \ge\ \int_{(1-C_1)t}^t x^{-\alpha}\,dx
= \frac{1-(1-C_1)^{1-\alpha}}{1-\alpha}\, t^{1-\alpha}.
$$
Therefore, on $\mathcal E_s(t)$,
\begin{equation*}
\sum_{k=1}^t p_k \ \ge\ C_4 t^{1-\alpha},
\qquad
C_4:=\frac{1-(1-C_1)^{1-\alpha}}{A(1-\alpha)}.
\end{equation*}

Define the martingale difference sequence $D_k:=X_k-p_k$, and the martingale 
\begin{equation*}
    M_t :=\sum_{k=1}^{t}D_{k} =  N(s,a;t) - \sum_{k=1}^{t}p_k
\end{equation*}
with respect to the filtration $\mathcal F_k:=\sigma(\mathcal G_k,a_k,R_k)$. We have $|D_k|\le 1$, hence $|M_t-M_{t-1}|\le 1$. 

Now observe that on the event $\mathcal E_s(t)$, if
$$
N(s,a;t)\le \frac{C_4}{2}t^{1-\alpha},
$$
then
$$
M_t
=
N(s,a;t)-\sum_{k=1}^t p_k
\le
\frac{C_4}{2}t^{1-\alpha}-C_4t^{1-\alpha}
=
-\frac{C_4}{2}t^{1-\alpha}.
$$
Hence,
\begin{equation*}
\left\{N(s,a;t)\le \frac{C_4}{2}t^{1-\alpha}\right\}\cap \mathcal E_s(t)
\subseteq
\left\{M_t\le -\frac{C_4}{2}t^{1-\alpha}\right\}.
\end{equation*}
Therefore,
$$
\mathbb P_{\mathcal M}\left(N(s,a;t)\le \frac{C_4}{2}t^{1-\alpha}\right)
\le
\mathbb P_{\mathcal M}\left(M_t\le -\frac{C_4}{2}t^{1-\alpha}\right)
+
\mathbb P_{\mathcal M}\big(\mathcal E_s(t)^c\big).
$$
Applying Azuma-Hoeffding inequality to the martingale $M_t$ gives
$$
\mathbb P_{\mathcal M}\left(M_t \leq -\frac{C_4}{2}t^{1-\alpha}\right)
\leq \exp\left(-\frac{(C_4/2)^2 t^{2-2\alpha}}{2t}\right)
= \exp\left(-\frac{C_4^2t^{1-2\alpha}}{8}\right).
$$
Since $\alpha<1/2$, the exponent $t^{1-2\alpha}$ diverges.

Combining with $\mathbb P(\mathcal E_s(t)^c)\le C_2 e^{-C_3 t}$ yields, for each fixed $(s,a)$, 
\begin{equation}
\label{eq: error up bound}
    \mathbb P_{\mathcal M}\!\left(N(s,a;t)\le \frac{C_4}{2}t^{1-\alpha}\right)
\le
\exp\!\left(-\frac{C_4^2}{8}\,t^{1-2\alpha}\right)
+ C_2 \exp(-C_3 t).
\end{equation}
Applying a union bound over all $(s,a)\in\mathcal S\times\mathcal A$ gives
\begin{equation*}
    \mathbb P_{\mathcal M}(V_t^c)
\le
SA\exp\!\left(-\frac{C_4^2}{8}\,t^{1-2\alpha}\right)
+ SA\,C_2 \exp(-C_3 t),
\end{equation*}
where 
\begin{equation*}
    V_t = \left\{\forall s\in\mathcal{S},a\in\mathcal{A}: N(s,a;t)>\frac{C_4}{2}t^{1-\alpha}\right\}.
\end{equation*}
The RHS is summable over $t$ because $\alpha<1/2$.
Thus, by the Borel-Cantelli lemma,
\begin{equation}
\label{eq: sample_lb}
    \mathbb{P}_{\mathcal{M}}(\limsup_{t \rightarrow \infty} V^c_t) = 0,
\end{equation}
which implies that almost surely, $V_t$ holds for all sufficiently large $t$.
Consequently, for every $(s,a)$,
$$
\lim_{t\to\infty} N(s,a;t)=\infty \qquad \text{a.s.}
$$
This completes the proof.

\section{Proof of Theorem~\ref{thm: sub-grad converge}}
Recall that $\tilde{\omega}^*(\mathcal M)\in\arg\min_{\omega\in\mathcal W}F(\omega,\mathcal M)$ is an arbitrary optimal solution. Let $z_n := \Pi_{\mathcal W^{\epsilon}(\bar{\mathcal{M}}(t_n))}(\tilde\omega^*(\mathcal M))$ denote the projection of the true optimal solution onto the current estimated feasible set. 

\begin{lemma}
\label{lemma: set stability}
Under the condition $\bar{\mathcal M}(t_n)\to\mathcal M$ a.s., there exist an integer $n_1>0$ and a constant $L_{\mathcal M}>0$ such that for all $n\ge n_1$,
\begin{equation*}
    \lVert z_n - \tilde{\omega}^*(\mathcal{M})\rVert_2 \leq L_{\mathcal{M}} \lVert \bar{\mathcal{M}}(t_n)-\mathcal{M}\rVert.
\end{equation*}
\end{lemma}
\proof{Proof of Lemma~\ref{lemma: set stability}
}
Since the feasible set $\mathcal{W}$ is defined by linear equalities (the Bellman flow constraints) and inequalities (non-negativity), it is a Polyhedron and can be represented as: 
$$
   \Omega:=\{\omega\in\mathbb R^{SA}: \omega\ge 0,\ I^\top \omega=1\},
\qquad \mathcal{W} := \left\{\omega \in \Omega: A(\mathcal{M})\omega = 0\right\},
$$
where $I$ is the vector of all ones, $A(\mathcal{M})$ is a matrix with $S$ rows (one for each state $s\in\mathcal{S}$) and $SA$ columns (one for each state-action pair $(s,a)\in\mathcal{S}\times \mathcal{A}$). The entries of the matrix $A(\mathcal{M})$ at row $s$ and column $(s^\prime,a^\prime)$ are defined as:
\begin{equation*}
    A(\mathcal{M})_{s,(s^\prime,a^\prime)} = \begin{cases}
        1- P_{\mathcal{M}}(s|s^\prime,a^\prime),\quad& \text{if}\quad s^\prime =s\\
        - P_{\mathcal{M}}(s|s^\prime,a^\prime),\quad& \text{if}\quad s^\prime \neq s.
    \end{cases}
\end{equation*}
In matrix notation, $A(\mathcal{M})$ can be written as:
\begin{equation}
\label{eq: constraint matrix}
    A(\mathcal{M}) = \Phi - P_\mathcal{M}^\top,
\end{equation}
where $\Phi$ is a matrix with the element 
$
    \Phi_{s,(s^\prime,a^\prime)} = \mathbb{I}(s=s^\prime).
$

Since $\tilde{\omega}^*(\mathcal{M})$ is the optimal solution for the true model $\mathcal{M}$, by definition, $\tilde{\omega}^*(\mathcal{M})$ satisfies the flow constraints for $\mathcal{M}$:
\begin{equation*}
A(\mathcal{M})\tilde{\omega}^*(\mathcal{M}) = 0,\quad {I}^\top\tilde{\omega}^*(\mathcal{M}) = 1, \quad \tilde{\omega}^*(\mathcal{M})\ge 0.
\end{equation*}
Now, consider the estimated model $\bar{\mathcal{M}}(t_n)$. Note that the solution $\tilde{\omega}^*(\mathcal{M})$ satisfies the normalization and non-negativity constraints but may violate the flow constraints of the estimated model. We quantify this violation as:
$$
    r_n = A(\bar{\mathcal{M}}(t_n))\tilde{\omega}^*(\mathcal{M}) =  A(\bar{\mathcal{M}}(t_n))\tilde{\omega}^*(\mathcal{M}) - A(\mathcal{M})\tilde{\omega}^*(\mathcal{M}) = (A(\bar{\mathcal{M}}(t_n)) - A(\mathcal{M}))  \tilde{\omega}^*(\mathcal{M}).
$$
Using the linearity of the matrix construction, the difference in matrices is exactly the difference in transition probabilities (transposed). Moreover, since $\tilde\omega^*(\mathcal M)\in\Omega$, $\|\tilde\omega^*(\mathcal M)\|_2\le \|\tilde\omega^*(\mathcal M)\|_1=1$. Therefore,
\begin{equation*}
    \lVert r_n\rVert_2 \le \lVert A(\mathcal{M}) - A(\bar{\mathcal{M}}(t_n)) \rVert  \le c_0\lVert \mathcal{M} - \bar{\mathcal{M}}(t_n) \rVert,
\end{equation*}
where $c_0>0$ depends only on the norm used to define $\|\bar{\mathcal M}(t_n)-\mathcal M\|$.

Since $\mathcal W(\bar{\mathcal{M}}(t_n))$ is a polyhedron defined by linear equalities and inequalities,
Hoffman's error bound~\cite{hoffman2003approximate} implies that there exists a constant $C_n>0$ such that for any $x\in\Omega$,
\begin{equation*}
\mathrm{dist}(x,\mathcal W(\bar{\mathcal{M}}(t_n)))\le C_n \|A(\bar{\mathcal M}(t_n))x\|_2.
\end{equation*}
Applying this to $x=\tilde\omega^*(\mathcal M)$ and define
$\hat z_n=\Pi_{\mathcal W(\bar{\mathcal{M}}(t_n))}(\tilde\omega^*(\mathcal M))$ satisfies
$\|\hat z_n-\tilde\omega^*(\mathcal M)\|_2=\mathrm{dist}(\tilde\omega^*(\mathcal M),\mathcal W(\bar{\mathcal{M}}(t_n)))$, we obtain
\begin{equation*}
\|\hat z_n-\tilde\omega^*(\mathcal M)\|_2 \le C_n \|r_n\|_2 \le C_n c_0 \|\bar{\mathcal M}(t_n)-\mathcal M\|.
\end{equation*}
Finally, because $\bar{\mathcal M}(t_n)\to\mathcal M$, we may restrict attention to $n$ large enough so that $\bar{\mathcal M}(t_n)$ stays in a (non-degenerate) neighborhood of $\mathcal M$ on which the Hoffman constants of $\mathcal W(\bar{\mathcal M}(t_n))$ are uniformly bounded (Theorem 5.6 of \citealp{luo1994perturbation}); hence $C_n\le C_{\mathcal M}$ for all $n\ge n_0$.
Therefore, for $n\ge n_0$,
\begin{equation*}
\|\hat z_n-\tilde\omega^*(\mathcal M)\|_2 \le C_{\mathcal M}c_0\, \|\bar{\mathcal M}(t_n)-\mathcal M\|.
\end{equation*}
By Lemma ~\ref{lemma: convexity}, $\tilde{\omega}^*(\mathcal M)$ is strictly positive coordinatewise , we define $
\epsilon_0(\mathcal{M}):=\min_{(s,a)\in\mathcal S\times\mathcal A}\tilde{\omega}^*_{sa}(\mathcal M)>0.
$
Choose $\epsilon\in(0,\epsilon_0)$, and let
$
\delta:=\epsilon_0-\epsilon>0.
$ By the above estimate, there exists $n_1\ge n_0$ such that for all $n\ge n_1$,
$$
\|\hat z_n-\tilde{\omega}^*(\mathcal M)\|_2 \le \delta.
$$
Hence, for every $(s,a)\in\mathcal S\times\mathcal A$ and every $n\ge n_1$,
$$
\hat z_{n,sa}
\ge
\tilde{\omega}^*_{sa}(\mathcal M)-\|\hat z_n-\tilde{\omega}^*(\mathcal M)\|_2
\ge
\epsilon_0-\delta
=
\epsilon.
$$
Therefore, $\hat z_n\in\mathcal W^\epsilon(\bar{\mathcal M}(t_n))$ for all $n\ge n_1$.
Since
$
\mathcal W^\epsilon(\bar{\mathcal M}(t_n))\subseteq \mathcal W(\bar{\mathcal M}(t_n)),
$
we have
$$
\mathrm{dist}\!\left(\tilde{\omega}^*(\mathcal M),
\mathcal W(\bar{\mathcal M}(t_n))\right)
\le
\mathrm{dist}\!\left(\tilde{\omega}^*(\mathcal M),
\mathcal W^\epsilon(\bar{\mathcal M}(t_n))\right).
$$
On the other hand, since $\hat z_n\in\mathcal W^\epsilon(\bar{\mathcal M}(t_n))$, we also have
$$
\mathrm{dist}\!\left(\tilde{\omega}^*(\mathcal M),
\mathcal W^\epsilon(\bar{\mathcal M}(t_n))\right)
\le
\|\hat z_n-\tilde{\omega}^*(\mathcal M)\|_2.
$$
Moreover, because $\hat z_n$ is the Euclidean projection of $\tilde{\omega}^*(\mathcal M)$ onto
$\mathcal W(\bar{\mathcal M}(t_n))$,
$$
\|\hat z_n-\tilde{\omega}^*(\mathcal M)\|_2
=
\mathrm{dist}\!\left(\tilde{\omega}^*(\mathcal M),
\mathcal W(\bar{\mathcal M}(t_n))\right).
$$
Combining the above inequalities gives
$$
\mathrm{dist}\!\left(\tilde{\omega}^*(\mathcal M),
\mathcal W^\epsilon(\bar{\mathcal M}(t_n))\right)
=
\mathrm{dist}\!\left(\tilde{\omega}^*(\mathcal M),
\mathcal W(\bar{\mathcal M}(t_n))\right).
$$
Since $\mathcal W^\epsilon(\bar{\mathcal M}(t_n))$ is closed and convex, the Euclidean projection is unique. Hence, for all $n\ge n_1$,
$$
z_n
=
\Pi_{\mathcal W^\epsilon(\bar{\mathcal M}(t_n))}(\tilde{\omega}^*(\mathcal M))
=
\hat z_n.
$$
Consequently, for all $n\ge n_1$,
$$
\|z_n-\tilde{\omega}^*(\mathcal M)\|_2
=
\|\hat z_n-\tilde{\omega}^*(\mathcal M)\|_2
\le
C_{\mathcal M}c_0\,\|\bar{\mathcal M}(t_n)-\mathcal M\|.
$$
Let $L_{\mathcal M}:=C_{\mathcal M}c_0$. This concludes the proof.

\begin{lemma}[One-step projected subgradient inequality]
\label{lemma: subgradient_inequal}
    For each update index $n\ge1$, 
$$
\|x_n-z_n\|_2^2
\le 
\|x_{n-1}-z_n\|_2^2 + \eta_n^2\|\bar g_n\|_2^2
+2\eta_n\big(F(z_n, \bar{\mathcal{M}}(t_n))-F(x_{n-1},\bar{\mathcal{M}}(t_n))\big),
$$
where $\bar g_n\in\partial_\omega F(x_{n-1},\bar{\mathcal M}(t_n))$ and $x_n=\Pi_{\mathcal W^{\epsilon}(\bar{\mathcal{M}}(t_n))}(x_{n-1}-\eta_n\bar g_n)$.
\end{lemma}
\proof{Proof of Lemma \ref{lemma: subgradient_inequal}}
By the non-expansiveness of Euclidean projection onto a closed convex set, and expanding the square, we obtain
\begin{equation}
\label{eq: solution dist}
\begin{aligned}
    \lVert x_n - z_n\rVert_2^2 &= \lVert \Pi_{\mathcal W^{\epsilon}(\bar{\mathcal{M}}(t_n))}(x_{n-1}-\eta_n\bar{g}_n) - z_n\rVert_2^2\\
    &\le \lVert x_{n-1}-\eta_n\bar{g}_n - z_n  \rVert_2^2\\
    &= \lVert x_{n-1} - z_n \rVert_2^2 + \eta_n^2 \lVert \bar{g}_n\rVert_2^2 - 2\eta_n \bar{g}_n^\top (x_{n-1} - z_n).
\end{aligned}
\end{equation}
Since $\bar{g}_n$ is the subgradient of the function $F(\cdot,\bar{\mathcal{M}}(t_n))$ at $x_{n-1}$, by the definition of the subgradient, we have that
\begin{equation*}
    F(z_n, \bar{\mathcal{M}}(t_n)) \geq  F(x_{n-1}, \bar{\mathcal{M}}(t_n)) + \bar{g}_n^\top (z_n - x_{n-1}).
\end{equation*}
It follows that
\begin{equation*}
    -\bar{g}_n^\top (x_{n-1}-z_n) \le F(z_n, \bar{\mathcal{M}}(t_n)) - F(x_{n-1}, \bar{\mathcal{M}}(t_n)).
\end{equation*}
Substitute this back into the equation \eqref{eq: solution dist}:
$$
    \|x_n-z_n\|_2^2
\le 
\|x_{n-1}-z_n\|_2^2 + \eta_n^2\|\bar g_n\|_2^2
+2\eta_n\big(F(z_n, \bar{\mathcal{M}}(t_n))-F(x_{n-1},\bar{\mathcal{M}}(t_n))\big).
$$

\begin{lemma}[Uniform bound on subgradients]
\label{lemma: subgradients bound}
Under the condition $\bar{\mathcal M}(t_n)\to\mathcal M$ a.s., there exists an almost surely finite random constant $G$ such that $\lVert\bar{g}_n\rVert_2\le G$ for all $n\ge 1$.
\end{lemma}
\proof{Proof of \ref{lemma: subgradients bound}}

According to Lemma \ref{lemma: convexity} the optimal solution $\tilde{\omega}^*(\mathcal{M})$ to the problem \eqref{eq: rate lower bound opt 2} satisfies $\tilde{\omega}_{sa}^*(\mathcal{M})>0$ for all $(s,a)\in\mathcal{S}\times\mathcal{A}$. Hence, there exists a sufficiently small constant $\epsilon_0>0$ such that $\tilde{\omega}_{sa}^*(\mathcal{M})\ge \epsilon_0$ for all $(s,a)\in\mathcal{S}\times\mathcal{A}$.

For any $0<\epsilon\le \epsilon_0$, define
\begin{equation*}
    \mathcal{W}^{\epsilon}(\mathcal M^\prime) := \left\{\omega\in\mathcal{W}(\mathcal M^\prime):\quad
    \omega_{sa}\geq \epsilon,\forall (s,a)\in\mathcal{S}\times\mathcal{A}\right\},
\end{equation*} Since $\tilde{\omega}^*(\mathcal{M})\in \mathcal{W}^{\epsilon_0}(\mathcal{M})$, enforcing this lower bound constraint does not induce any optimality gap for the true model $\mathcal M$. 
Fix $\epsilon\in(0,1]$ and consider any vector $\omega\in\mathbb R^{SA}$ such that $\omega_{sa}\ge\epsilon$ for all $(s,a)\in \mathcal{S}\times \mathcal{A}$. Then, $\omega_o\ge\epsilon$.

Consider the convex function $\varphi(\omega):=1/\omega_o$ on the domain $\{\omega:\omega_o>0\}$. Since
\begin{equation*}
\frac{1}{\omega_o}
=\max_{s\in\mathcal S}\frac{1}{\omega_{s\pi^*_{\mathcal M}(s)}},
\end{equation*} $\varphi$ is the point-wise maximum of finitely many differentiable convex functions. Therefore,
\begin{equation*}
\partial \varphi(\omega)
=
\mathrm{conv}\left\{-\frac{1}{\omega_{s\pi^*_{\mathcal M}(s)}^2}\,e_{s\pi^*_{\mathcal M}(s)}:\ 
s\in\arg\min_{u\in\mathcal S}\omega_{u\pi^*_{\mathcal M}(u)}\right\},
\end{equation*}
where $e_{sa}$ denotes the standard basis vector. Because $\omega_{s\pi^*_{\mathcal M}(s)}\ge\epsilon$, any $v\in\partial\varphi(\omega)$ satisfies
$\|v\|_2\le 1/\epsilon^2$. Similarly, the function $\omega\mapsto 1/\omega_{sa}$ is differentiable with gradient
$-(1/\omega_{sa}^2)e_{sa}$, whose norm is also bounded by $1/\epsilon^2$ whenever $\omega_{sa}\ge \epsilon$. Thus, for the true model $\mathcal M$, there exists a finite constant $C(\mathcal M)<\infty$ such that
$$
\sup_{h\in\partial_\omega L_{sa}(\omega,\mathcal M)}\|h\|_2
\le \frac{C(\mathcal M)}{\epsilon^2},
\qquad \forall s\in\mathcal{S}, a\in\mathcal{A}\setminus \{\pi^*_{\mathcal{M}}(s)\}.
$$
for every $\omega$ satisfying $\omega_{sa}\ge \epsilon$ for all $(s,a)$.

Next, since the true MDP $\mathcal M$ has a unique optimal policy, all suboptimality gaps are strictly positive:
$
\Delta_{sa}(\mathcal M)>0,
\forall a\neq \pi^*_{\mathcal M}(s).
$ By continuity of the value and $Q$-functions with respect to the model parameters, there exists a sufficiently small compact neighborhood $\mathcal U$ of $\mathcal M$ such that, for every $\mathcal M'\in\mathcal U$,
$$
\pi^*_{\mathcal M'}=\pi^*_{\mathcal M},
\qquad
\inf_{\mathcal M'\in\mathcal U}\Delta_{sa}(\mathcal M')>0,
\qquad
\forall a\neq \pi^*_{\mathcal M}(s).
$$
Moreover, the variance terms and all coefficients appearing in $L_{sa}(\cdot,\mathcal M')$ depend continuously on $\mathcal M'$. Hence, since $\mathcal U$ is compact, there exists a finite constant
$
\bar C_\epsilon:=\sup_{\mathcal M'\in\mathcal U} C(\mathcal M') < \infty
$
such that for every $\mathcal M'\in\mathcal U$ and every $\omega$ satisfying $\omega_{sa}\ge\epsilon$ for all $(s,a)$,
$$
\sup_{h\in\partial_\omega L_{sa}(\omega,\mathcal M')}\|h\|_2
\le
\frac{\bar C_\epsilon}{\epsilon^2},
\qquad
\forall s\in\mathcal S,\ a\in\mathcal A\setminus\{\pi^*_{\mathcal M}(s)\}.
$$

Finally, since 
$$
F(\omega,\mathcal M')
=
\max_{s\in\mathcal S,\ a\in\mathcal A\setminus\{\pi^*_{\mathcal M}(s)\}}
L_{sa}(\omega,\mathcal M')
$$
is the pointwise maximum of finitely many convex functions, we have
$$
\partial_\omega F(\omega,\mathcal M^\prime)
=
\mathrm{conv}\Big\{\partial_\omega L_{sa}(\omega,\mathcal M^\prime):\ 
(s,a)\in\arg\max_{u\in\mathcal{S},b\in\mathcal{A}\setminus\{\pi_{\mathcal M}^*(u)\}}L_{ub}(\omega,\mathcal M^\prime)\Big\}.
$$
Therefore, for every $\mathcal M'\in\mathcal U$ and every $\omega$ with all coordinates at least $\epsilon$,
$$
\sup_{g\in\partial_\omega F(\omega,\mathcal M')}\|g\|_2
\le
\frac{\bar C_\epsilon}{\epsilon^2}.
$$

Since $\bar{\mathcal M}(t_n)\to\mathcal M$ almost surely, on the event of convergence there exists $n_0$ such that $\bar{\mathcal M}(t_n)\in\mathcal U$ for all $n\ge n_0$. Moreover, by construction of the algorithm, $x_{n-1}\in \mathcal W^\epsilon(\bar{\mathcal M}(t_{n-1})),$ and hence every coordinate of $x_{n-1}$ is at least $\epsilon$. Therefore, for all $n\ge n_0+1$,
$$
\|\bar g_n\|_2
\le
\frac{\bar C_\epsilon}{\epsilon^2}.
$$
For the finitely many indices $n\le n_0$, each $\|\bar g_n\|_2$ is finite. Hence,
$$
G
:=
\max\left\{
\max_{1\le n\le n_0}\|\bar g_n\|_2,\ 
\frac{\bar C_\epsilon}{\epsilon^2}
\right\}
<\infty
\qquad \text{a.s.}
$$
Thus,
$
\|\bar g_n\|_2\le G, \forall n\ge 1.
$

\begin{lemma}[Model perturbation bound]
\label{lemma: model_perturbation} 
Under the condition $\bar{\mathcal M}(t_n)\to\mathcal M$ a.s., define
$$
\Omega^\epsilon
:=
\left\{
\omega\in\Omega:\ \omega_{sa}\ge \epsilon,\ \forall (s,a)\in\mathcal S\times\mathcal A
\right\},
E_n
:=
\sup_{\omega\in\Omega^\epsilon}
\big|F(\omega,\bar{\mathcal M}(t_n))-F(\omega,\mathcal M)\big|.
$$
Then $E_n\to 0$ almost surely as $n\to\infty$.
Moreover, for any $\omega\in\Omega^\epsilon$,
$$
F(\omega,\bar{\mathcal M}(t_n))\le F(\omega,\mathcal M)+E_n,
\qquad
F(\omega,\bar{\mathcal M}(t_n))\ge F(\omega,\mathcal M)-E_n.
$$
\end{lemma}
\proof{Proof of Lemma \ref{lemma: model_perturbation}}
Since $\bar{\mathcal M}(t_n)\to\mathcal M$ almost surely and the optimal policy $\pi^*_{\mathcal{M}}$ is unique, the optimality gap at each state is strictly positive. By continuity of the optimality comparisons in the model parameters, there exists an open neighborhood $\mathcal{U}$ of $\mathcal M$ such that $\pi^*_{\mathcal{M}^\prime} = \pi^*_{\mathcal{M}}$ for all $\mathcal{M}^\prime \in \mathcal{U}$. Hence, on the event of convergence, $\bar{\mathcal M}(t_n)\in\mathcal U$ for all sufficiently large $n$.

Since $\Omega$ is compact and $\Omega^\epsilon$ is closed, $\Omega^\epsilon$ is compact. Moreover, for any $\omega\in\Omega^\epsilon$, we have $\omega_{sa}\ge \epsilon$ for all $(s,a)$, and hence $\omega_o\ge \epsilon$. Therefore, all denominators appearing in the definition of $F(\omega,\cdot)$ are uniformly bounded away from zero on $\Omega^\epsilon$.

Fix the neighborhood $\mathcal U$ above. By possibly replacing $\mathcal U$ with the intersection of $\mathcal U$, a closed ball around $\mathcal M$, and the compact model parameter space, we may assume that $\mathcal U$ is a compact neighborhood of $\mathcal M$ and that $\pi^*_{\mathcal M'}=\pi^*_{\mathcal M}$ for all $\mathcal M'\in\mathcal U$. For each component $L_{sa}(\omega,\mathcal M')$, its dependence on $\mathcal M'$ enters only through finitely many bounded model parameters, and it is jointly continuous in $(\omega,\mathcal M')$ on
$\Omega^\epsilon\times \mathcal U$.
Because $F$ is the pointwise maximum of finitely many such components, with the same index set for all $\mathcal M'\in\mathcal U$ since $\pi^*_{\mathcal M'}=\pi^*_{\mathcal M}$, it follows that $F(\omega,\mathcal M')$ is also jointly continuous on $\Omega^\epsilon\times\mathcal U$.

Since $\Omega^\epsilon\times\mathcal U$ is compact, after possibly shrinking $\mathcal U$, the Heine-Cantor theorem implies that $F$ is uniformly continuous on this set. Consequently,
$$
\|\mathcal M'-\mathcal M\|\to 0
\quad\Longrightarrow\quad
\sup_{\omega\in\Omega^\epsilon}|F(\omega,\mathcal M')-F(\omega,\mathcal M)|\to 0.
$$
Applying this with $\mathcal M'=\bar{\mathcal M}(t_n)$ and using $\bar{\mathcal M}(t_n)\to\mathcal M$ almost surely yields
$E_n\to 0$ almost surely.

For any fixed $\omega\in\Omega^\epsilon$,
$$
|F(\omega,\bar{\mathcal M}(t_n))-F(\omega,\mathcal M)|\le E_n,
$$
which is equivalent to
$
F(\omega,\mathcal M)-E_n\le F(\omega,\bar{\mathcal M}(t_n))\le F(\omega,\mathcal M)+E_n.
$

\begin{lemma}[Telescoping]
\label{lem:telescope}
Let $N\ge 1$ and define
$$
S_N:=\sum_{n=1}^N\Big(\|x_{n-1}-z_n\|_2^2-\|x_n-z_n\|_2^2\Big).
$$
Assume $x_n,z_n\in\Omega$ for all $n$, where $\Omega$ is the probability simplex, and let  $D:=\sup_{u,v\in\Omega}\|u-v\|_2<\infty$. Then
\begin{equation}
\label{eq:telescope_bound}
S_N
\le 
\|x_0-z_1\|_2^2
+\sum_{n=1}^{N-1} 2D\,\|z_{n+1}-z_n\|_2 .
\end{equation}
\end{lemma}
\proof{Proof of Lemma \ref{lem:telescope}}
Start from the definition and shift the index in the first sum:
\begin{align*}
S_N
&=\sum_{n=1}^N\|x_{n-1}-z_n\|_2^2
-\sum_{n=1}^N\|x_n-z_n\|_2^2 \\
&=\sum_{n=0}^{N-1}\|x_n-z_{n+1}\|_2^2
-\sum_{n=1}^N\|x_n-z_n\|_2^2 \\
&=\|x_0-z_1\|_2^2-\|x_N-z_N\|_2^2 \\
&\quad
+\sum_{n=1}^{N-1}
\Big(\|x_n-z_{n+1}\|_2^2-\|x_n-z_n\|_2^2\Big).
\end{align*}
Dropping the non-positive term $-\|x_N-z_N\|_2^2\le 0$ gives
\begin{equation}
\label{eq:telescope_step1}
S_N
\le
\|x_0-z_1\|_2^2
+\sum_{n=1}^{N-1}\Big(\|x_n-z_{n+1}\|_2^2-\|x_n-z_n\|_2^2\Big).
\end{equation}

Fix any $n\in\{1,\dots,N-1\}$. Using $|a^2-b^2|=|a+b||a-b|$ with
$a=\|x_n-z_{n+1}\|_2$ and $b=\|x_n-z_n\|_2$, we have
$$
\big|\|x_n-z_{n+1}\|_2^2-\|x_n-z_n\|_2^2\big|
=
\big|\|x_n-z_{n+1}\|_2+\|x_n-z_n\|_2\big|\,
\big|\|x_n-z_{n+1}\|_2-\|x_n-z_n\|_2\big|.
$$
By the reverse triangle inequality,
$$
\big|\|x_n-z_{n+1}\|_2-\|x_n-z_n\|_2\big|
\le \|z_{n+1}-z_n\|_2.
$$
Moreover, since $x_n,z_n,z_{n+1}\in\Omega$ and $\mathrm{diam}(\Omega)=D$, we have
$\|x_n-z_{n+1}\|_2\le D$ and $\|x_n-z_n\|_2\le D$, hence
$\|x_n-z_{n+1}\|_2+\|x_n-z_n\|_2\le 2D$.
Therefore,
$$
\|x_n-z_{n+1}\|_2^2-\|x_n-z_n\|_2^2
\le
\big|\|x_n-z_{n+1}\|_2^2-\|x_n-z_n\|_2^2\big|
\le 2D\,\|z_{n+1}-z_n\|_2.
$$
Summing this bound over $n=1,\dots,N-1$ and substituting into \eqref{eq:telescope_step1}
yields \eqref{eq:telescope_bound}.

We combine the above lemmas to prove Theorem \ref{thm: sub-grad converge}. Recall the notation $\mathcal{C}^*(\mathcal{M})$, $z_n:=\Pi_{\mathcal W^{\epsilon}(\bar{\mathcal M}(t_n))}(\tilde\omega^*(\mathcal M))$, and define $F^*:=\min_{\omega\in \mathcal{W}}F(\omega,\mathcal{M})$. By Lemma~\ref{lemma: convexity}, $\mathcal C^*(\mathcal M)$ is nonempty and convex. Fix an arbitrary representative $\tilde\omega^*(\mathcal M)\in\mathcal C^*(\mathcal M)$ (the choice is immaterial since
$F(\tilde\omega^*(\mathcal M),\mathcal M)=F^*$). 

Starting from Lemma~\ref{lemma: subgradient_inequal},
$$
\|x_n-z_n\|_2^2
\le 
\|x_{n-1}-z_n\|_2^2 + \eta_n^2\|\bar g_n\|_2^2
+2\eta_n\big(F(z_n, \bar{\mathcal{M}}(t_n))-F(x_{n-1},\bar{\mathcal{M}}(t_n))\big),
$$
and applying Lemma~\ref{lemma: model_perturbation} at $\omega=z_n$ and $\omega=x_{n-1}$, we obtain
\begin{equation}
\label{eq:gap_recursion_pre_new}
\|x_n-z_n\|_2^2
\le 
\|x_{n-1}-z_n\|_2^2 + \eta_n^2\|\bar g_n\|_2^2
+2\eta_n\big(F(z_n,\mathcal{M})-F(x_{n-1},\mathcal{M})\big)
+4\eta_n E_n .
\end{equation}

Next, the proof of Lemma~\ref{lemma: subgradients bound} establishes that
$F(\cdot,\mathcal M)$ has uniformly bounded subgradients on $\Omega^\epsilon$. Hence,
$F(\cdot,\mathcal M)$ is Lipschitz on this compact set with some finite Lipschitz constant
$G_\epsilon(\mathcal M)$. Therefore, for all $\omega\in\Omega^\epsilon$,
\begin{equation}
\label{eq:lipschitz_to_optset}
F(\omega,\mathcal M)-F^*
\le
G_\epsilon(\mathcal M)\,\mathrm{dist}(\omega,\mathcal C^*(\mathcal M)),
\end{equation}
where $\mathrm{dist}(\omega,\mathcal C^*(\mathcal M)):=\inf_{u\in\mathcal C^*(\mathcal M)}\|\omega-u\|_2$.
In particular, since $\tilde\omega^*(\mathcal M)\in\mathcal C^*(\mathcal M)$ and $z_n\in\Omega^\epsilon$, we have
$F(z_n,\mathcal M)\le F^*+G_{\epsilon}(\mathcal{M})\|z_n-\tilde\omega^*(\mathcal M)\|_2$. 
Using Lemma~\ref{lemma: set stability}, which applies to our fixed representative $\tilde\omega^*(\mathcal M)$, for all sufficiently large $n$,
$$
\|z_n-\tilde\omega^*(\mathcal M)\|_2\le L_{\mathcal M}\|\bar{\mathcal M}(t_n)-\mathcal M\|.
$$
Substituting the bound $F(z_n,\mathcal M)
\le
F^*+G_\epsilon(\mathcal M)L_{\mathcal M}\|\bar{\mathcal M}(t_n)-\mathcal M\|$
into \eqref{eq:gap_recursion_pre_new} and rearranging yields, for all sufficiently large $n$,
\begin{equation}
\label{eq:gap_recursion_new}
2\eta_n\big(F(x_{n-1},\mathcal M)-F^*\big)
\le 
\|x_{n-1}-z_n\|_2^2-\|x_n-z_n\|_2^2
+\eta_n^2\|\bar g_n\|_2^2
+2\eta_n G_\epsilon(\mathcal M)L_{\mathcal M}\|\bar{\mathcal M}(t_n)-\mathcal M\|
+4\eta_n E_n .
\end{equation}

Let $T$ denote the total budget in the original time scale, and let
$0<t_1<t_2<\cdots$ be the update times with increments $\Gamma_n:=t_n-t_{n-1}$ (with $t_0:=0$).
Define $N(T)$ as the number of completed updates by time $T$:
$$
N(T):=\max\Big\{n\ge 0:\ \sum_{k=1}^n \Gamma_k \le T\Big\},
$$
Since $\Gamma_n\in\mathbb N$ and $\Gamma_n\ge 1$ for all $n$, we have
$\sum_{k=1}^n \Gamma_k \ge n$, hence $\sum_{k=1}^n \Gamma_k\to\infty$ and therefore $N(T)\to\infty$ as $T\to\infty$.

Summing \eqref{eq:gap_recursion_new} over $n=1,\dots,N(T)$ gives
\begin{equation}
\label{eq:summed_new}
\sum_{n=1}^{N(T)} 2\eta_n\big(F(x_{n-1},\mathcal{M})-F^*\big)
\le 
\underbrace{\sum_{n=1}^{N(T)}\big(\|x_{n-1}-z_n\|_2^2-\|x_n-z_n\|_2^2\big)}_{=:S_{N(T)}}
+\sum_{n=1}^{N(T)} \eta_n^2\|\bar g_n\|_2^2 \nonumber
+\sum_{n=1}^{N(T)} \eta_n \delta_n,
\end{equation}
where we define
$$
\delta_n
:=
2 G_\epsilon(\mathcal M)L_{\mathcal M}\|\bar{\mathcal M}(t_n)-\mathcal M\|
+4E_n .
$$
By Lemma~\ref{lem:telescope}, $S_{N(T)}\le \|x_0-z_1\|_2^2 + 2D\sum_{n=1}^{{N(T)}-1}\|z_{n+1}-z_n\|_2$.

By the triangle inequality and Lemma \ref{lemma: set stability}, for $n\ge n_1$,
$$
\|z_{n+1}-z_n\|_2
\le \|z_{n+1}-\tilde\omega^*(\mathcal{M})\|_2+\|z_n-\tilde\omega^*(\mathcal{M})\|_2
\le L_{\mathcal M}\big(\|\bar{\mathcal M}(t_{n+1})-\mathcal M\|+\|\bar{\mathcal M}(t_n)-\mathcal M\|\big).
$$
Summing over $n=n_1,\dots,N(T)-1$ yields
$$
\sum_{n=n_1}^{N(T)-1}\|z_{n+1}-z_n\|_2
\le 2L_{\mathcal M}\sum_{n=n_1}^{N(T)}\|\bar{\mathcal M}(t_n)-\mathcal M\|.
$$
The finitely many initial terms $\sum_{n=1}^{n_1-1}\|z_{n+1}-z_n\|_2$ are $O(1)$ and thus negligible after dividing by $\sum_{n=1}^{N(T)} \eta_n$. Hence, it suffices to show
\begin{equation}
\label{eq:key_reduce}
\lim_{N(T)\rightarrow\infty}\frac{\sum_{n=1}^{N(T)}\|\bar{\mathcal M}(t_n)-\mathcal M\|}{\sum_{n=1}^{N(T)} \eta_n}= 0.
\end{equation}

By \eqref{eq: sample_lb}, there exists a constant $c_*>0$
and an almost surely finite random time $t_0$ such that
$
N(s,a;t)\ge c_*\, t^{1-\alpha},
\forall (s,a)\in\mathcal S\times\mathcal A,\ \forall t\ge t_0,
$
almost surely.
For each $(s,a)\in\mathcal S\times\mathcal A$, let
$$
A_t(s,a):=\{N(s,a;t)\ge c_*\, t^{1-\alpha}\}.
$$
Fix $(s,a)$ and a next state $s'$. Conditional on $N(s,a;t)=m$, by Hoeffding's inequality, for any $\lambda>0$,
$$
\mathbb P\!\left(
|\bar P_t(s'|s,a)-P(s'|s,a)|>\lambda \,\middle|\, N(s,a;t)=m
\right)
\le
2e^{-2m\lambda^2},
$$
and similarly,
$$
\mathbb P\!\left(
|\bar R_t(s,a)-R(s,a)|>\lambda \,\middle|\, N(s,a;t)=m
\right)
\le
2e^{-2 m\lambda^2}.
$$

Take
$
\lambda_t:=K_\beta\sqrt{\frac{\log t}{t^{1-\alpha}}},
$
with $K_\beta>0$ to be specified later. Then, by conditioning on $N(s,a;t)$ and restricting to the event
$A_t(s,a)$, we obtain
\begin{align*}
&\mathbb P\!\left(
|\bar P_t(s'|s,a)-P(s'|s,a)|>\lambda_t,\ A_t(s,a)
\right) \\
&\qquad\le
\sum_{m\ge c_* t^{1-\alpha}} 2e^{-2m\lambda_t^2}\,\mathbb P(N(s,a;t)=m) \\
&\qquad\le
2e^{-2c_* t^{1-\alpha}\lambda_t^2}
=
2e^{-2c_*K_\beta^2\log t}
=
2t^{-2c_*K_\beta^2}.
\end{align*}
Choosing $K_\beta$ sufficiently large such that $2c_*K_\beta^2\ge \beta+1$, we get
$$
\mathbb P\!\left(
|\bar P_t(s'|s,a)-P(s'|s,a)|>\lambda_t,\ A_t(s,a)
\right)
\le 2t^{-(\beta+1)}.
$$
Therefore,
$$
\mathbb P\!\left(
|\bar P_t(s'|s,a)-P(s'|s,a)|>\lambda_t
\right)
\le
\mathbb P(A_t(s,a)^c)+2t^{-(\beta+1)}.
$$
By \eqref{eq: error up bound}, for each $(s,a)$ there exist constants $C_1,C_2,C_3,C_4>0$ such that
$$
\mathbb P(A_t(s,a)^c)\le C_1 \exp(-C_2 t^{1-2\alpha})+ C_3\exp(-C_4t)
$$
for all sufficiently large $t$. Hence,
$$
\mathbb P\!\left(
|\bar P_t(s'|s,a)-P(s'|s,a)|>\lambda_t
\right)
\le
C_1 \exp(-C_2 t^{1-2\alpha})+ C_3\exp(-C_4t)+2t^{-(\beta+1)}.
$$
The same argument yields
$$
\mathbb P\!\left(
|\bar R_t(s,a)-R(s,a)|>\lambda_t
\right)
\le
C_1 \exp(-C_2 t^{1-2\alpha})+ C_3\exp(-C_4t)+2t^{-(\beta+1)}.
$$

Since the state-action space and state space are finite, a union bound over all transition and reward coordinates gives
$$
\mathbb P\!\left(
\|\bar{\mathcal M}(t)-\mathcal M\|_\infty
>
C_\beta \sqrt{\frac{\log t}{t^{1-\alpha}}}
\right)
\le
C\exp(-ct^{1-2\alpha}) + C' \exp(-c't)+C'' t^{-(\beta+1)}
$$
for all sufficiently large $t$, where $C_\beta,C,C',C'',c,c'>0$ are constants.
Because
$$
\sum_{t=1}^\infty \big(C\exp(-ct^{1-2\alpha})+C' \exp(-c't)+C'' t^{-(\beta+1)}\big)<\infty,
$$
the Borel-Cantelli lemma yields
\begin{equation}
\label{eq:as_rate}
\|\bar{\mathcal M}(t)-\mathcal M\|_\infty = O\!\left(\sqrt{\frac{\log t}{t^{1-\alpha}}}\right)
\qquad\text{a.s.}
\end{equation}
In particular, along the subsequence $t=t_n$,
$$
\|\bar{\mathcal M}(t_n)-\mathcal M\|
\le C\,\sqrt{\frac{\log t_n}{t^{1-\alpha}_n}}
\qquad\text{a.s. for all large }n,
$$
for some finite random constant $C$ with probability one.

Under $\Gamma_n\ge \tilde c n$, we have
$$
t_n=\sum_{k=1}^n \Gamma_k \ge \tilde c\sum_{k=1}^n k = \frac{\tilde c}{2}n(n+1)\ge c\,n^2
$$
for some $c>0$. Therefore, for all large $n$,
$$
\sqrt{\frac{\log t_n}{t^{1-\alpha}_n}}
\le
\sqrt{\frac{\log(cn^2)}{(c n^2)^{1-\alpha}}}
\le
\frac{C'}{n^{1-\alpha}}\sqrt{\log n}
$$
for some constant $C'>0$. Combining with \eqref{eq:as_rate} gives, almost surely for all large $n$,
$$
\|\bar{\mathcal M}(t_n)-\mathcal M\|
\le
\frac{C''\sqrt{\log n}}{n^{1-\alpha}}
$$
for some finite random $C''$.
Hence
$$
\sum_{n=1}^{N(T)}\|\bar{\mathcal M}(t_n)-\mathcal M\|
\le
C''\sum_{n=2}^{N(T)}\frac{\sqrt{\log n}}{n^{1-\alpha}}
= O(N(T)^\alpha\log N(T)^{1/2}).
$$
Since $\|x_0-z_1\|_2^2$ and $D$ is finite, it holds that
\begin{equation}
    \lim_{N(T) \rightarrow \infty} \frac{\|x_0-z_1\|_2^2 + 2D\sum_{n=1}^{N(T)-1}\|z_{n+1}-z_n\|_2}{ \sum_{n=1}^{N(T)} \eta_n}=0,
\end{equation}
where we use $\alpha\in(0,1/2)$.

Moreover, by Lemma~\ref{lemma: model_perturbation}, $E_n\to 0$ a.s., hence $\delta_n\to 0$ a.s.
Since $\eta_n=1/\sqrt{N(T)}$ for $n=1,\dots,N(T)$, we have
$$
\frac{\sum_{n=1}^{N(T)}\eta_n\delta_n}{\sum_{n=1}^{N(T)}\eta_n}
=
\frac{\frac1{\sqrt{N(T)}}\sum_{n=1}^{N(T)}\delta_n}{\sqrt{N(T)}}
=
\frac1{N(T)}\sum_{n=1}^{N(T)}\delta_n.
$$
Therefore, by Cesàro's theorem,
\begin{equation}
\label{eq:toeplitz}
\lim_{N(T)\to\infty}\frac{\sum_{n=1}^{N(T)}\eta_n\delta_n}{\sum_{n=1}^{N(T)}\eta_n}=0\qquad\text{a.s.}
\end{equation}

According to Lemma \ref{lemma: subgradients bound}, $\|\bar g_n\|_2 \le G$ for all $n$, almost surely. Then,
$$
\frac{\sum_{n=1}^{N(T)} \eta_n^2\|\bar g_n\|_2^2}{\sum_{n=1}^{N(T)}\eta_n}
\le
G^2\frac{\sum_{n=1}^{N(T)}\eta_n^2}{\sum_{n=1}^{N(T)}\eta_n}
\to 0,
$$
Finally, $\|x_0-z_1\|_2^2$ is constant and $D<\infty$, so
$\frac{\|x_0-z_1\|_2^2}{\sum_{n=1}^{N(T)}\eta_n}\to 0$.

Divide both sides of \eqref{eq:summed_new} by $2\sum_{n=1}^{N(T)}\eta_n$.
Using Lemma~\ref{lem:telescope} and the limits above, we obtain
$$
\limsup_{{N(T)}\to\infty}
\frac{\sum_{n=1}^{N(T)}\eta_n\big(F(x_{n-1},\mathcal{M})-F^*\big)}{\sum_{n=1}^{N(T)}\eta_n}
\le 0
\qquad\text{a.s.}
$$

Since $\eta_n=1/\sqrt{N(T)}$ for $n=1,\dots,N(T)$, this implies
$$
\limsup_{N(T)\to\infty}
\frac1{N(T)}\sum_{n=1}^{N(T)}F(x_{n-1},\mathcal M)
\le F^*
\qquad\text{a.s.}
$$

Next, because $x_n\in\mathcal W^\epsilon(\bar{\mathcal M}(t_n))$ and
$\bar{\mathcal M}(t_n)\to\mathcal M$ almost surely, the violation of the true flow constraints satisfies
$$
\|A(\mathcal M)x_n\|_2
=
\|(A(\mathcal M)-A(\bar{\mathcal M}(t_n)))x_n\|_2
\le
\|A(\mathcal M)-A(\bar{\mathcal M}(t_n))\|
\to 0
\qquad\text{a.s.}
$$

Since $x_n\in\Omega^\epsilon$ for all $n$ and the true flow residual satisfies
$
\|A(\mathcal M)x_n\|_2\to 0,
$
every limit point of $\{x_n\}$ belongs to $\mathcal W(\mathcal M)$. By continuity of $F(\cdot,\mathcal M)$ on $\Omega^\epsilon$, it follows that
$$
\liminf_{n\to\infty}F(x_n,\mathcal M)\ge F^* \qquad\text{a.s.}
$$
Hence, by the elementary fact that the Ces\`aro average of a sequence is bounded below by its liminf,
$$
\liminf_{N(T)\to\infty}\frac1{N(T)}\sum_{n=1}^{N(T)}F(x_n,\mathcal M)\ge F^*
\qquad\text{a.s.}
$$

Since $F(\cdot,\mathcal M)$ is bounded on $\Omega^\epsilon$, there exists $B<\infty$ such that
$
|F(x_n,\mathcal M)|\le B,\forall n.
$
Therefore,
$$
\left|
\frac1{N(T)}\sum_{n=1}^{N(T)}F(x_n,\mathcal M)
-
\frac1{N(T)}\sum_{n=1}^{N(T)}F(x_{n-1},\mathcal M)
\right|
=
\frac{|F(x_{N(T)},\mathcal M)-F(x_0,\mathcal M)|}{N(T)}
\le \frac{2B}{N(T)}\to 0.
$$
Combining the previous bounds, we conclude that
$$
\lim_{N(T)\to\infty}\frac1{N(T)}\sum_{n=1}^{N(T)}F(x_n,\mathcal M)=F^*
\qquad\text{a.s.}
$$

By definition of the $\omega_n$, we know that $\omega_{N(T)} =\frac{1}{N(T)}\sum_{n=1}^{N(T)}x_n$. By convexity of $F(\cdot,\mathcal{M})$ and Jensen's inequality,
$$
F(\omega_{N(T)},\mathcal M)
\le
\frac1{N(T)}\sum_{n=1}^{N(T)}F(x_n,\mathcal M).
$$
Since $\Omega^\epsilon$ is convex and $x_n\in\Omega^\epsilon$ for all $n$, we have
$$
\omega_{N(T)}=\frac1{N(T)}\sum_{n=1}^{N(T)}x_n \in \Omega^\epsilon.
$$
Moreover,
$$
A(\mathcal M)\omega_{N(T)}
=
\frac1{N(T)}\sum_{n=1}^{N(T)}A(\mathcal M)x_n \to 0
\qquad\text{a.s.},
$$
since $\|A(\mathcal M)x_n\|_2\to 0$ and Ces\`aro averages preserve convergence.
Therefore, every limit point of $\{\omega_{N(T)}\}$ belongs to $\mathcal W(\mathcal M)\cap\Omega^\epsilon$.
By continuity of $F(\cdot,\mathcal M)$ on $\Omega^\epsilon$, it follows that
$$
\liminf_{N(T)\to\infty}F(\omega_{N(T)},\mathcal M)\ge F^*
\qquad\text{a.s.}
$$
Therefore,
$$
F^*
\le
\liminf_{N(T)\to\infty}F(\omega_{N(T)},\mathcal M)
\le
\limsup_{N(T)\to\infty}F(\omega_{N(T)},\mathcal M)
\le
\lim_{N(T)\to\infty}\frac1{N(T)}\sum_{n=1}^{N(T)}F(x_n,\mathcal M)
=
F^*,
$$
which implies
$$
\lim_{N(T)\to\infty}\big(F(\omega_{N(T)},\mathcal M)-F^*\big)=0
\qquad\text{a.s.}
$$

\section{Proof of Lemma~\ref{lemma: ratio convergence}}
\label{sec: ratio converge}
Let $\mathcal{E}_n(\epsilon)$ denote the event that the empirical sampling ratio deviates from the stationary distribution $\beta^n$ by more than $\epsilon$ (in the infinity norm) during the $n$-th time interval $[t_n, t_{n+1})$:
\begin{equation*}
\mathcal{E}_n(\epsilon) = \left\{ \max_{(s,a)\in\mathcal{S}\times\mathcal{A}} \left|\frac{1}{\Gamma_n}N^n(s,a)-\beta^n(s,a)\right| > \epsilon \right\},
\end{equation*}
where $\Gamma_n:=t_{n+1}-t_n$, $N^n(s,a)$ is the number of visits to $(s,a)$ during $[t_n,t_{n+1})$,
and $\beta^n$ is the stationary distribution of the Markov chain induced by the fixed behavior policy
$\pi_{t_n}$ on the $n$-th interval.

According to Algorithm \ref{alg:algorithm1}, during the time interval $[t_n,t_{n+1})$, the behavior policy is fixed as $\pi_{t_n}$. Consequently, the process $\{(s_l,a_l)\}_{l=t_n}^{t_{n+1}-1}$ evolves as a time-homogeneous ergodic Markov chain on the finite space $\mathcal{S}\times\mathcal{A}$ with the unique stationary distribution $\beta^n$.

We invoke the large deviations principle for the empirical measure of Markov chains (Theorem 1.2 in~\citealt{ellis1988large}). For any $\epsilon>0$, there exists a rate function $I_n(\epsilon,\beta^n)$ depending on the transition dynamics induced by $\pi_{t_n}$. Specifically, for the finite state-action space $\mathcal{S}\times \mathcal{A}$, the deviation probability is uniformly bounded over all possible initial state-action pairs $(s_{t_n},a_{t_n})\in\mathcal{S}\times\mathcal{A}$. That is, there exist  a constant $C(\epsilon)>0$ such that for sufficiently large $\Gamma_n$ \citep{zhu2024uncertainty}:
\begin{equation}
\label{eq: ldt bound}
    \mathbb{P}(\mathcal{E}_n(\epsilon)\mid s_{t_n},a_{t_n}) \le C(\epsilon)\exp(-\Gamma_n I_n(\epsilon,\beta^n)),\quad \forall (s_{t_n},a_{t_n}).
\end{equation}

Next, we establish a uniform lower bound for the rate function.
By construction of the algorithm, the target allocation satisfies
$
x_n\in\mathcal W^\epsilon(\bar{\mathcal M}(t_n))
$.
Therefore, $\beta^n$ stay in a compact set $K\subset \mathrm{int}(\Delta(\mathcal S\times\mathcal A))$. The rate function $I_n(\epsilon,\beta^n)$ is continuous in the induced transition kernel. 
Fix $\epsilon>0$. For each $\beta\in K$, the LDP rate satisfies $I(\epsilon,\beta)>0$.
By continuity of $I(\epsilon,\cdot)$ on the compact set $K$, we have
$$
I^*(\epsilon):=\min_{\beta\in K} I(\epsilon,\beta) > 0.
$$
Therefore, for all $n\ge n_0(\epsilon)$,
$
I_n(\epsilon,\beta^n)\ge I^*(\epsilon).
$
Moreover, by \eqref{eq: ldt bound}, there exists $\Gamma_0(\epsilon)$ such that
for all $n$ with $\Gamma_n\ge \Gamma_0(\epsilon)$,
$$
\sup_{(s_{t_n},a_{t_n})}\mathbb{P}(\mathcal{E}_n(\epsilon)\mid s_{t_n},a_{t_n})
\le C(\epsilon)\,\exp\!\big(-\Gamma_n I_n(\epsilon,\beta^n)\big).
$$
Since $\Gamma_n\to\infty$, there exists $n_1(\epsilon)$ such that
$\Gamma_n\ge \Gamma_0(\epsilon)$ for all $n\ge n_1(\epsilon)$.
Let $N(\epsilon):=\max\{n_0(\epsilon),n_1(\epsilon)\}$.
Then for all $n\ge N(\epsilon)$
$$
\mathbb{P}(\mathcal{E}_n(\epsilon))
\le \sup_{(s_{t_n},a_{t_n})}\mathbb{P}(\mathcal{E}_n(\epsilon)\mid s_{t_n},a_{t_n})
\le C(\epsilon)\exp\!\big(-\Gamma_n I^*(\epsilon)\big).
$$
Therefore,
$$
\sum_{n=1}^{\infty}\mathbb{P}(\mathcal{E}_n(\epsilon))
\le \sum_{n=1}^{N(\epsilon)-1}\mathbb{P}(\mathcal{E}_n(\epsilon))
+\sum_{n=N(\epsilon)}^{\infty} C(\epsilon)\exp\!\big(-\Gamma_n I^*(\epsilon)\big)
<\infty,
$$
where the last inequality follows from the growth condition $\Gamma_n = \lceil \tilde c n \rceil$ for some $\tilde c>0$.
Applying the Borel-Cantelli lemma, the event $\mathcal E_n(\epsilon)$ occurs infinitely often with probability $0$. Thus, for each fixed $\epsilon>0$,
$$
\max_{(s,a) \in \mathcal S\times \mathcal A}
\left| \frac{1}{\Gamma_n} N^n(s,a) - \beta^n(s,a) \right|
\le \epsilon
$$
for all sufficiently large $n$, almost surely.
Now let $\epsilon_m:=1/m$ for $m\in\mathbb N$. By the above argument, for each $m$,
$
\mathbb P\big(\mathcal E_n(1/m)\ \text{i.o.}\big)=0.
$
Taking the countable intersection over all $m\in\mathbb N$, we conclude that
\begin{equation*}
\max_{(s,a) \in \mathcal{S}\times \mathcal{A}}
\left| \frac{1}{\Gamma_n} N^n(s,a) - \beta^n(s,a) \right| \rightarrow 0
\end{equation*}
almost surely as $n\to\infty$.

\section{Proof of Theorem~\ref{thm: robust optimality}}
\label{sec: robust optimality}
According to Lemma \ref{lemma: consistency}, every state-action pair is visited infinitely often almost surely:
\begin{equation*}
    \lim_{t \rightarrow \infty} N(s,a;t) = \infty, \quad \forall s \in \mathcal{S}, a \in \mathcal{A}.
\end{equation*}
Therefore, by Azuma--Hoeffding together with Borel--Cantelli (or the martingale SLLN),
the empirical transition probabilities and rewards converge coordinate-wise almost surely along visit counts,
and since $N(s,a;t)\to\infty$, we obtain
$\bar{\mathcal M}(t)\to\mathcal M$ almost surely as $t\to\infty$ (equivalent to $T\rightarrow \infty$).

Since the optimal policy $\pi^*_{\mathcal M}$ is unique (equivalently, all state-wise optimality gaps are strictly positive).
Then by continuity of $Q^*_{\mathcal M}$ in the model parameters, there exists a neighborhood $\mathcal U$ of $\mathcal M$
such that $\pi^*_{\mathcal M'}=\pi^*_{\mathcal M}$ for all $\mathcal M'\in\mathcal U$.
Since $\bar{\mathcal M}(T)\to\mathcal M$ a.s., we have $\bar{\mathcal M}(T)\in\mathcal U$ for all large $T$,
hence $\hat\pi_T=\pi^*_{\bar{\mathcal M}(T)}=\pi^*_{\mathcal M}$ eventually, and thus
$\hat\pi_T\to\pi^*_{\mathcal M}$ almost surely.

As the total budget $T\to\infty$, the number of algorithm updates $N(T)\to\infty$.
According to Theorem~\ref{thm: sub-grad converge}, we have $F(\omega_{N(T)},\mathcal M)\to F^*$ almost surely,
where $\omega_{N(T)}=\frac1{N(T)}\sum_{n=1}^{N(T)} x_n$.
In particular, $\mathrm{dist}(\omega_{N(T)},\mathcal C^*(\mathcal M))\to0$ a.s.; hence every limit point belongs to
$\mathcal C^*(\mathcal M)$. 

At update times $t_n$, the behavior policy satisfies
$\pi_{t_n}=\epsilon_{t_n}\pi^u + (1-\epsilon_{t_n})\pi^e_{\bar{\mathcal M}(t_n)}$, and $\pi_t$ is constant between updates.
Since $\epsilon_t\to0$ and $\bar{\mathcal M}(t_n)\to\mathcal M$ a.s., and since the map
$\omega\mapsto \pi^e(\cdot|s)$ defined by
$\pi^e(a|s)=\omega_{sa}/\sum_{b}\omega_{sb}$ is continuous on $\{\omega:\min_{s,a}\omega_{sa}\ge\epsilon\}$, every limit point of the sequence
$\{\pi_{t_n}\}$ is induced by some element of $\mathcal C^*(\mathcal M)$.

Let $\beta^n$ denote the unique stationary distribution on $\mathcal S\times\mathcal A$
of the Markov chain induced by the fixed policy $\pi_{t_n}$ on the true MDP $\mathcal M$.
By continuity of the stationary distribution for finite ergodic chains,
every limit point of
$\{\beta^n\}$ belongs to $\mathcal C^*(\mathcal M)$. Equivalently,
$
\mathrm{dist}(\beta^n,\mathcal C^*(\mathcal M))\to 0,\text{a.s.}
$

By definition,  
$$
N(s,a;T)=\sum_{n=1}^{N(T)} N^n(s,a) + R_T(s,a),
\qquad 0\le R_T(s,a)\le \Gamma_{N(T)+1},
$$
hence $\|R_T\|_\infty/T\to0$ because $\Gamma_{N(T)+1}/T\to0$ for $\Gamma_n=\lceil\tilde c n\rceil$.
Therefore,
$$
\frac1T N(s,a;T)
=
\sum_{n=1}^{N(T)}\frac{\Gamma_n}{T}\Big(\frac{1}{\Gamma_n}N^n(s,a)\Big) + o(1).
$$
Add and subtract $\beta^n(s,a)$:
$$
\frac1T N(s,a;T)
=
\sum_{n=1}^{N(T)}\frac{\Gamma_n}{T}\Big(\frac{1}{\Gamma_n}N^n(s,a)-\beta^n(s,a)\Big)
+
\sum_{n=1}^{N(T)}\frac{\Gamma_n}{T}\beta^n(s,a)
+o(1).
$$

By Lemma~\ref{lemma: ratio convergence},
$\max_{s,a}\big|\frac{1}{\Gamma_n}N^n(s,a)-\beta^n(s,a)\big|\to0$ a.s.
Moreover, since $\sum_{n\le\tau(T)}\Gamma_n/T\to1$ and $\max_{n\le\tau(T)}\Gamma_n/T\to0$,
the weighted Toeplitz lemma implies
$$
\mathrm{dist}\!\left(
\sum_{n=1}^{N(T)}\frac{\Gamma_n}{T}\beta^n,\,
\mathcal C^*(\mathcal M)
\right)\to0
\qquad\text{a.s.}
$$
Combining the above displays yields
$$
\mathrm{dist}\!\left(
\left(\frac1T N(s,a;T)\right)_{s,a},
\mathcal C^*(\mathcal M)
\right)\to0
\qquad\text{a.s. as }T\to\infty.
$$

Let $\alpha(T):=(N(s,a;T)/T)_{s,a}$. Since $F(\cdot,\mathcal M)$ is continuous and equals $F^*$ on $\mathcal C^*(\mathcal M)$, we obtain $F(\alpha(T),\mathcal M)\to F^*$ almost surely. 

We now prove the robust optimality. Throughout the worst-case comparison, the infimum over $\mathcal M$ is taken over the class $\mathfrak M(\Delta_0)$. In the generative model setting, the feasible allocation set is the probability simplex $\Omega$. Let
\[
\mathcal C^*_{\Omega}(\mathcal M)
:=
\arg\min_{\omega\in\Omega}F(\omega,\mathcal M)
\]
denote the set of surrogate optimal allocations under the generative model setting. The convergence argument in Theorem~\ref{thm: sub-grad converge} implies
\[
\mathrm{dist}\!\left(\alpha(T),\mathcal C^*_{\Omega}(\mathcal M)\right)\to 0
\qquad\text{a.s.}
\]
Thus, by continuity,
$
F(\alpha(T),\mathcal M)
\to
\min_{\omega\in\Omega}F(\omega,\mathcal M)
$ almost surely. By optimality over $\Omega$, and since the uniform allocation $
\omega^u_{sa}=\frac1{C(S,A)}
$ belongs to $\Omega$, we have $\min_{\omega\in\Omega}F(\omega,\mathcal M)
\le
F(\omega^u,\mathcal M).$
Using
$$
\mathbb V[R_{\mathcal M}(s,a)]\le \frac14,\qquad
\mathbb V_{P_{\mathcal M}(s,a)}
\!\left[V^{\pi^*_{\mathcal M}}_{\mathcal M}\right]
\le \frac1{(1-\gamma)^2},
\qquad
\Delta_{sa}\ge \Delta_{\min}(\mathcal M)\ge \Delta_0,
$$
we obtain, for some universal constant $c>0$,
$$
F(\omega^u,\mathcal M)
\le
c\,\frac{C(S,A)}{(1-\gamma)^4\Delta_{\min}(\mathcal M)^2}.
$$
Therefore,
\[
\limsup_{T\to\infty}F(\alpha(T),\mathcal M)
\le
c\,\frac{C(S,A)}{(1-\gamma)^4\Delta_0^2}
\qquad\text{a.s.}
\]
By the leading-order lower bound in Theorem~\ref{thm: rate lower bound opt},
\[
\mathcal R(\mathcal M,\alpha(T))
\ge
\big(F(\alpha(T),\mathcal M)\big)^{-1}.
\]
Hence,
\[
\liminf_{T\to\infty}
\mathcal R(\mathcal M,\alpha(T))
=
\Omega\!\left(
\frac{(1-\gamma)^4\Delta_0^2}{C(S,A)}
\right)
\qquad\text{a.s.}
\]

Since the constants are uniform over $\mathfrak M(\Delta_0)$, we obtain the worst-case guarantee
\[
\inf_{\mathcal M\in\mathfrak M(\Delta_0)}\liminf_{T\to\infty}
\mathcal R(\mathcal M,\alpha(T))
=
\Omega\!\left(
\frac{(1-\gamma)^4\Delta_0^2}{C(S,A)}
\right)
\]

On the other hand, Appendix~\ref{sec: robust opt} shows that the robust optimal value over this hard-instance class satisfies
\[
\mathcal R^*
=
O\!\left(
\frac{(1-\gamma)^3\Delta_0^2}{C(S,A)}
\right).
\]
Therefore, Algorithm~\ref{alg:algorithm1} achieves the robust-optimal scaling up to an additional factor of $(1-\gamma)$.

\section{Proof of Lemma~\ref{lemma: linear rate}}
Recall that $I_1(x(s,a))$ is the Fenchel-Legendre transform of the logarithmic moment generating function of $X_{\mathcal{M}}(s,a)$ and is defined as
$$
     I_1(x(s,a)) =
     \sup_{\rho(s,a)}
     \left(
     {\rho(s,a)^\top}x(s,a)
     -
     \log\mathbb{E}_{P_{\mathcal{M}}(s,a)}
     \left[
     \exp(\rho(s,a)^\top X_{\mathcal{M}}(s,a))
     \right]
     \right).
$$
By choosing $\rho(s,a)=\eta_1\gamma v$ for some constant $\eta_1\in\mathbb R$, it holds that
\begin{equation}
\label{eq: linear bound 1}
\begin{aligned}
    I_1(x(s,a))
    \geq\;&
    \eta_1\gamma v^\top x(s,a)
    -
    \log \mathbb{E}
    \left[
    \exp\left(
    \eta_1\gamma v^\top X_{\mathcal{M}}(s,a)
    \right)
    \right]
\end{aligned}
\end{equation}

Since $\mathbb{E}[X_{\mathcal M}(s,a)]=P_{\mathcal M}(s,a)$, we upper-bound the logarithmic moment generating function as follows:
\begin{equation*}
\begin{aligned}
&\log \mathbb{E}
\left[
\exp\left(
\eta_1\gamma v^\top X_{\mathcal{M}}(s,a)
\right)
\right] \\
=\;&
\eta_1\gamma v^\top P_{\mathcal M}(s,a)
+
\log \mathbb{E}
\left[
\exp\left(
\eta_1\gamma v^\top
\big(X_{\mathcal M}(s,a)-P_{\mathcal M}(s,a)\big)
\right)
\right] \\
=\;&
\eta_1\gamma v^\top P_{\mathcal M}(s,a)
+
\log
\bigg[
1
+
\sum_{k=2}^{\infty}
\frac{
\eta_1^k
\mathbb{E}
\left[
\left(
\gamma v^\top
\big(X_{\mathcal M}(s,a)-P_{\mathcal M}(s,a)\big)
\right)^k
\right]
}{k!}
\bigg] \\
\le\;&
\eta_1\gamma v^\top P_{\mathcal M}(s,a)
+
\log
\bigg[
1
+
\sum_{k=2}^{\infty}
\frac{
|\eta_1|^k
\mathbb{E}
\left[
\left|
\gamma v^\top
\big(X_{\mathcal M}(s,a)-P_{\mathcal M}(s,a)\big)
\right|^k
\right]
}{k!}
\bigg] \\
\le\;&
\eta_1\gamma v^\top P_{\mathcal M}(s,a)
+
\log
\bigg[
1
+
\left(
\exp\left(
|\eta_1|
\left\lVert
\gamma v^\top X_{\mathcal M}(s,a)
\right\rVert_\infty
\right)
-
|\eta_1|
\left\lVert
\gamma v^\top X_{\mathcal M}(s,a)
\right\rVert_\infty
-1
\right)
\bigg] \\
\le\;&
\eta_1\gamma v^\top P_{\mathcal M}(s,a)
+
\left(
\exp\left(
|\eta_1|
\left\lVert
\gamma v^\top X_{\mathcal M}(s,a)
\right\rVert_\infty
\right)
-
|\eta_1|
\left\lVert
\gamma v^\top X_{\mathcal M}(s,a)
\right\rVert_\infty
-1
\right),
\end{aligned}
\end{equation*}
where the second equality uses the Taylor expansion and the fact that
$$
\mathbb{E}\left[
\gamma v^\top
\big(X_{\mathcal M}(s,a)-P_{\mathcal M}(s,a)\big)
\right]=0,
$$
and the inequalities use
$$
\left|
\gamma v^\top
\big(X_{\mathcal M}(s,a)-P_{\mathcal M}(s,a)\big)
\right|
\le
\left\lVert
\gamma v^\top X_{\mathcal M}(s,a)
\right\rVert_\infty
$$
and $\log(1+u)\le u$.

Substituting this upper bound into \eqref{eq: linear bound 1}, we obtain
\begin{equation*}
I_1(x(s,a))
\ge
\eta_1 \Delta
-
\left(e^{|\eta_1|M}-|\eta_1|M-1\right),
\end{equation*}
where
$$
\Delta
:=
\gamma v^\top
\big(x(s,a)-P_{\mathcal M}(s,a)\big),
\qquad
M
:=
\left\lVert
\gamma v^\top X_{\mathcal M}(s,a)
\right\rVert_\infty .
$$
If \(I_1(x(s,a))=\infty\), the desired inequality holds trivially. Hence, it suffices to consider the case \(I_1(x(s,a))<\infty\), which implies that \(x(s,a)\) is absolutely continuous with respect to \(P_{\mathcal M}(s,a)\).
If $M=0$, then, under the above support condition, $\Delta=0$, and the desired bound is trivial. Hence, assume $M>0$.

Since the right-hand side depends on $\eta_1$ only through $\eta_1\Delta$ and $|\eta_1|$, it is optimized by choosing $\eta_1$ with the same sign as $\Delta$. Thus,
\begin{equation*}
I_1(x(s,a))
\ge
\sup_{t\ge 0}
\left\{
t|\Delta|
-
\left(e^{tM}-tM-1\right)
\right\}.
\end{equation*}
The supremum is attained at
$
t
=
\frac1M
\log\left(1+\frac{|\Delta|}{M}\right).
$
Plugging this choice back gives
\begin{equation*}
I_1(x(s,a))
\ge
\left(1+\frac{|\Delta|}{M}\right)
\log\left(1+\frac{|\Delta|}{M}\right)
-
\frac{|\Delta|}{M}.
\end{equation*}
Using the inequality
\begin{equation*}
(1+u)\log(1+u)-u
\ge
\frac{u^2}{2+u},
\qquad u\ge 0,
\end{equation*}
with $u=|\Delta|/M$, we obtain
\begin{equation*}
I_1(x(s,a))
\ge
\frac{\Delta^2}{2M^2+M|\Delta|}.
\end{equation*}
Since $x(s,a)$ and $P_{\mathcal M}(s,a)$ are probability vectors and
$
0
\le
\gamma v^\top X_{\mathcal M}(s,a)
\le
M,
$
we have $|\Delta|\le M$. Therefore,
\begin{equation*}
I_1(x(s,a))
\ge
\frac{\Delta^2}{3M^2}.
\end{equation*}
Substituting the definitions of $\Delta$ and $M$ back, we obtain
\begin{equation*}
I_1(x(s,a))
\ge
\frac{
\left(
\gamma v^\top
(x(s,a)-P_{\mathcal M}(s,a))
\right)^2
}{
3
\left\lVert
\gamma v^\top X_{\mathcal M}(s,a)
\right\rVert_\infty^2
}.
\end{equation*}

\section{Proof of Lemma~\ref{lemma: linear constraint}}
Recall that $\Delta_{sa} := V^{\pi^*_{\mathcal{M}}}_{\mathcal{M}}(s)-Q^{\pi^*_{\mathcal{M}}}_{\mathcal{M}}(s,a)$ denotes the optimality gap of the state-action pair $(s,a)$. Conditioned on the set $ \mathcal{E}_{s,a}$, the perturbed value function satisfies $\tilde{\mathcal{M}}$ satisfies $Q^{\pi^*_{\mathcal{M}}}_{\tilde{\mathcal{M}}}(s,a) >V^{\pi^*_{\mathcal{M}}}_{\tilde{\mathcal{M}}}(s)$. We can bound the optimality gap as follows:
\begin{equation*}
\begin{aligned}
\Delta_{sa}
&=
V^{\pi^*_{\mathcal{M}}}_{\mathcal{M}}(s)
-
Q^{\pi^*_{\mathcal{M}}}_{\mathcal{M}}(s,a) \\
&\le
V^{\pi^*_{\mathcal{M}}}_{\mathcal{M}}(s)
-
Q^{\pi^*_{\mathcal{M}}}_{\mathcal{M}}(s,a)
+
Q^{\pi^*_{\mathcal{M}}}_{\tilde{\mathcal{M}}}(s,a)
-
V^{\pi^*_{\mathcal{M}}}_{\tilde{\mathcal{M}}}(s) \\
&=
Q^{\pi^*_{\mathcal{M}}}_{\tilde{\mathcal{M}}}(s,a)
-
Q^{\pi^*_{\mathcal{M}}}_{\mathcal{M}}(s,a)
-
\Big(
V^{\pi^*_{\mathcal{M}}}_{\tilde{\mathcal{M}}}(s)
-
V^{\pi^*_{\mathcal{M}}}_{\mathcal{M}}(s)
\Big) \\
&=
Q^{\pi^*_{\mathcal{M}}}_{\tilde{\mathcal{M}}}(s,a)
-
Q^{\pi^*_{\mathcal{M}}}_{\mathcal{M}}(s,a) -
\Big(
Q^{\pi^*_{\mathcal{M}}}_{\tilde{\mathcal{M}}}
(s,\pi_{\mathcal{M}}^*(s))
-
Q^{\pi^*_{\mathcal{M}}}_{\mathcal{M}}
(s,\pi_{\mathcal{M}}^*(s))
\Big).
\end{aligned}
\end{equation*}
For any action $a\in\mathcal{A}$, by the definition of the Q-function:
$$
Q^{\pi^*_{\mathcal{M}}}_{{\tilde{\mathcal{M}}}}(s,a) - Q^{\pi^*_{\mathcal{M}}}_{{\mathcal{M}}}(s,a) = r_{\tilde{\mathcal{M}}}(s,a) + \gamma P_{\tilde{\mathcal{M}}}(s,a)^\top V_{\tilde{\mathcal{M}}}^{\pi^*_{\mathcal{M}}} - r_{{\mathcal{M}}}(s,a) - \gamma P_{{\mathcal{M}}}(s,a)^\top V_{{\mathcal{M}}}^{\pi^*_{\mathcal{M}}}.
$$
Under the linear MDP assumption, this expands to:
\begin{equation*}
\begin{aligned}
&Q^{\pi^*_{\mathcal{M}}}_{\tilde{\mathcal{M}}}(s,a)
- Q^{\pi^*_{\mathcal{M}}}_{\mathcal{M}}(s,a) \\
&=
r_{\tilde{\mathcal{M}}}(s,a)-r_{\mathcal{M}}(s,a)
+\gamma
\big(P_{\tilde{\mathcal{M}}}(s,a)-P_{\mathcal{M}}(s,a)\big)^\top
V_{\tilde{\mathcal{M}}}^{\pi^*_{\mathcal{M}}}
+\gamma P_{\mathcal{M}}(s,a)^\top
\big(
V_{\tilde{\mathcal{M}}}^{\pi^*_{\mathcal{M}}}
-
V_{\mathcal{M}}^{\pi^*_{\mathcal{M}}}
\big) \\
&=
\phi(s,a)^\top
(\theta_{\tilde{\mathcal{M}}}-\theta_{\mathcal{M}})
+\gamma \phi(s,a)^\top
(\mu_{\tilde{\mathcal{M}}}-\mu_{\mathcal{M}})
V_{\tilde{\mathcal{M}}}^{\pi^*_{\mathcal{M}}}
+\gamma P_{\mathcal{M}}(s,a)^\top
\big(
V_{\tilde{\mathcal{M}}}^{\pi^*_{\mathcal{M}}}
-
V_{\mathcal{M}}^{\pi^*_{\mathcal{M}}}
\big) \\
&=
\phi(s,a)^\top
\Big(
\theta_{\tilde{\mathcal{M}}}-\theta_{\mathcal{M}}
+\gamma
(\mu_{\tilde{\mathcal{M}}}-\mu_{\mathcal{M}})
V_{\tilde{\mathcal{M}}}^{\pi^*_{\mathcal{M}}}
\Big)
+\gamma P_{\mathcal{M}}(s,a)^\top
\big(
V_{\tilde{\mathcal{M}}}^{\pi^*_{\mathcal{M}}}
-
V_{\mathcal{M}}^{\pi^*_{\mathcal{M}}}
\big).
\end{aligned}
\end{equation*}
Applying the same logic to the Q function $Q^{\pi^*_{\mathcal{M}}}_{\mathcal{M}}(s,\pi_{\mathcal{M}}^*(s))$, we get:
\[
\begin{aligned}
&Q^{\pi^*_{\mathcal{M}}}_{\tilde{\mathcal{M}}}(s,\pi_{\mathcal{M}}^*(s))- Q^{\pi^*_{\mathcal{M}}}_{\mathcal{M}}(s,\pi_{\mathcal{M}}^*(s))
\\&=\phi(s,\pi_{\mathcal{M}}^*(s))^\top(\theta_{\tilde{\mathcal{M}}}-\theta_{\mathcal{M}}+\gamma(\mu_{\tilde{\mathcal{M}}}-\mu_{{\mathcal{M}}}) V_{\tilde{\mathcal{M}}}^{\pi^*_{\mathcal{M}}}) + \gamma P_{{\mathcal{M}}}(s,\pi_{\mathcal{M}}^*(s))^\top (V_{\tilde{\mathcal{M}}}^{\pi^*_{\mathcal{M}}}-V_{{\mathcal{M}}}^{\pi^*_{\mathcal{M}}})
\end{aligned}
\]
Therefore, we have that:
\[
\begin{aligned}
    \Delta_{sa} &\le (\phi(s,a)-\phi(s,\pi_{\mathcal{M}}^*(s)))^\top(\theta_{\tilde{\mathcal{M}}}-\theta_{\mathcal{M}}+\gamma(\mu_{\tilde{\mathcal{M}}}-\mu_{{\mathcal{M}}}) V_{\tilde{\mathcal{M}}}^{\pi^*_{\mathcal{M}}}) \\&+ \gamma (P_{{\mathcal{M}}}(s,a)-P_{{\mathcal{M}}}(s,\pi_{\mathcal{M}}^*(s)))^\top (V_{\tilde{\mathcal{M}}}^{\pi^*_{\mathcal{M}}}-V_{{\mathcal{M}}}^{\pi^*_{\mathcal{M}}}).
\end{aligned}
\]
We provide an upper bound for the absolute value of the second term. Note that
$$
    \big| \gamma (P_{{\mathcal{M}}}(s,a)-P_{{\mathcal{M}}}(s,\pi_{\mathcal{M}}^*(s)))^\top (V_{\tilde{\mathcal{M}}}^{\pi^*_{\mathcal{M}}}-V_{{\mathcal{M}}}^{\pi^*_{\mathcal{M}}})\big| \le 2\gamma \lVert V_{\tilde{\mathcal{M}}}^{\pi^*_{\mathcal{M}}}-V_{{\mathcal{M}}}^{\pi^*_{\mathcal{M}}}\rVert_{\infty}.
$$
The $\lVert V_{\tilde{\mathcal{M}}}^{\pi^*_{\mathcal{M}}}-V_{{\mathcal{M}}}^{\pi^*_{\mathcal{M}}}\rVert_{\infty}$ can be upper bounded by:
\[
\begin{aligned}
\lVert V_{\tilde{\mathcal{M}}}^{\pi^*_{\mathcal{M}}}-V_{{\mathcal{M}}}^{\pi^*_{\mathcal{M}}}\rVert_{\infty} &\le \lVert Q_{\tilde{\mathcal{M}}}^{\pi^*_{\mathcal{M}}} - Q_{{\mathcal{M}}}^{\pi^*_{\mathcal{M}}} \rVert_{\infty}\\&\le \max_{(s^\prime,a^\prime)\in\mathcal{S}\times\mathcal{A}} \big|\phi(s^\prime,a^\prime)^\top(\theta_{\tilde{\mathcal{M}}}-\theta_{\mathcal{M}}+\gamma(\mu_{\tilde{\mathcal{M}}}-\mu_{{\mathcal{M}}}) V_{\tilde{\mathcal{M}}}^{\pi^*_{\mathcal{M}}})\big| + \gamma \lVert V_{\tilde{\mathcal{M}}}^{\pi^*_{\mathcal{M}}}-V_{{\mathcal{M}}}^{\pi^*_{\mathcal{M}}}\rVert_{\infty},
\end{aligned}
\]
which further implies that
$$
    \lVert V_{\tilde{\mathcal{M}}}^{\pi^*_{\mathcal{M}}}-V_{{\mathcal{M}}}^{\pi^*_{\mathcal{M}}}\rVert_{\infty}  \le \frac{1}{1-\gamma}\max_{(s^\prime,a^\prime)\in\mathcal{S}\times\mathcal{A}} \big|\phi(s^\prime,a^\prime)^\top(\theta_{\tilde{\mathcal{M}}}-\theta_{\mathcal{M}}+\gamma(\mu_{\tilde{\mathcal{M}}}-\mu_{{\mathcal{M}}}) V_{\tilde{\mathcal{M}}}^{\pi^*_{\mathcal{M}}})\big|.
$$

Therefore, we can upper bound $\Delta_{sa}$ by
\[ 
\begin{aligned}
&\big|(\phi(s,a)-\phi(s,\pi_{\mathcal{M}}^*(s)))^\top(\theta_{\tilde{\mathcal{M}}}-\theta_{\mathcal{M}}+\gamma(\mu_{\tilde{\mathcal{M}}}-\mu_{{\mathcal{M}}}) V_{\tilde{\mathcal{M}}}^{\pi^*_{\mathcal{M}}})\big| \\&+ \frac{2\gamma}{1-\gamma}\max_{(s^\prime,a^\prime)\in\mathcal{S}\times \mathcal{A}}\big| \phi(s^\prime,a^\prime)^\top(\theta_{\tilde{\mathcal{M}}}-\theta_{\mathcal{M}}+\gamma(\mu_{\tilde{\mathcal{M}}}-\mu_{{\mathcal{M}}}) V_{\tilde{\mathcal{M}}}^{\pi^*_{\mathcal{M}}})\big|.
\end{aligned}
\]
Finally, applying H\"{o}lder's inequality with the weighted norm $\lVert \cdot \rVert_{\Lambda(\omega)}$ and its dual $\lVert \cdot \rVert_{\Lambda(\omega)^{-1}}$:
\begin{equation*}
\begin{aligned}
&\Big|
\big(\phi(s,a)-\phi(s,\pi_{\mathcal{M}}^*(s))\big)^\top
\Big(
\theta_{\tilde{\mathcal{M}}}-\theta_{\mathcal{M}}
+\gamma(\mu_{\tilde{\mathcal{M}}}-\mu_{\mathcal{M}})
V_{\tilde{\mathcal{M}}}^{\pi^*_{\mathcal{M}}}
\Big)
\Big| \\
&\quad\le
\big\|
\phi(s,a)-\phi(s,\pi_{\mathcal{M}}^*(s))
\big\|_{\Lambda(\omega)^{-1}}
\Big\|
\theta_{\tilde{\mathcal{M}}}-\theta_{\mathcal{M}}
+\gamma(\mu_{\tilde{\mathcal{M}}}-\mu_{\mathcal{M}})
V_{\tilde{\mathcal{M}}}^{\pi^*_{\mathcal{M}}}
\Big\|_{\Lambda(\omega)},
\end{aligned}
\end{equation*}
and 
\[
\begin{aligned}
&\max_{(s^\prime,a^\prime)\in\mathcal{S}\times \mathcal{A}}\big| \phi(s^\prime,a^\prime)^\top(\theta_{\tilde{\mathcal{M}}}-\theta_{\mathcal{M}}+\gamma(\mu_{\tilde{\mathcal{M}}}-\mu_{{\mathcal{M}}}) V_{\tilde{\mathcal{M}}}^{\pi^*_{\mathcal{M}}})\big|
\\&\le  \max_{(s^\prime,a^\prime)\in\mathcal{S}\times \mathcal{A}} \lVert \phi(s^\prime,a^\prime)\rVert_{\Lambda(\omega)^{-1}} \lVert \theta_{\tilde{\mathcal{M}}}-\theta_{\mathcal{M}}+\gamma(\mu_{\tilde{\mathcal{M}}}-\mu_{{\mathcal{M}}}) V_{\tilde{\mathcal{M}}}^{\pi^*_{\mathcal{M}}}\rVert_{\Lambda(\omega)}.
\end{aligned}
\]
Therefore, we conclude that
\[
\begin{aligned}
    &\Delta_{sa} \le \left( \lVert \phi(s,a)-\phi(s,\pi_{\mathcal{M}}^*(s)) \rVert_{\Lambda(\omega)^{-1}}+\frac{2\gamma}{1-\gamma} \max_{(s^\prime,a^\prime)\in\mathcal{S}\times \mathcal{A}}\lVert \phi(s^\prime,a^\prime)\rVert_{\Lambda(\omega)^{-1}}\right) \\&\times
\lVert \theta_{\tilde{\mathcal{M}}}-\theta_{\mathcal{M}}+\gamma(\mu_{\tilde{\mathcal{M}}}-\mu_{{\mathcal{M}}}) V_{\tilde{\mathcal{M}}}^{\pi^*_{\mathcal{M}}}\rVert_{\Lambda(\omega)},
\end{aligned}
\]
where the design matrix is defined as:
\begin{equation*}
   \Lambda(\omega) =  \sum_{s^\prime \in \mathcal{S}, a^\prime \in \mathcal{A}} \omega_{s^\prime a^\prime} \phi(s^\prime,a^\prime) \phi(s^\prime,a^\prime)^\top.
\end{equation*}

\section{Proof of Theorem~\ref{thm: linear rate lower bound opt}}
For any state-action pair $(s,a)\in\mathcal{S}\times \mathcal{A}$, $R_{\mathcal{M}}(s,a)$ is a bounded random variable with support $[0,1]$. Thus, it is sub-Gaussian, and we use the bound 
$
\mathbb{E}[\exp(\theta(R_{\mathcal{M}}(s,a)-r_{\mathcal{M}}(s,a)))] \le \exp(\theta^2/2),   
$
which means 
$
    \log \mathbb{E}[\exp(\theta(R_{\mathcal{M}}(s,a)))] \le \theta r_{\mathcal{M}}(s,a) +\frac{\theta^2}{2}.
$
Since the rate function $I_2(y(s,a))$ of $R_{\mathcal{M}}(s,a)$ is defined as 
\begin{equation}
\label{eq: I2 lb}
\begin{aligned}
    I_2(y(s,a)) &= \sup_{\theta}\left\{\theta y(s,a) - \log\mathbb{E}[\exp(\theta(R_{\mathcal{M}}(s,a)))]\right\}\\
    &\ge \sup_{\theta}\left\{\theta (y(s,a)-r_{\mathcal{M}}(s,a)) - \frac{\theta^2}{2} \right\}\\
    &=\frac{(y(s,a)-r_{\mathcal{M}}(s,a))^2}{2}.
\end{aligned}
\end{equation}

Applying Lemma~\ref{lemma: linear rate} with $v=V_{\tilde{\mathcal{M}}}^{\pi^*_{\mathcal{M}}}$, and combining it with
\eqref{eq: I2 lb}, we obtain
$$
\sum_{\substack{s^\prime \in \mathcal{S}, a^\prime \in \mathcal{A}}}
 \omega_{s^\prime a^\prime}
 \left(
 I_1(x(s^\prime,a^\prime))
 +I_2(y(s^\prime,a^\prime))
 \right)\ge
\frac{(1-\gamma)^2}{6}
\left\lVert
\theta_{\tilde{\mathcal{M}}}-\theta_{\mathcal{M}}
+
\gamma(\mu_{\tilde{\mathcal{M}}}-\mu_{\mathcal{M}})
V_{\tilde{\mathcal{M}}}^{\pi^*_{\mathcal{M}}}
\right\rVert_{\Lambda(\omega)}^2,
$$
where the third inequality follows from $a^2+b^2 \ge \frac{1}{2}(a+b)^2$.

Therefore, the original inner layer's optimization
\begin{equation*}
    \inf_{(x,y) \in \mathcal{E}_{s,a}}
    \sum_{s^\prime \in \mathcal{S}, a^\prime \in \mathcal{A}}
    \omega_{s^\prime a^\prime}
    \left(I_1(x(s^\prime,a^\prime))+I_2(y(s^\prime,a^\prime))\right)
\end{equation*}
is lower bounded by the following optimization problem:
\begin{equation*}
\begin{aligned}
\inf_{\tilde{\mathcal{M}}}\quad
&
\frac{(1-\gamma)^2}{6}
\Big\|
\theta_{\tilde{\mathcal{M}}}-\theta_{\mathcal{M}}
+\gamma(\mu_{\tilde{\mathcal{M}}}-\mu_{\mathcal{M}})
V_{\tilde{\mathcal{M}}}^{\pi^*_{\mathcal{M}}}
\Big\|_{\Lambda(\omega)}^2
\\[0.5em]
\text{s.t.}\quad
\Delta_{sa}
&\le
\Bigg(
\big\|
\phi(s,a)-\phi\!\left(s,\pi_{\mathcal{M}}^*(s)\right)
\big\|_{\Lambda(\omega)^{-1}}
\\
&\qquad\quad
+
\frac{2\gamma}{1-\gamma}
\max_{(s',a')\in\mathcal{S}\times\mathcal{A}}
\big\|
\phi(s',a')
\big\|_{\Lambda(\omega)^{-1}}
\Bigg)
\\
&\qquad\qquad
\times
\Big\|
\theta_{\tilde{\mathcal{M}}}-\theta_{\mathcal{M}}
+\gamma(\mu_{\tilde{\mathcal{M}}}-\mu_{\mathcal{M}})
V_{\tilde{\mathcal{M}}}^{\pi^*_{\mathcal{M}}}
\Big\|_{\Lambda(\omega)} .
\end{aligned}
\end{equation*}

The optimal value of this relaxed problem is
$$
    \frac{(1-\gamma)^2}{6}
    \left(
    \frac{\Delta_{sa}}{
    \big\|
    \phi(s,a)-\phi\!\left(s,\pi_{\mathcal{M}}^*(s)\right)
    \big\|_{\Lambda(\omega)^{-1}}
    +
    \frac{2\gamma}{1-\gamma}
    \max_{(s^\prime,a^\prime)\in\mathcal{S}\times\mathcal{A}}
    \big\|
    \phi(s^\prime,a^\prime)
    \big\|_{\Lambda(\omega)^{-1}}
    }
    \right)^2.
$$
Taking the minimum over all suboptimal state-action pairs and maximizing over
$\omega\in\mathcal W$ gives the lower bound.

%%%%%%%%%%%%%%%%%
\end{document}